\documentclass[11pt,twoside]{article}
\usepackage[latin9]{inputenc}
\usepackage{verbatim}
\usepackage{cprotect}
\usepackage{booktabs}
\usepackage{hepnames}
\usepackage{amsmath}
\usepackage{amsthm}
\usepackage{amssymb}
\usepackage{makeidx}
\makeindex
\usepackage{graphicx}
\usepackage{nomencl}
\makenomenclature

\makeatletter

\ProvideTextCommandDefault{\guillemotleft}{%
  {\usefont{U}{lasy}{m}{n}\char'50\kern-.15em\char'50}%
\penalty10000\hskip0pt\relax%
}
\ProvideTextCommandDefault{\guillemotright}{%
  \penalty10000\hskip0pt%
  {\usefont{U}{lasy}{m}{n}\char'51\kern-.15em\char'51}%
}
\DeclareTextSymbolDefault{\textquotedbl}{T1}
\providecommand{\tabularnewline}{\\}

\theoremstyle{plain}
\newtheorem{thm}{\protect\theoremname}
\theoremstyle{definition}
\newtheorem{example}[thm]{\protect\examplename}
\theoremstyle{definition}
\newtheorem{defn}[thm]{\protect\definitionname}
\theoremstyle{plain}
\newtheorem{prop}[thm]{\protect\propositionname}
\newenvironment{lyxcode}
	{\par\begin{list}{}{
		\setlength{\rightmargin}{\leftmargin}
		\setlength{\listparindent}{0pt}
		\raggedright
		\setlength{\itemsep}{0pt}
		\setlength{\parsep}{0pt}
		\normalfont\ttfamily}%
	 \item[]}
	{\end{list}}
\theoremstyle{plain}
\newtheorem{lem}[thm]{\protect\lemmaname}
\theoremstyle{remark}
\newtheorem{rem}[thm]{\protect\remarkname}
\newcommand\thmsname{\protect\theoremname}
\newcommand\nm@thmtype{theorem}
\theoremstyle{plain}

\newenvironment{namedthm}[1][Undefined Theorem Name]{
  \ifx{#1}{Undefined Theorem Name}\renewcommand\nm@thmtype{theorem*}
  \else\renewcommand\thmsname{#1}\renewcommand\nm@thmtype{namedtheorem}
  \fi
  \begin{\nm@thmtype}}
  {\end{\nm@thmtype}}
\theoremstyle{plain}
\newtheorem{cor}[thm]{\protect\corollaryname}

\usepackage{jair}
\usepackage{theapa}
\usepackage{rawfonts}

\usepackage{tikz}
\usetikzlibrary{arrows, automata}

\jairheading{1}{1993}{1-15}{6/91}{9/91}
\ShortHeadings{Combined Qualitative Constraint Networks}
{Cohen-Solal, Bouzid, \& Niveau}
\firstpageno{1}

\usepackage{mathtools} 
\usepackage{amssymb}
\usepackage{stmaryrd}

\makeatletter
\newcommand*{\reflectmathsymbol}[2]{%
  \reflectbox{\rotatebox[origin=c]{180}{$\m@th#1#2$}}%
}
\makeatother

\newcommand*{\mirrorRsh}{%
\mathop{%
  \mathpalette\reflectmathsymbol{\raisebox{-.3ex}{$\Rsh$}}
}%
}

\usepackage{amsbsy}
\usepackage{graphics} 

\makeatother

\providecommand{\corollaryname}{Corollary}
\providecommand{\definitionname}{Definition}
\providecommand{\examplename}{Example}
\providecommand{\lemmaname}{Lemma}
\providecommand{\propositionname}{Proposition}
\providecommand{\remarkname}{Remark}
\providecommand{\theoremname}{Theorem}

\begin{document}
\title{Deciding the Satisfiability of Combined Qualitative Constraint Networks}
\author{\name{Q}uentin Cohen-Solal \email quentin.cohen-solal@dauphine.psl.eu
\\
 \addr LAMSADE, Université Paris-Dauphine, PSL, CNRS, France\\
 \AND \name{M}aroua Bouzid \email maroua.bouzid-mouaddib@unicaen.fr\\
\name{A}lexandre Niveau \email alexandre.niveau@unicaen.fr\\
 \addr Normandie Univ, UNICAEN, ENSICAEN, CNRS, GREYC, 14000 Caen,
France}
\maketitle
\begin{abstract}
Among the various forms of reasoning studied in the context of artificial
intelligence, \emph{qualitative reasoning} makes it possible to infer
new knowledge in the context of imprecise, incomplete information
without numerical values. In this paper, we propose a formal framework
unifying several forms of extensions and combinations of \emph{qualitative
formalisms}, including \emph{multi-scale reasoning}, \emph{temporal
sequences}, and \emph{loose integrations}. This framework makes it
possible to reason in the context of each of these combinations and
extensions, but also to study in a unified way the satisfiability
decision and its complexity. In particular, we establish two complementary
theorems guaranteeing that the satisfiability decision is polynomial,
and we use them to recover the known results of the \emph{size-topology
combination}. We also generalize the main definition of qualitative
formalism to include qualitative formalisms excluded from the definitions
of the literature, important in the context of combinations.
\end{abstract}
\noindent\begin{minipage}[t]{1\columnwidth}%
\global\long\def\Cl#1{\mathfrak{C}\left(#1\right)}%

\global\long\def\id{\mathrm{id}}%

\global\long\def\conv#1#2{\mathop{\Rsh_{{#1}}^{{#2}}}}%

\global\long\def\cconv{\mathop{\Rsh}}%

\global\long\def\convDe#1#2#3{{_{{\scriptscriptstyle #3}}{\Rsh_{#1}^{#2}}}}%

\global\long\def\convinv#1#2{\mathop{{\mirrorRsh}_{{#1}}^{{#2}}}}%

\global\long\def\cconvinv{\mathop{\mirrorRsh}}%

\global\long\def\upgc{\mathop{\uparrow}}%

\global\long\def\downgc{\mathop{\downarrow}}%

\global\long\def\sinv#1{\overline{#1}}%

\global\long\def\scompl#1{#1^{\lhook}}%

\global\long\def\B{\mathcal{B}}%

\global\long\def\A{\mathcal{A}}%

\global\long\def\se{\mathrm{e}}%

\global\long\def\Base{\boldsymbol{\B}}%

\global\long\def\svide{\varnothing}%

\global\long\def\atome{\in}%

\global\long\def\comp{\circ}%

\global\long\def\wcomp{\diamond}%

\global\long\def\pcomp{\ast}%

\global\long\def\acomp{\varocircle}%

\global\long\def\ainter{\wedge}%

\global\long\def\aunion{\vee}%

\global\long\def\inv#1{\overline{#1}}%

\global\long\def\ainv#1{#1{}^{{\scriptscriptstyle \smallsmile}}}%

\global\long\def\winv#1{\widetilde{#1}}%

\global\long\def\acompl#1{\neg#1}%

\global\long\def\compl#1{#1^{\lhook}}%

\global\long\def\size#1{\mathrm{size}(#1)}%

\global\long\def\scomp{\varocircle}%

\global\long\def\sinter{\cap}%

\global\long\def\sunion{\cup}%

\global\long\def\sbigunion{\bigcup}%

\global\long\def\Ptime{\mathsf{P}}%

\global\long\def\PSPACE{\mathsf{PSPACE}}%

\global\long\def\NP{\mathsf{NP}}%

\global\long\def\T{\mathcal{T}}%

\global\long\def\U{\mathcal{U}}%

\global\long\def\I{\varphi}%

\global\long\def\topo{\mathfrak{T}}%

\global\long\def\clotureTopo#1{#1^{-}}%

\global\long\def\et{\,\wedge\,}%

\global\long\def\ou{\,\vee\,}%

\global\long\def\bigand{\qquad\wedge\qquad}%

\global\long\def\liste#1#2{\{#1\mid#2\}}%

\global\long\def\bigliste#1#2{\left\{  #1\mid#2\right\}  }%

\global\long\def\sep{\quad}%

\global\long\def\minisep{\ }%

\global\long\def\aeq{\mathrm{I}}%

\global\long\def\Id{\mathrm{Id}}%

\global\long\def\Ba{\top}%

\global\long\def\Br{\mathtt{B}}%

\global\long\def\Bw{\Re}%

\global\long\def\avide{\bot}%

\global\long\def\vide{\varnothing}%

\global\long\def\Aa{\mathcal{A}}%

\global\long\def\Au{\mathfrak{A}}%

\global\long\def\Ar{\mathtt{A}}%

\global\long\def\S{\mathcal{S}}%

\global\long\def\BaseR{\pmb{\Br}}%

\global\long\def\BaseA{\pmb{\Ba}}%

\global\long\def\BaseU{\mathfrak{B}}%

\global\long\def\PA{\mathrm{PA}}%

\global\long\def\IA{\mathrm{IA}}%

\global\long\def\PAs{\mathcal{S}_{\PA}}%

\global\long\def\PAc{\mathcal{C}_{\PA}}%

\global\long\def\IAs{\mathcal{S}_{\IA}}%

\global\long\def\IAc{\mathcal{C}_{\IA}}%

\global\long\def\CPA{\mathrm{CPA}}%

\global\long\def\ERA{\mathrm{ERA}}%

\global\long\def\INDU{\mathrm{INDU}}%

\global\long\def\QTC{\mathrm{QTC}}%

\global\long\def\MBRC{\mathrm{MBRC}}%

\global\long\def\DLR{\mathrm{DLR}}%

\global\long\def\hornDLR{\mathrm{hornDLR}}%

\global\long\def\RCC{\mathrm{RCC}}%

\global\long\def\RCA{\mathrm{RCA}}%

\global\long\def\RCCH{\RCC_{8}}%

\global\long\def\RCAH{\mathrm{\RCA_{8}}}%

\global\long\def\HH{\mathcal{\hat{H}}_{8}}%

\global\long\def\CH{\mathcal{C}_{8}}%

\global\long\def\QH{\mathcal{Q}_{8}}%

\global\long\def\RCAHs{\mathcal{S}_{\RCAH}}%

\global\long\def\RCAHc{\mathcal{C}_{\RCAH}}%

\global\long\def\RCCF{\RCC_{5}}%

\global\long\def\RCAF{\RCA_{5}}%

\global\long\def\Aseize{\mathrm{A}_{16}}%

\global\long\def\Acinq{\mathrm{A}_{5}}%

\global\long\def\PC{\mathrm{PC}}%

\global\long\def\IC{\mathrm{IC}}%

\global\long\def\QS{\mathrm{QS}}%

\global\long\def\CPC{\mathrm{CPC}}%

\global\long\def\CRC{\mathrm{CRC}}%

\global\long\def\CDC{\mathrm{CDC}}%

\global\long\def\preconv{\mathcal{P}}%

\global\long\def\IAp{\preconv_{\IA}}%

\global\long\def\CPAp{\preconv_{\CPA}}%

\global\long\def\Arp{\preconv_{\Ar}}%

\global\long\def\Arc{\mathcal{C}_{\Ar}}%

\global\long\def\nord{\mathrm{N}}%

\global\long\def\sud{\mathrm{S}}%

\global\long\def\ouest{\mathrm{W}}%

\global\long\def\est{\mathrm{E}}%

\global\long\def\nordouest{\nord\ouest}%

\global\long\def\nordest{\nord\est}%

\global\long\def\sudouest{\sud\ouest}%

\global\long\def\sudest{\sud\est}%

\global\long\def\notni{\reflectbox{\ensuremath{\notin}}}%

\global\long\def\connexion{\mathrm{C}}%

\global\long\def\p{\mathrm{P}}%

\global\long\def\pI{\overline{\mathrm{P}}}%

\global\long\def\o{\mathrm{O}}%

\global\long\def\dc{\mathrm{DC}}%

\global\long\def\ec{\mathrm{EC}}%

\global\long\def\po{\mathrm{PO}}%

\global\long\def\tpp{\mathrm{TPP}}%

\global\long\def\ntpp{\mathrm{NTPP}}%

\global\long\def\tppi{\mathrm{\overline{TPP}}}%

\global\long\def\ntppi{\mathrm{\overline{NTPP}}}%

\global\long\def\eq{\mathrm{EQ}}%

\global\long\def\pp{\mathrm{PP}}%

\global\long\def\ppi{\mathrm{\overline{PP}}}%

\global\long\def\contrainte#1#2#3{#1\mathrel{#2}#3}%

\global\long\def\E{\mathtt{E}}%

\global\long\def\CdeN{\mathtt{C}}%

\global\long\def\miniunion{{\scriptscriptstyle \cup}}%

\global\long\def\lsun#1{\{\mathord{#1}\}}%

\global\long\def\lsdeux#1#2{\{\mathord{#1},\ \mathord{#2}\}}%

\global\long\def\lstrois#1#2#3{\{\mathord{#1},\ \mathord{#2},\ \mathord{#3}\}}%

\global\long\def\lsquatre#1#2#3#4{\{\mathord{#1},\ \mathord{#2},\ \mathord{#3},\ \mathord{#4}\}}%

\global\long\def\lscinq#1#2#3#4#5{\{\mathord{#1},\ \mathord{#2},\ \mathord{#3},\ \mathord{#4},\ \mathord{#5}\}}%

\global\long\def\lssix#1#2#3#4#5#6{\{\mathord{#1},\ \mathord{#2},\ \mathord{#3},\ \mathord{#4},\ \mathord{#5},\ \mathord{#6}\}}%

\global\long\def\lssept#1#2#3#4#5#6#7{\{\mathord{#1},\ \mathord{#2},\ \mathord{#3},\ \mathord{#4},\ \mathord{#5},\ \mathord{#6},\ \mathord{#7}\}}%

\global\long\def\lshuit#1#2#3#4#5#6#7#8{\{\mathord{#1},\ \mathord{#2},\ \mathord{#3},\ \mathord{#4},\ \mathord{#5},\ \mathord{#6},\ \mathord{#7},\ \mathord{#8}\}}%

\global\long\def\lsneuf#1#2#3#4#5#6#7#8#9{\{\mathord{#1},\ \mathord{#2},\ \mathord{#3},\ \mathord{#4},\ \mathord{#5},\ \mathord{#6},\ \mathord{#7},\ \mathord{#8},\ \mathord{#9}\}}%
\global\long\def\un#1{\{\mathord{#1}\}}%

\global\long\def\deux#1#2{\{\mathord{#1},\mathord{#2}\}}%

\global\long\def\trois#1#2#3{\{\mathord{#1},\mathord{#2},\mathord{#3}\}}%

\global\long\def\quatre#1#2#3#4{\{\mathord{#1},\mathord{#2},\mathord{#3},\mathord{#4}\}}%

\global\long\def\cinq#1#2#3#4#5{\{\mathord{#1},\mathord{#2},\mathord{#3},\mathord{#4},\mathord{#5}\}}%

\global\long\def\six#1#2#3#4#5#6{\{\mathord{#1},\mathord{#2},\mathord{#3},\mathord{#4},\mathord{#5},\mathord{#6}\}}%

\global\long\def\sept#1#2#3#4#5#6#7{\{\mathord{#1},\mathord{#2},\mathord{#3},\mathord{#4},\mathord{#5},\mathord{#6},\mathord{#7}\}}%

\global\long\def\huit#1#2#3#4#5#6#7#8{\{\mathord{#1},\mathord{#2},\mathord{#3},\mathord{#4},\mathord{#5},\mathord{#6},\mathord{#7},\mathord{#8}\}}%

\global\long\def\neuf#1#2#3#4#5#6#7#8#9{\{\mathord{#1},\mathord{#2},\mathord{#3},\mathord{#4},\mathord{#5},\mathord{#6},\mathord{#7},\mathord{#8},\mathord{#9}\}}%

\global\long\def\F{\mathcal{F}}%

\global\long\def\STC{\mathrm{STC}}%

\global\long\def\hmax{h_{\mathrm{max}}}%

\global\long\def\hHH{h_{\HH}}%

\global\long\def\hQH{h_{\QH}}%

\global\long\def\hCH{h_{\CH}}%

\global\long\def\correctTable{\raisebox{3.5mm}{}}%

\noindent\begin{minipage}[t]{1\columnwidth}%
\global\long\def\scomp{\varocircle}%

\global\long\def\sinter{\cap}%

\global\long\def\sunion{\cup}%

\global\long\def\sbigunion{\bigcup}%

\global\long\def\sinv#1{\overline{#1}}%

\global\long\def\scompl#1{#1^{\lhook}}%

\global\long\def\B{\mathcal{B}}%

\global\long\def\A{\mathcal{A}}%

\global\long\def\se{\mathrm{e}}%

\global\long\def\Base{\boldsymbol{\B}}%

\global\long\def\svide{\varnothing}%

\global\long\def\atome{\in}%

\end{minipage}

\noindent\begin{minipage}[t]{1\columnwidth}%
\global\long\def\arbo{\mathrm{A}}%

\global\long\def\inArbo{\overset{\arbo}{\rightarrow}}%

\global\long\def\inA{\rightarrow}%

\global\long\def\inAinv{\leftarrow}%
\end{minipage}

\noindent\begin{minipage}[t]{1\columnwidth}%
\global\long\def\RCCS{\mathrm{\RCC_{7}}}%

\global\long\def\DRA{\mathcal{DRA}_{24}}%

\global\long\def\SPA{\mathrm{SPA}}%

\global\long\def\SIA{\mathrm{SIA}}%

\global\long\def\SPIA{\mathrm{SPIA}}%

\global\long\def\SPC{\mathrm{SPC}}%

\global\long\def\SIC{\mathrm{SIC}}%

\global\long\def\SPIC{\mathrm{SPIC}}%

\global\long\def\STA{\mathrm{STA}}%

\global\long\def\PIA{\mathrm{PIA}}%

\global\long\def\PIC{\mathrm{PIC}}%
\end{minipage}

\noindent\begin{minipage}[t]{1\columnwidth}%
\global\long\def\cun#1{(\mathord{#1})}%

\global\long\def\cdeux#1#2{(\mathord{#1},\mathord{#2})}%

\global\long\def\ctrois#1#2#3{(\mathord{#1},\mathord{#2},\mathord{#3})}%

\global\long\def\cquatre#1#2#3#4{(\mathord{#1},\mathord{#2},\mathord{#3},\mathord{#4})}%

\global\long\def\ccinq#1#2#3#4#5{(\mathord{#1},\mathord{#2},\mathord{#3},\mathord{#4},\mathord{#5})}%

\global\long\def\csix#1#2#3#4#5#6{(\mathord{#1},\mathord{#2},\mathord{#3},\mathord{#4},\mathord{#5},\mathord{#6})}%

\global\long\def\csept#1#2#3#4#5#6#7{(\mathord{#1},\mathord{#2},\mathord{#3},\mathord{#4},\mathord{#5},\mathord{#6},\mathord{#7})}%

\global\long\def\chuit#1#2#3#4#5#6#7#8{(\mathord{#1},\mathord{#2},\mathord{#3},\mathord{#4},\mathord{#5},\mathord{#6},\mathord{#7},\mathord{#8})}%

\global\long\def\cneuf#1#2#3#4#5#6#7#8#9{(\mathord{#1},\mathord{#2},\mathord{#3},\mathord{#4},\mathord{#5},\mathord{#6},\mathord{#7},\mathord{#8},\mathord{#9})}%
\end{minipage}

\global\long\def\prodDyna{\otimes}%
\nomenclature[ST Ensemble]{$\prodDyna$}{$\Cvierge \prodDyna \Cvierge'$ est l'ensemble des évolutions de couples d'entités correspondant à l'ensemble des couples d'évolutions $\Cvierge \times \Cvierge'$ de deux dynamiques $\Cvierge$ et $\Cvierge'$ (déf. \vref{def:produit-dynamique}).}

\global\long\def\unpack#1{#1^{\star}}%
\nomenclature[ST Ensemble]{$\unpack{\Cvierge}$}{Ensemble des couples d'évolutions correspondant à l'ensemble des évolutions de couples $\Cvierge$ (déf. \vref{def:unpack}).}

\global\long\def\fcarac{\mathrm{c}}%

\global\long\def\voisinage{\mathrm{p}}%

\global\long\def\dominance{\mathrm{d}}%

\global\long\def\Fcarac#1#2{{#1}^{#2}}%
\nomenclature[ST Formalisme 0]{$\Fcarac{\F}{\fcarac}$}{Formalisme des relations de $\F$ sur la caractéristique $\fcarac$ des entités de $\U$ (déf. \vref{def:changement-semantique}).}

\global\long\def\interpCarac#1#2{#1^{#2}}%
\nomenclature[ST Fonction]{$\interpCarac{\I}{\fcarac}$}{Interprétation selon la caractéristique $\fcarac$ : $\interpCarac{\I}{\fcarac}(r)=\liste{(x,y)\in\U\times\U}{\contrainte{\fcarac(x)}{\I(r)}{\fcarac(y)}}$ (déf. \vref{def:changement-semantique}).}

\global\long\def\carac{\mathscr{C}}%

\global\long\def\mesure{\mathfrak{t}}%

\global\long\def\mesureLebesgue{\mathbf{\mathfrak{\boldsymbol{t}}}}%
\nomenclature[ST Fonction]{$\mesureLebesgue$}{Mesure de Lebesgue (voir la section  \vref{subsec:Definitions-mathematiques-continuite-taille}).}

\global\long\def\IQSF{\I^{\fcarac}}%

\global\long\def\Tv#1{{\mathrm{T}#1}^{\mathrm{n}}}%
\nomenclature[ST Formalisme 1]{$\Tv{\F}^{\Cvierge}$}{Formalisme $\F$ temporalisé sans relation intermédiaire par la dynamique $\Cvierge$ (déf. \vref{def:TFv}).}

\global\long\def\Tvf#1{{\mathrm{T}#1}^{\voisinage'}}%

\global\long\def\Td#1{{\mathrm{T}#1}^{\mathrm{d}}}%
\nomenclature[ST Formalisme 1]{$\Td{\F}^{\Cvierge}$}{Formalisme $\F$ temporalisé sur intervalles  par la dynamique $\Cvierge$ (déf. \vref{def:TFd}).}

\global\long\def\Tt#1{{\mathrm{T}#1}^{\mesure}}%
\nomenclature[ST Formalisme 2]{$\Tt{\RCCH}$}{Calcul des régions connectés et mesurables temporalisé avec relations de taille (il se note également $\Tc{\RCCm}{\PC}{\mesureLebesgue}{}$ ; déf. \vref{def:TMS}).}\nomenclature[ST Formalisme 2]{$\Tc{\F}{\tilde{\F}}{\fcarac}{\Cvierge}$}{Formalisme $\F$ temporalisé par la dynamique $\Cvierge$ avec relations de $\tilde{F}$ sur la caractéristique $\fcarac$ (déf. \vref{def:TMS}).}

\global\long\def\Tdt#1{{\mathrm{T}#1}^{\dominance,\mesure}}%
\nomenclature[ST Formalisme 3]{$\Tdc{\F}{\tilde{\F}}{\fcarac}{\Cvierge}$}{Formalisme $\F$ temporalisé sur intervalles par la dynamique $\Cvierge$ avec relations de $\tilde{F}$ sur la caractéristique $\fcarac$ (déf. \vref{def:Fdc}).}

\global\long\def\Tvt#1{{\mathrm{T}#1}^{\voisinage,\mesure}}%
\nomenclature[ST Formalisme 3]{$\Tvc{\F}{\tilde{\F}}{\fcarac}{\Cvierge}$}{Formalisme $\F$ temporalisé sans relation intermédiaire par la dynamique $\Cvierge$ avec relations de $\tilde{F}$ sur la caractéristique $\fcarac$ (déf. \vref{def:Fvc}).}

\nomenclature[ST Opérateur]{$\conv{\F}{\Fcarac{\tilde{F}}{\fcarac}}$}{Projection de l'algèbre $\A$ vers l'algèbre $\tilde{\A}$  de l'intégration lâche de $\F$ et $\Fcarac{\tilde{\F}}{\fcarac}$ (déf. \vref{def:projection_integration_F_QS}).}

\global\long\def\Tc#1#2#3#4{{\mathrm{T}#1}^{#4}+\Fcarac{#2}{#3}}%

\global\long\def\Tdc#1#2#3#4{{\Td{#1}}^{#4}+\Fcarac{#2}{#3}}%

\global\long\def\Tvc#1#2#3#4{{\Tv{#1}}^{#4}+\Fcarac{#2}{#3}}%

\global\long\def\Te#1{{\mathrm{T}#1}_{\ast}}%

\global\long\def\Tec#1#2#3#4{{\Te{#1}}^{#4}+\Fcarac{#2}{#3}}%

\global\long\def\convFT#1{\conv{#1}{\mesure}}%

\global\long\def\convTF#1{\conv{\mesure}{#1}}%

\global\long\def\SRI#1#2{\mathrm{SRI}_{#1}^{#2}}%
\nomenclature[ST Relation]{$\SRI{[\tau,\tau']}{\F}$}{Dynamique de $\F$ sans relation intermédiaire durant l'intervalle $[\tau,\tau']$ : \\$\bigliste{(x,y)}{\exists t\in[\tau,\tau']\minisep\exists b,b'\in\B\sep(x,y)\in\interpTemporel b{[\tau,t[}{\Cvierge}\cap\interpTemporel{b\sunion b'}t{\Cvierge}\cap\interpTemporel{b'}{]t,\tau']}{\Cvierge}}$ (déf. \vref{def:SRI}).}

\global\long\def\FSRI#1#2{\mathrm{FSRI}_{#1}^{#2}}%

\global\long\def\interpTemporel#1#2#3{\left\langle #1\right\rangle _{#2}^{#3}}%
\nomenclature[ST Relation]{$\interpTemporel rI{\Cvierge}$}{Relation sur $\Cvierge$, interprétation de $r\in\A$ durant l'intervalle $I$ :\\$\interpTemporel rI{\Cvierge}=\liste{(x,y)\in\Cvierge}{\exists b\in r\minisep\forall t\in I\sep\contrainte{x(t)}{\I(b)}{y(t)}}$ (déf. \vref{def:relations_temporalisees}).}\nomenclature[ST Relation]{$\interpTemporel rt{\Cvierge}$}{Relation sur $\Cvierge$, interprétation de $r\in\A$ à l'instant t :\\$\interpTemporel rt{\Cvierge}=\liste{(x,y)\in\Cvierge}{\contrainte{x(t)}{\I(r)}{y(t)}}$ (déf. \vref{def:relations_temporalisees}).}

\global\long\def\interpTemporelTaille#1#2#3{\left\langle #1\right\rangle _{#2}^{#3}}%

\global\long\def\TPA{\Tv{\PA}}%

\global\long\def\TPC{\Tv{\PC}}%

\global\long\def\postdomine#1{\dashv_{#1}}%

\global\long\def\predomine#1{\vdash_{#1}}%

\global\long\def\domine#1{\succ_{#1}}%

\global\long\def\estdomine#1{\prec_{#1}}%

\global\long\def\perturbe#1{\rightsquigarrow_{#1}}%

\global\long\def\voisine#1{\wr_{#1}}%

\nomenclature[ST Opérateur]{$\conv{\mathfrak{r}}{}$ }{Projection du graphe de transition $\mathfrak{r}$, avec $\mathfrak{r}\in\{\wr,\succ,\prec,\rightsquigarrow,\vdash,\dashv\}$ : $\conv{\mathfrak{r}}{}b=\sbigunion\liste{b'\in\Base}{\contrainte b{\mathfrak{r}}{b'}}$ (déf. \vref{def:projections_espace_modes_topo}).}

\global\long\def\cpostdomine#1#2{\dashv_{#1}^{#2}}%
\nomenclature[ST Relation transition 2]{$\cpostdomine{\F}{\Cvierge}$}{$b \cpostdomine{\F}{\Cvierge} b'$ si et seulement s'il existe $(x,y)\in\Cvierge$ et $\tau,\tau'\in I$ tels quel $\tau<\tau'$,  $(x(t),y(t))\in\I(b)$ pour tout $t\in]\tau,\tau'[$ et $(x(\tau'),y(\tau'))\in\I(b')$ (déf. \vref{def:modele_espace_modes_topo} et déf. \vref{def:espace-de-modes_de_FQ}).}

\global\long\def\cpredomine#1#2{\vdash_{#1}^{#2}}%
\nomenclature[ST Relation transition 1]{$\cpredomine{\F}{\Cvierge}$}{$b \cpredomine{\F}{\Cvierge} b'$ si et seulement s'il existe $(x,y)\in\Cvierge$ et $\tau,\tau'\in I$ tels quel $\tau<\tau'$, $(x(\tau),y(\tau))\in\I(b)$ et $(x(t),y(t))\in\I(b')$ pour tout $t\in]\tau,\tau'[$ (déf. \vref{def:modele_espace_modes_topo} et déf. \vref{def:espace-de-modes_de_FQ}).}

\global\long\def\cdomine#1#2{\succ_{#1}^{#2}}%
\nomenclature[ST Relation transition 3]{$\cdomine{\F}{\Cvierge}$}{$b\cdomine{\F}{\Cvierge} b'$ si et seulement si $b\cpredomine{\F}{\Cvierge} b'\ou b'\cpostdomine{\F}{\Cvierge} b$ (déf. \vref{def:graphe_voisinage_formel}, déf. \vref{def:modele_espace_modes_topo} et déf. \vref{def:espace-de-modes_de_FQ}).}

\global\long\def\cvoisine#1#2{\wr_{#1}^{#2}}%
\nomenclature[ST Relation transition 5]{$\cvoisine{\F}{\Cvierge}$}{$ b\cvoisine{\F}{\Cvierge} b'$ si et seulement si $b\cperturbe{\F}{\Cvierge} b'\ou b\cperturbe{\F}{\Cvierge} b'$ (déf. \vref{def:graphe_voisinage_formel}, déf. \vref{def:modele_espace_modes_topo} et déf. \vref{def:espace-de-modes_de_FQ}).}

\global\long\def\cperturbe#1#2{\rightsquigarrow_{#1}^{#2}}%
\nomenclature[ST Relation transition 4]{$\cperturbe{\F}{\Cvierge}$}{$ b\cperturbe{\F}{\Cvierge} b'$ si et seulement si $b\cpostdomine{\F}{\Cvierge} b'\ou b\cpredomine{\F}{\Cvierge} b'$ (déf. \vref{def:graphe_voisinage_formel}, déf. \vref{def:modele_espace_modes_topo} et déf. \vref{def:espace-de-modes_de_FQ}).}

\global\long\def\tpostdomine#1{\cpostdomine{#1}{\mesure}}%

\global\long\def\tpredomine#1{\cpredomine{#1}{\mesure}}%

\global\long\def\tdomine#1{\cdomine{#1}{\mesure}}%

\global\long\def\tvoisine#1{\cvoisine{#1}{\mesure}}%

\global\long\def\tperturbe#1{\cperturbe{#1}{\mesure}}%

\global\long\def\Cvierge{\mathcal{C}}%
\nomenclature[ST Ensemble]{$\Cvierge$}{Une dynamique (déf. \vref{def:Dynamique}).}\nomenclature[ST Ensemble]{$\sqrt{\Cvierge}$}{Ensemble des évolutions d'entités de la dynamique $\Cvierge$ : $\sqrt{\Cvierge}=\liste x{\exists y\minisep (x,y)\in\Cvierge}$ (déf. \vref{def:sqrt-dynamique}).}

\global\long\def\C#1#2{\Cvierge_{#1}^{#2}}%
\nomenclature[ST Ensemble 1]{$\C{\topo}{I}$}{Dynamique générale de $\topo$ sur un intervalle $I$ (déf. \vref{def:dynamique-generale}).}\nomenclature[ST Ensemble 2]{$\C{\F}{I}$}{Dynamique générale de $\F$ sur un intervalle $I$ (déf. \vref{def:dynamique_FQ-dyna_symetrique}).}

\global\long\def\Cc#1#2#3{\Cvierge_{#1,#3}^{#2}}%
\nomenclature[ST Ensemble 3]{$\Cc{\topo}{I}{\fcarac}$}{Dynamique de la préservation de la caractéristique $\fcarac$ des entités de $\topo$ sur un intervalle $I$ :\\$\liste{f\in\C{\topo}{I}}{\forall t,t'\in I\sep\fcarac(f(t))=\fcarac(f(t'))}$ (déf. \vref{def:dynamique_caracteristique}).}\nomenclature[ST Ensemble 4]{$\Cc{\F}{I}{\fcarac}$}{Dynamique de la préservation de la caractéristique $\fcarac$ de $\F$ sur un intervalle $I$ (déf. \vref{def:dynamique-FQ-caracteristique}).}

\global\long\def\Ct#1#2{\Cc{#1}{#2}{\mesure}}%

\global\long\def\reg#1{\mathfrak{R}_{#1}}%

\global\long\def\regc#1{\mathfrak{R}_{#1}^{\mathrm{c}}}%

\global\long\def\regcm#1{\mathfrak{R}_{#1}^{\mathrm{c,m}}}%

\global\long\def\dh#1#2{d_{\mathrm{H}}(#1,#2)}%

\global\long\def\P#1#2{\smash{P_{#1}^{#2}}}%

\global\long\def\pRecb#1#2{P_{#1}(#2)}%

\global\long\def\rank#1{\operatorname{rang}(#1)}%

\global\long\def\nbsub#1{\mathcal{N}_{#1}}%

\global\long\def\Urccmb{{\mathfrak{R}_{\mathrm{b}}^{\mathrm{m}}}}%

\global\long\def\Urccm{{\mathfrak{R}_{\mathrm{m}}}}%
\nomenclature[ST Ensemble]{$\Urccm$}{Ensemble des régions mesurables de $\mathbb{R}^{n}$ (déf. \vref{def:Rm}).}

\global\long\def\RCCmb{\RCCH^{\mathrm{b,m}}}%

\global\long\def\RCCm{{\RCC_{8,\mathrm{m}}}}%
\nomenclature[ST Formalisme 0]{$\RCCm$}{$\RCCH$ restreint aux régions mesurables (déf. \vref{def:RCCHm}).}

\global\long\def\interpRCCmb{\I_{\mathrm{b,m}}}%

\global\long\def\interpRCCm{\I_{\mathrm{m}}}%
\nomenclature[ST Fonction]{$\interpRCCm$}{Interprétation de $\RCCH$ restreinte aux régions mesurables (déf. \vref{def:RCCHm}).}

\global\long\def\c{\mathrm{C}}%
\global\long\def\P{\mathcal{P}}%
\global\long\def\Nrcc{\mathcal{N}}%
\global\long\def\NPrcc{\mathcal{NP}_{8}}%
\global\long\def\Prcc{\mathcal{P}_{8}}%

\global\long\def\icclo{\mathord{\scomp\sinter}\text{-closed}}%
\end{minipage}

\section{Introduction}

Reasoning about time and space is omnipresent in our daily lives.
Computers can achieve them using quantitative approaches; however,
for human-computer interaction, quantitative data is often unavailable
or unnecessary. This is why research has been carried out about \emph{qualitative}
approaches to temporal and spatial reasoning -{}- such as the interval
algebra of Allen \cite{Allen1983} -{}- not only in artificial intelligence
but also in geographical information systems, databases, and multimedia
\cite{chittaro2000temporal,chen2015survey}. 

Qualitative reasoning deals with the analysis of qualitative descriptions
of the world. These descriptions are logical formulas and in general
more particularly \emph{constraint networks,} that is to say systems
of relations of the form $\contrainte x{R_{xy}}y$, $\contrainte y{R_{yz}}z$,
$\contrainte x{R_{xz}}z$, ... with $R_{xy},R_{yz}$, $R_{xz}$, ...
relations belonging to a particular set of relations and $x,y,z$,
... some variables belonging to a particular set of entities (generally
geometric entities or numbers). One of the central problems of reasoning
is to decide if a description is \emph{satisfiable}, i.e. if there
exists a solution which satisfies each constraint of the constraint
network (i.e. which satisfies each relation of the system).
\begin{example}
An example of constraint networks about $3$ set variables $x,y,z$
is: ``$\contrainte x{\subseteq}y$, $\contrainte y{\subseteq}z$,
$\contrainte x{\mathrm{\dc}}z$'' with the relation $\dc=\liste{\left(x,y\right)}{x\cap y=\vide}$.
This description is not satisfiable since there does not exist such
sets $x,y,z$.
\end{example}

Research on qualitative formalisms focuses in particular on 3 axes:
determining algorithms to decide the satisfiability of descriptions
corresponding to a particular set of relations, studying the algorithmic
complexity of satisfiability problems associated with this set of
relations, and determining tractable fragments (that is to say subsets
of relations having a polynomial complexity).

Some recent research has also focused on the combination of qualitative
approaches in order to increase their number of applications. One
of the most popular combinations is \emph{loose integration} \cite{wolfl2009combinations},
which combines different languages about the same entities. Spatio-temporal
formalisms, which combines space and time informations are other kinds
of combinations \cite{ligozat2013qualitative}, allowing for example
the processing of temporal sequences of spatial information \cite{westphal2013transition}.
Also fundamental to human intelligence is \emph{multi-scale reasoning},
i.e., the ability to reason at different levels of detail \cite{hobbs1985granularity},
namely to reason with information of different precisions. Both temporal
and spatial multi-scale qualitative reasoning have been studied \cite{Bittner2002,euzenat2005time,li2007qualitative,cohen2015algebra}.
Multi-scale reasoning is also a form of combination \cite{hobbs1985granularity,li2007qualitative,cohen2015algebra}.
\begin{example}
An example of loose integration about $3$ mesurable sets is: ``$\contrainte x{\subseteq}y$,
$\contrainte y{\subseteq}z$, and $\contrainte x{\subseteq}z$, ;
$\contrainte x>y$, $\contrainte y>z$, and $\contrainte x>z$''
where $\contrainte x>y$ means $\contrainte{\mu\left(x\right)}>{\mu\left(x\right)}$
(the size of $x$ is greater than the size of $y$). Although the
set relations are satisfiable and the size relations are satisfiable
the whole is not satisfiable.

An example of spatio-temporal description about 3 sets at two consecutive
instants $t_{1}$ and $t_{2}$ is: at $t_{1}$ $\contrainte x{\subseteq}y$,
$\contrainte y{\subseteq}z$, and $\contrainte x{\subseteq}z$ ; at
$t_{2}$ $\contrainte x{\supsetneq}y$, $\contrainte y{\supsetneq}z$,
and $\contrainte x{\supsetneq}z$. This description is satisfiable
except if we require that the sets have a constant size.

An example of multi-scale description about 3 instants at different
scales is: at the scale of seconds $\contrainte x<y$, $\contrainte y<z$,
and $\contrainte x<z$ ; at the scale of minutes $\contrainte x=y$
(the instants $x$ and $y$ are indistinguishable to the scale of
minutes), $\contrainte y<z$, and $\contrainte x<z$. This description
is satisfiable.
\end{example}

This paper is an extended version of a conference paper published
in the acts of AAAI 2017 \cite{cohen2017checking}, which introduces
the multi-algebra framework, a formal framework capturing the common
structure of loose integrations, multi-scale representations, and
temporal sequences of spatial information; all of these can indeed
be seen as tuples of \emph{constraint networks} having interdependencies.
With loose integrations, each network is based on a different formalism,
whereas with multi-scale and spatio-temporal representations all networks
are based on the same formalism but hold at different scales and different
time periods, respectively. Moreover, in each case, constraints in
one network of the tuple can entail constraints between the same variables
in the other networks. The entailed constraints correspond to how
the initial constraints are transformed by language (basic relations)
change, scale change, or reference period change, respectively. 

We study in particular the \emph{satisfiability decision} (also called
\emph{consistency checking}) in the context of our framework; we focus
on general results that are common to the three kinds of combination,
using a dissociable, well-known instance of a loose integration as
a running example. Specifically, we identify sufficient conditions
so that the generalized \emph{algebraic closure} can be used to decide
satisfiability of networks over some subclass -{}- which is therefore
\emph{tractable}. To sum up, we propose a framework for representing
knowledge, reasoning, and identifying tractable fragments, in a unified
way, for the three kinds of combination.

This extended version appends a new definition of qualitative formalism,
which includes qualitative formalisms excluded from previous definitions,
important in the context of combinations, shows how this framework
encompasses qualitative multi-scale reasoning, offers stronger conditions
to the tractability theorems to simplify their application, as well
as a technique to circumvent the non-satisfaction of some of their
conditions, and we add many details throughout the document.

We begin, in section \ref{sec:Background-and-Related}, by recalling
concepts related to temporal and spatial formalisms, then we give
some background on combinations of formalisms, notably loose integrations,
spatio-temporal sequences, and multi-scale reasoning. Then, in the
section \ref{sec:Symmetrical-Qualitative-Formalis}, we expose our
definition of qualitative formalism. Section \ref{sec:Framework_Multi-algebras}
presents the multi-algebra framework. Section \ref{sec:Tractability}
establishes our tractability results. Section \ref{sec:Properties-to-Apply}
introduces properties to simplify and expand the application of our
tractability results and then illustrates them on the combination
of size and topology. Finally, in the last section, we discuss about
the limitations of our framework and we conclude.

The original conference paper \cite{cohen2017checking} only contains
the main definitions, a weaker form of the two tractability theorems
and their illustrative application. This extended version thus contains
in addition, in particular, our generalized definition of qualitative
formalism, the generalized two tractability theorems, the properties
to simplify and expand their applications, and the discussion about
the limitations of our framework. 

\section{Background and Related Work\label{sec:Background-and-Related}}

\subsection{Qualitative Temporal and Spatial Formalisms\label{subsec:RW:Qualitative}}

In the context of qualitative temporal and spatial reasoning, we are
particularly interested in deciding  satisfiability of temporal or
spatial \emph{descriptions}, encoded by relations between spatial
or temporal \emph{entities} of a set $\U$ (called \emph{entity domain}
or also \emph{universe}). Each relation is a union of \emph{basic}
relations from a set $\Base$: this represents the uncertainty about
the actual basic relation -- e.g., $x\mathrel{\left(<\cup=\right)}y$
means that either $x<y$ or $x=y$. The set of all relations forms
a \emph{non-associative relation algebra} $\mathcal{A}$~\cite{ligozat2013qualitative}\index{non-associative relation algebra}
(Chapter 11). 
\begin{defn}[\cite{maddux1982some,hirsch2002relation}]
\label{def:algebre-non-associative} A non-associative relation algebra
is a tuple $(\Aa,\aunion,\acompl{},\avide,\Ba,\acomp,\ainv{\cdot},\aeq)$
such that $(\Aa,\aunion,\acompl{},\avide,\Ba)$ is a boolean algebra
and that the following properties are satisfied for all $x,y,z\in\Aa$,
with $x\ainter y=\acompl{(\acompl x\aunion\acompl y)}$ : 
\begin{align*}
\begin{aligned}\ainv{\ainv x} & =x\\
\ainv{(x\aunion y)} & =\ainv x\aunion\ainv y\\
\ainv{(x\acomp y)} & =\ainv y\acomp\ainv x
\end{aligned}
 &  & \begin{gathered}\aeq\acomp x=x\acomp\aeq=x\\
x\acomp(y\aunion z)=x\acomp y\aunion x\acomp z\\
(x\acomp y)\ainter\ainv z=\avide\iff(y\acomp z)\ainter\ainv x=\avide
\end{gathered}
\end{align*}

A non-associative relation algebra is \emph{finite} if $\A$ is finite.
\end{defn}

In each non-associative relation algebra, there is a particular relation,
the \emph{universal relation}\index{universal relation}, denoted
by $\B$, which is the union of all basic relations. Well-known algebras
include the interval algebra of Allen \cite{Allen1983}, but also
the \emph{point algebra} $\PA$ \cite{Vilain89constraintpropagation}\index{point algebra PA@point algebra $\PA$},
whose basic relations are $\Base_{\PA}=\{<,=,>\}$, and the algebra
$\RCCH$ of \emph{topological relations} \cite{randell1992spatial}\index{topological relations, RCCH@topological relations, $\RCCH$},
whose basic relations are described in Fig.~\ref{fig:RCC8_relations_horizontal}
and formalized in Table~\ref{tab:def-relations-rcc8}. There are
several operators over relations in $\mathcal{A}$: \emph{converse}
of a relation $r$, denoted by $\bar{r}$, \emph{intersection} of
$r_{1}$ and $r_{2}$, denoted by $r_{1}\cap r_{2}$, and \emph{(weak)
composition} of $r_{1}$ and $r_{2}$~\cite{renz2005weak}\index{weak composition wcomp@\emph{weak composition $\wcomp$}},
denoted by $r_{1}\wcomp r_{2}$. These operators allow one to reason,
by reducing the uncertainty about basic relations of entities: $x\mathrel{r}y\iff y\mathrel{\bar{r}}x$;
$x\mathrel{r}_{1}y\,\wedge\,{x\mathrel{r}_{2}y}\iff x\mathrel{(r_{1}\cap r_{2})}y$;
and $x\mathrel{r}_{1}y\,\wedge\,y\mathrel{r}_{2}z\implies x\mathrel{(r_{1}\diamond r_{2})}z$.
For example, if we know that $\contrainte x{\ec}y$ and $\contrainte y{\ntpp}z$,
we can deduce that $\contrainte x{\left(\po\cup\tpp\cup\ntpp\right)}z$.
The composition of basic relations of $\RCCH$ are described in Table~\ref{tab:Table-de-composition_rcc8}.

\begin{figure}[t]
\centering{}\includegraphics[width=0.95\columnwidth]{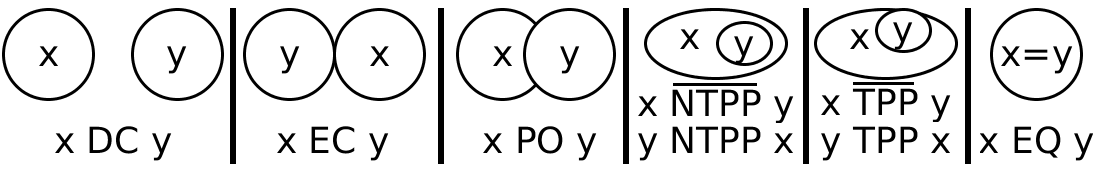}
\caption{The $8$ relations of $\protect\RCCH$ in the plane.}
\label{fig:RCC8_relations_horizontal}
\end{figure}

\begin{table}
\begin{centering}
\begin{tabular}{|c|c|}
\hline 
Relation & Definition\tabularnewline
\hline 
\hline 
$\contrainte x{\dc}y$ & $\lnot\left(\contrainte x{\c}y\right)$\tabularnewline
\hline 
$\contrainte x{\p}y$ & $\forall z\quad\contrainte z{\c}x\implies\contrainte z{\c}y$\tabularnewline
\hline 
$\contrainte x{\pp}y$ & $\contrainte x{\p}y\et\lnot\left(\contrainte y{\p}x\right)$\tabularnewline
\hline 
$\contrainte x{\eq}y$ & $\contrainte x{\p}y\et\contrainte y{\p}x$\tabularnewline
\hline 
$\contrainte x{\o}y$ & $\exists z\quad\contrainte z{\p}x\et\contrainte z{\p}y$\tabularnewline
\hline 
$\contrainte x{\po}y$ & $\contrainte x{\o}y\et\lnot\left(\contrainte x{\p}y\right)\et\lnot\left(\contrainte y{\p}x\right)$\tabularnewline
\hline 
$\contrainte x{\ec}y$ & $\contrainte x{\c}{y\et\lnot\left(\contrainte x{\o}y\right)}$\tabularnewline
\hline 
$\contrainte x{\tpp}y$ & $\contrainte x{\pp}{y\et\left(\exists z\ \contrainte z{\ec}x\et\contrainte z{\ec}y\right)}$\tabularnewline
\hline 
$\contrainte x{\ntpp}y$ & $\contrainte x{\pp}{y\et\lnot\left(\exists z\ \contrainte z{\ec}x\et\contrainte z{\ec}y\right)}$\tabularnewline
\hline 
$\contrainte x{\tppi}y$ & $\contrainte y{\tpp}x$\tabularnewline
\hline 
$\contrainte x{\ntppi}y$ & $\contrainte y{\ntpp}x$\tabularnewline
\hline 
\end{tabular}
\par\end{centering}
\caption{Definitions of RCC relations from the \emph{contact relation} $C(x,y)$
which means that the closed regions $x$ and $y$ intersect ($P$
is the \emph{part} \emph{relation} and $\protect\o$ the \emph{overlap
relation} ; variables $x,y,z$ are closed regions).\label{tab:def-relations-rcc8}}
\end{table}
\begin{figure}
\begin{centering}
\begin{tikzpicture}[scale=1,bend angle=23]  
       
\tikzstyle{every state}=[draw=none]

\node[state] (Bnn) at (1,3) {$\dc$};    
                                     
\node[state] (O1) at (3,3) {$\ec$};     
\node[state] (On) at (5,3) {$\po$};        
                                 
\node[state] (Sn) at (7,4.5) {$\tpp$};                     
\node[state] (fn) at (7,1.5) {$\tppi$};                  
                                       
\node[state] (Dn) at (9,6) {$\ntpp$};                     
\node[state] (en) at (9,3) {$\eq$};                     
\node[state] (dn) at (9,0) {$\ntppi$};

\draw[->,>=latex] (Bnn) -> (O1);                                                           
                                                                                                                         
\draw[->,>=latex] (On) -> (Sn);                                         
\draw[->,>=latex] (On) -> (fn);                 
\draw[->,>=latex] (O1) -> (On);   
                                                                                      
\draw[->,>=latex] (Sn) -> (en);

\draw[->,>=latex] (fn) -> (en);

\draw[->,>=latex] (Sn) -> (Dn);               
                                              
\draw[->,>=latex] (fn) -> (dn);

\end{tikzpicture}
\par\end{centering}
\caption{Partial order of basic relations of $\protect\RCCH$ (induced by the
order of Allen's relations~\cite{ligozat2013qualitative})\label{fig:Ordre-partiel-RCC8}}
\end{figure}
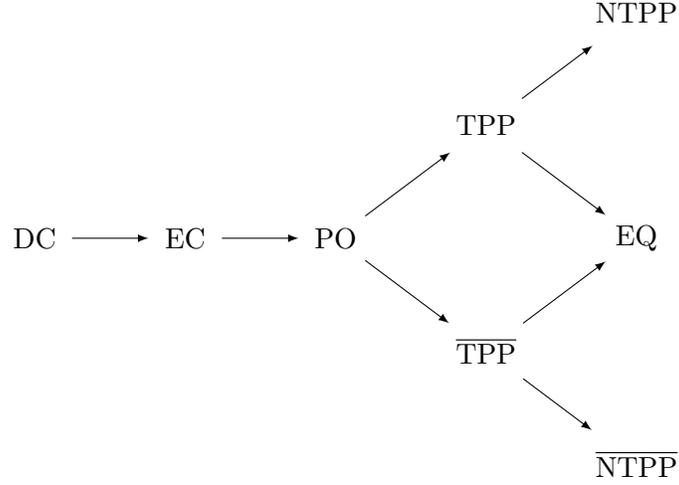
\begin{table}
\begin{centering}
{\scriptsize{}%
\begin{tabular}{|c|c|c|c|c|c|c|c|}
\hline 
{\scriptsize$\wcomp$} & {\scriptsize$\dc$} & {\scriptsize$\ec$} & {\scriptsize$\po$} & {\scriptsize$\tpp$} & {\scriptsize$\ntpp$} & {\scriptsize$\tppi$} & {\scriptsize$\ntppi$}\tabularnewline
\hline 
\hline 
{\scriptsize$\dc$} & {\scriptsize$[\dc,]$} & \multicolumn{4}{c|}{{\scriptsize$[\dc,\ntpp]$}} & \multicolumn{2}{c|}{{\scriptsize$\dc$}}\tabularnewline
\hline 
{\scriptsize$\ec$} & {\scriptsize$[\dc,\ntppi]$} & {\scriptsize$[\dc,\eq]$} & {\scriptsize$[\dc,\ntpp]$} & {\scriptsize$[\ec,\ntpp]$} & {\scriptsize$[\po,\ntpp]$} & {\scriptsize$[\dc,\ec]$} & {\scriptsize$\dc$}\tabularnewline
\hline 
{\scriptsize$\po$} & \multicolumn{2}{c|}{{\scriptsize$[\dc,\ntppi]$}} & {\scriptsize$[\dc,]$} & \multicolumn{2}{c|}{{\scriptsize$[\po,\ntpp]$}} & \multicolumn{2}{c|}{{\scriptsize$[\dc,\ntppi]$}}\tabularnewline
\hline 
{\scriptsize$\tpp$} & {\scriptsize$\dc$} & {\scriptsize$[\dc,\ec]$} & {\scriptsize$[\dc,\ntpp]$} & {\scriptsize$[\tpp,\ntpp]$} & {\scriptsize$\ntpp$} & {\scriptsize$[\dc,\eq]$} & {\scriptsize$[\dc,\ntppi]$}\tabularnewline
\hline 
{\scriptsize$\ntpp$} & \multicolumn{2}{c|}{{\scriptsize$\dc$}} & {\scriptsize$[\dc,\ntpp]$} & \multicolumn{2}{c|}{{\scriptsize$\ntpp$}} & {\scriptsize$[\dc,\ntpp]$} & {\scriptsize$[\dc,]$}\tabularnewline
\hline 
{\scriptsize$\tppi$} & {\scriptsize$[\dc,\ntppi]$} & {\scriptsize$[\ec,\ntppi]$} & {\scriptsize$[\po,\ntppi]$} & {\scriptsize$[\po,\eq]$} & {\scriptsize$[\po,\ntpp]$} & {\scriptsize$[\tppi,\ntppi]$} & {\scriptsize$\ntppi$}\tabularnewline
\hline 
{\scriptsize$\ntppi$} & {\scriptsize$[\dc,\ntppi]$} & \multicolumn{3}{c|}{{\scriptsize$[\po,\ntppi]$}} & {\scriptsize$[\po,]$} & \multicolumn{2}{c|}{{\scriptsize$\ntppi$}}\tabularnewline
\hline 
{\scriptsize$\eq$} & {\scriptsize$\dc$} & {\scriptsize$\ec$} & {\scriptsize$\po$} & {\scriptsize$\tpp$} & {\scriptsize$\ntpp$} & {\scriptsize$\tppi$} & {\scriptsize$\ntppi$}\tabularnewline
\hline 
\end{tabular}}{\scriptsize\par}
\par\end{centering}
\caption{Composition table of $\protect\RCCH$ \cite{ligozat2013qualitative}
(we denote by $[b,b']$ the relation $\bigcup\protect\liste{b''}{b\protect\leq b''\protect\et b''\protect\leq b'}$
and by $[b,]$ the relation $\bigcup\protect\liste{b'}{b\protect\leq b'}$
according to the order $\protect\leq$ of Figure~\ref{fig:Ordre-partiel-RCC8}
; column of $\protect\eq$ is omitted)\label{tab:Table-de-composition_rcc8}}
\end{table}
A temporal or spatial description can be modeled by a \emph{qualitative
constraint network}\index{qualitative constraint network} -- a labeled
graph in which nodes are \emph{entity variables} and edges are labeled
with a binary relation of $\mathcal{A}$. More formally, a qualitative
constraint network $N$ over a set of relations $\mathcal{S}\subseteq\A$
is a pair $N=(\mathtt{E},\mathtt{C})$ where $\mathtt{E}$ is a finite
set of entity variables of $\U$ and $\mathtt{C}$ is a set of \emph{constraints
over $\mathcal{S}$}, i.e., tuples $(x,r,y)$ with $x,y\in\mathtt{E}$
, $x\neq y$ and $r\in\mathcal{S}$. An example of a network over
$\RCCH$ is $N=(\{x,y,z\},\{(x,\ec,y),(y,\ntpp,z)\})$. Like all networks
in this paper, it is normalized, in the sense that for each pair $\{x,y\}$
there is at most one constraint, whose relation is denoted $N^{xy}$
(implicitly, $N^{xy}=\overline{N^{yx}}$). If there is no constraint
between two entities, the constraint is implicit, and the corresponding
relation is the universal relation $\mathcal{B}$. We say that a network
$N$ \emph{refines} another network $N'$ if it holds that $\forall x,y\in\mathtt{E}\colon N^{xy}\subseteq(N')^{xy}$,
which we denote by $N\subseteq N'$.\index{network refinining relation subseteq@\emph{network refinining relation $\subseteq$}}

\subsubsection*{Satisfiability}

\index{solution}The notion of \emph{solution} of a constraint network
depends on a semantics, which is given by an \emph{interpretation
function} $\varphi$ mapping any relation $r$ of the algebra to the
set of all pairs of entities from the domain $\U$ \emph{satisfying}
$r$. When $\varphi$ verifies specific properties, the triple $(\mathcal{A},\U,\varphi)$
constitutes a \emph{qualitative formalism} \cite{ligozat2013qualitative}
(Chapter 11). Thus, a solution of a constraint network $N$ is a set
$\{u_{x}\}_{x\in\mathtt{E}}\subseteq\U$ such that $\forall x,y\in\mathtt{E}\colon(u_{x},u_{y})\in\varphi\left(N^{xy}\right)$.
A fundamental problem is to determine whether a constraint network
has at least one solution, in which case it is said to be \emph{satisfiable}\index{satisfiable}
(or also \emph{consistent}). To each solution of a network $N$ corresponds
a unique \emph{scenario}\index{scenario} of $N$, i.e., a network
$S\subseteq N$ such that $\forall x,y\in\mathtt{E}\colon S^{xy}\in\Base$.
For instance, our example $\RCCH$ network is not a scenario ($N^{xy}$
is not basic). Finding a solution of a network $N$ amounts to finding
a satisfiable scenario $S$ of $N$, since any solution of $S$ is
a solution of $N$. Because the satisfiability decision is $\NP$-complete
for many algebras, some research focuses on particular tractable fragments:
\emph{tractable subclasses}. A \emph{subclass}\index{subclass} is
a set $\mathcal{S}\subseteq\A$ that is closed under intersection,
weak composition, and inversion. A subclass is \emph{tractable}\index{tractable}
if it is polynomial to decide  satisfiability of any network whose
relations are in $\mathcal{S}$ \cite{ligozat2013qualitative}.

A constraint network is \emph{algebraically closed}\index{algebraically closed}
-- a key concept to find satisfiable scenarios in a purely algebraic
way~-- if $N^{xz}\subseteq N^{xy}\wcomp N^{yz}$ for all $x,y,z\in\mathtt{E}$.
We can obtain from any network $N$ an algebraically closed network
having the same solutions by computing its \emph{algebraic closure}\index{algebraic closure}.
It can be done (in polynomial time) by repeatedly replacing each $N^{xz}$
by $(N^{xy}\wcomp N^{yz})\cap N^{xz}$ until a fixed point is reached.
If the resulting network is not \emph{trivially inconsistent}\index{trivially inconsistent}
(i.e., if none of its relations is the empty set), it is said to be
\emph{algebraically consistent} (or also \emph{$\diamond$-consistent}\index{diamond -consistent@\emph{$\diamond$-consistent}}).
In the literature, $\diamond$-consistency is often conflated with
\emph{path-consistency}, because for some formalisms they are equivalent~\cite{renz2005weak}.
An algebraically closed scenario is always algebraically consistent
by definition, but note that it is not necessarily satisfiable for
any formalism~\cite{renz2005weak}. However, when all the algebraically
closed scenarios of a formalism are satisfiable, the satisfiability
of its networks can be decided by searching for an algebraically closed
scenario, using backtracking methods based on algebraic closure~\cite{Ladkin1992}.
For some subclasses, any algebraically consistent network is satisfiable,
so there is no need to backtrack: such subclasses are thus tractable.
Examples of such tractable subclasses are $\HH$, $\QH$, and $\CH$:
the three maximal tractable basic subclasses of $\RCCH$~\cite{renz1999maximal}.
Other examples are the basic suclasses $\RCAHs$ again for $\RCCH$
and $\RCAHs$ for $\PAs$~\cite{long2015distributive}. Definitions
of these subclasses are recalled in Table~\ref{tab:def_classic_subclass}.
Tractability is proven in particular with the help of a particular
function of relations called \emph{refinement}. The refinement of
the point algebra is described in Table~\ref{tab:affinement_PA}.
The refinements of the tractable subclasses of RCC8 are described
in the following definition.

\begin{table}
\begin{centering}
\begin{tabular}{|c|c|}
\hline 
 & Definition\tabularnewline
\hline 
\hline 
$\Nrcc$ & $\liste{r\in\RCAH}{\po\nsubseteq r\et r\cap\left(\tpp\cup\ntpp\right)\neq\vide\et r\cap\left(\tppi\cup\ntppi\right)\neq\vide}$\tabularnewline
\hline 
$\NPrcc$ & $\Nrcc\cup\liste{r_{1}\cup\ec\cup r_{2}\cup\eq}{r_{1}\in\left\{ \vide,\dc\right\} \et r_{2}\in\left\{ \ntpp,\ntppi\right\} }$\tabularnewline
\hline 
$\Prcc$ & $\RCAH\backslash\NPrcc$\tabularnewline
\hline 
$\HH$ & $\Prcc\cap\liste{r\in\RCAH}{\ntpp\cup\eq\subseteq r\implies\tpp\subseteq r\et\ntppi\cup\eq\subseteq r\implies\tppi\subseteq r}$\tabularnewline
\hline 
$\QH$ & $\Prcc\cap\liste{r\in\RCAH}{\left(\eq\subseteq r\et r\cap\left(\tpp\cup\ntpp\cup\tppi\cup\ntppi\right)\neq\vide\right)\implies\po\subseteq r}$\tabularnewline
\hline 
$\CH$ & $\Prcc\cap\liste{r\in\RCAH}{\left(\ec\subseteq r\et r\cap\left(\tpp\cup\ntpp\cup\tppi\cup\ntppi\cup\eq\right)\neq\vide\right)\implies\po\subseteq r}$\tabularnewline
\hline 
$\PAs$ & $\left\{ \vide,<,=,>,<\cup>,\B_{\PA}\right\} $\tabularnewline
\hline 
\end{tabular}
\par\end{centering}
\begin{centering}
\begin{tabular}{|c|c|}
\hline 
\multicolumn{2}{|c|}{$r\in\RCAHs$}\tabularnewline
\hline 
\hline 
$\correctTable$$\ntpp$ $\po$ $\ec$ $\tpp$ & $\tppi$ $\tpp$ $\eq$ $\dc$ $\ec$ $\po$\tabularnewline
\hline 
$\correctTable$ $\tppi$ $\ntppi$ & $\po$ $\ec$\tabularnewline
\hline 
$\correctTable$$\ntpp$ $\tpp$ & $\po$ $\tppi$ $\dc$ $\ntppi$ $\ec$\tabularnewline
\hline 
$\correctTable$ $\po$ $\tpp$ & $\tppi$ $\ntppi$ $\tpp$ $\ntpp$ $\eq$ $\po$ $\ec$\tabularnewline
\hline 
$\correctTable$ $\dc$ & $\tppi$ $\ntppi$ $\tpp$ $\ntpp$ $\eq$ $\po$\tabularnewline
\hline 
$\correctTable$ $\tppi$ $\ntppi$ $\tpp$ $\eq$ $\dc$ $\ec$ $\po$ & $\po$ $\eq$\tabularnewline
\hline 
$\correctTable$$\tppi$ $\dc$ $\ec$ $\po$ & $\po$ $\tpp$ $\eq$\tabularnewline
\hline 
$\correctTable$ $\dc$ $\ec$ & $\po$ $\tppi$ $\eq$\tabularnewline
\hline 
$\correctTable$ $\eq$ & $\po$ $\tpp$ $\tppi$\tabularnewline
\hline 
$\correctTable$ $\po$ $\dc$ $\ec$ $\tpp$ & $\po$ $\tpp$ $\ntpp$ $\eq$\tabularnewline
\hline 
$\correctTable$ $\tppi$ $\po$ $\ntppi$ & $\po$ $\tppi$ $\ntppi$ $\eq$\tabularnewline
\hline 
$\correctTable$ $\tppi$ $\ntppi$ $\tpp$ $\eq$ $\po$ $\ec$ & $\po$ $\tpp$ $\tppi$ $\ntppi$\tabularnewline
\hline 
$\correctTable$ $\tpp$ & $\po$ $\tpp$ $\ntpp$ $\tppi$\tabularnewline
\hline 
$\correctTable$$\tppi$ $\po$ $\ec$ & $\po$ $\tpp$ $\ntpp$ $\tppi$ $\ntppi$\tabularnewline
\hline 
$\correctTable$ $\tppi$ & $\ec$ $\po$ $\eq$\tabularnewline
\hline 
$\correctTable$ $\tppi$ $\tpp$ $\ntpp$ $\eq$ $\dc$ $\ec$ $\po$ & $\ec$ $\po$ $\tpp$ $\eq$\tabularnewline
\hline 
$\correctTable$ $\tppi$ $\eq$ $\po$ $\ec$ $\tpp$ & $\ec$ $\po$ $\tppi$ $\eq$\tabularnewline
\hline 
$\correctTable$ $\po$ $\ntpp$ $\dc$ $\ec$ $\tpp$ & $\ec$ $\po$ $\tppi$ $\ntppi$ $\eq$\tabularnewline
\hline 
$\correctTable$$\tppi$ $\po$ & $\ec$ $\po$ $\tpp$ $\ntpp$ $\eq$\tabularnewline
\hline 
$\correctTable$ $\ntpp$ $\po$ $\tpp$ & $\ec$ $\po$ $\tpp$ $\tppi$\tabularnewline
\hline 
$\correctTable$ $\dc$ $\ec$ $\po$ & $\ec$ $\po$ $\tpp$ $\tppi$ $\ntppi$\tabularnewline
\hline 
$\correctTable$ $\tppi$ $\eq$ $\po$ $\ntppi$ $\tpp$ & $\ec$ $\po$ $\tpp$ $\ntpp$ $\tppi$\tabularnewline
\hline 
$\correctTable$ $\tppi$ $\ntpp$ $\eq$ $\po$ $\tpp$ & $\ec$ $\po$ $\tpp$ $\ntpp$ $\tppi$ $\ntppi$\tabularnewline
\hline 
$\correctTable$$\po$ $\ec$ $\tpp$ & $\dc$ $\ec$ $\po$ $\eq$\tabularnewline
\hline 
$\correctTable$ $\tppi$ $\tpp$ $\ntpp$ $\eq$ $\po$ $\ec$ & $\dc$ $\ec$ $\po$ $\tpp$ $\eq$\tabularnewline
\hline 
$\correctTable$ $\ntpp$ & $\dc$ $\ec$ $\po$ $\tppi$ $\eq$\tabularnewline
\hline 
$\correctTable$ $\ntppi$ & $\dc$ $\ec$ $\po$ $\tpp$ $\tppi$\tabularnewline
\hline 
$\correctTable$ $\tppi$ $\eq$ $\po$ $\tpp$ & $\dc$ $\ec$ $\po$ $\tppi$ $\ntppi$ $\eq$\tabularnewline
\hline 
$\correctTable$ $\ec$ & $\dc$ $\ec$ $\po$ $\tpp$ $\ntpp$ $\eq$\tabularnewline
\hline 
$\correctTable$$\tppi$ $\po$ $\ntppi$ $\ec$ & $\dc$ $\ec$ $\po$ $\tpp$ $\ntpp$ $\tppi$\tabularnewline
\hline 
$\correctTable$ $\ntppi$ $\ntpp$ $\dc$ $\ec$ $\tppi$ $\tpp$
$\eq$ $\po$ & $\dc$ $\ec$ $\po$ $\tpp$ $\tppi$ $\ntppi$\tabularnewline
\hline 
$\correctTable$ $\po$ & $\dc$ $\ec$ $\po$ $\tpp$ $\ntpp$ $\tppi$ $\ntppi$\tabularnewline
\hline 
\end{tabular}
\par\end{centering}
\begin{centering}
\par\end{centering}
\caption{Definitions of some classic subclasses.\label{tab:def_classic_subclass}}
\end{table}

\begin{table}
\begin{centering}
\begin{tabular}{|c|c|c|c|c|c|c|c|c|}
\hline 
$r$ & $\vide$ & $<$ & $=$ & $>$ & $<\cup=$ & $<\cup>$ & $=\cup>$ & $\B$\tabularnewline
\hline 
\hline 
$\hmax\left(r\right)$ & $\vide$ & $<$ & $=$ & $>$ & $<$ & $<\cup>$ & $>$ & $<\cup>$\tabularnewline
\hline 
\end{tabular}
\par\end{centering}
\caption{Definition of the refinement $\protect\hmax$: the refinement by the
basic relations of maximal dimension\label{tab:affinement_PA}}

\end{table}

\begin{defn}[\cite{renz1999maximal}]
\label{def:Affinement-RCC8} Let $\S$ be one of the following subclass
$\HH$, $\QH$, and $\CH$ and let $r\in\S$. 

The refinement of $\S$, denoted by $h_{\S}$, is defined by:
\[
h_{\S}(r)=\begin{cases}
r & \text{if }r\in\Br\\
\{\dc\} & \text{otherwise and if }\dc\in r\\
\{\ec\} & \text{otherwise and if }\ec\in r\text{ and }\S\neq\CH\\
\{\po\} & \text{otherwise and if }\po\in r\\
\{\ntpp\} & \text{otherwise and if }\ntpp\in r\text{ and }\S=\CH\\
\{\ntppi\} & \text{otherwise and if }\ntppi\in r\text{ and }\S=\CH\\
\{\tpp\} & \text{otherwise and if }\tpp\in r\\
\{\tppi\} & \text{otherwise and if }\tppi\in r\\
\vide & \text{otherwise}
\end{cases}
\]
\end{defn}

\subsubsection*{Minimality\label{subsec:Minimalit=0000E9-et-redondance}}

A related problem to the satisfiability problem is the \emph{minimality
problem}, also called minimal labeling problem, deductive closure
problem, and maximum information deduction problem~\cite{beek1990exact,gerevini2011computing,liu2012solving,chandra2005minimality,amaneddine2013minimal,bessiere1996global,gerevini1995computing,amaneddine2012path,long2015distributive}.
Calculate the \emph{minimal network} of a network $N$ consists in
determining the smallest network $M\subseteq N$ having the same solutions
as $N$. The network $M$ contains as much information as can be inferred
from $N$. The number of possible basic relations between each pair
of entities is thus minimal.
\begin{defn}
\label{def:minimalite}A network $N$ is\emph{ minimal\index{minimal network}}
when for all distinct $x,y\in\E$ and for all basic relation $b\in N^{xy}$,
there exists a satisfiable scenario $S\subseteq N$ such that $S^{xy}=b$.
\end{defn}

\begin{defn}
The \emph{minimal network of a network} $N$ is the network $M\subseteq N$
such that $M$ is minimal and that all solution $\{u_{x}\}_{x\in\E}$
of $N$ is solution of $M$.
\end{defn}

\begin{example}
A non-minimal network $N$ and a minimal network $M$ over the point
algebra are illustrated in Figure~\ref{fig:contre_ex_van_beek}.
In fact, $M$ is the minimal network of $N$. 
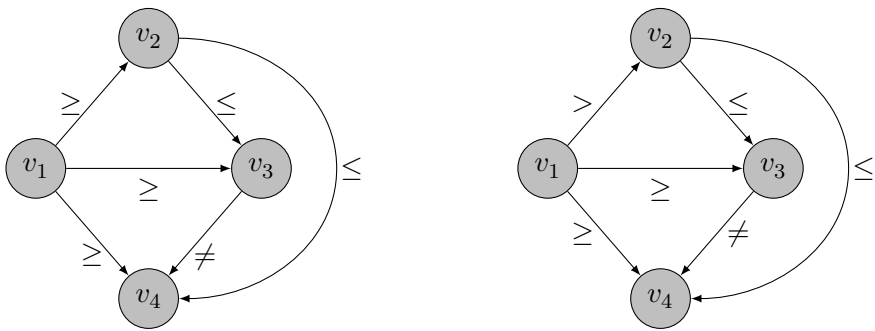
\begin{figure}
\begin{tikzpicture}
\node[draw,circle,fill=gray!50] (v1)at(0,0) {$v_1$};
\node[draw,circle,fill=gray!50] (v2)at(1.5,1.732) {$v_2$};
\node[draw,circle,fill=gray!50] (v3)at(3,0) {$v_3$};
\node[draw,circle,fill=gray!50] (v4)at(1.5,-1.732) {$v_4$};
\draw[->,>=latex] (v1) -- (v2) 
node[left, midway]{$\geq$};

\draw[->,>=latex] (v2) -- (v3) 
node[right, midway]{$\leq$};

\draw[->,>=latex] (v1) -- (v3) 
node[below, midway]{$\geq$};
 
\draw[->,>=latex] (v1) -- (v4) 
node[below, midway]{$\geq$};

\coordinate (n) at (4,0);
\draw[->,>=latex] (v2) to[out=0,in=90] (n) to[out=-90,in=-0] (v4);

\node[shift={(2mm,0)}] (label) at (n) {$\leq$};

\draw[->,>=latex] (v3) -- (v4) 
node[below, midway]{$\neq$};
\end{tikzpicture}~~~~~~~~~~~~~~~\begin{tikzpicture}
\node[draw,circle,fill=gray!50] (v1)at(0,0) {$v_1$};
\node[draw,circle,fill=gray!50] (v2)at(1.5,1.732) {$v_2$};
\node[draw,circle,fill=gray!50] (v3)at(3,0) {$v_3$};
\node[draw,circle,fill=gray!50] (v4)at(1.5,-1.732) {$v_4$};
\draw[->,>=latex] (v1) -- (v2) 
node[left, midway]{$>$};

\draw[->,>=latex] (v2) -- (v3) 
node[right, midway]{$\leq$};

\draw[->,>=latex] (v1) -- (v3) 
node[below, midway]{$\geq$};

\draw[->,>=latex] (v1) -- (v4) 
node[left, midway]{$\geq$};

\coordinate (n) at (4,0);
\draw[->,>=latex] (v2) to[out=0,in=90] (n) to[out=-90,in=-0] (v4);
\node[shift={(2mm,0)}] (label) at (n) {$\leq$};

\draw[->,>=latex] (v3) -- (v4) 
node[right, midway]{$\neq$};
\end{tikzpicture}
\centering{}\caption{A minimal network over $\protect\PA$ to the left and its minimal
network to the right ; the constraint $\protect\contrainte{v_{1}}={v_{2}}$
does not belong to any satisfiable scenario.\label{fig:contre_ex_van_beek}}
\end{figure}
Computing the minimal network consists of removing all the \emph{unrealizable}
basic relations. This operation thus refines all the relations as
much as possible, without changing the set of solutions. The algebraic
closure is thus a candidate of choice to solve this problem. Unfortunately,
in general, the algebraic closure does not compute the minimal network,
but only a network $N'$ such that
\[
M\subseteq N'\subseteq N\text{.}
\]
Note that computing the minimal network makes it possible to decide
satisfiability, since a network is satisfiable if and only if its
minimal network is not trivially inconsistent.
\end{example}

We are now interested in the subclasses whose algebraic closure calculates
the minimal network.
\begin{defn}
\label{def:classe_minimale}A subset $\S$ is \emph{minimal\index{minimal subset}}
when the algebraic closure of any network $N$ over $\S$ is the minimal
network of $N$.
\end{defn}

A minimal subclass is thus algebraically tractable. However, the converse
is not generally true. For example, the point algebra is not minimal.
See the following proposition for examples of minimal subclasses.
\begin{prop}[\cite{long2015distributive}]
\label{prop:sous-classes-classiques-id-scenarisables} The subclasses
$\PAs$ and $\RCAHs$ are minimal.
\end{prop}

\subsection{Combined Spatial Formalisms and Combined Temporal Formalisms\label{subsec:RW:Combined}}

Some research has recently been focusing on combining qualitative
formalisms. One of these combinations, which our framework encompasses,
is the \emph{loose integration} of two qualitative formalisms and
its \emph{biconstraint networks} (two networks having interdependencies)~\cite{westphal2008bipath}.
The satisfiability decision problem is then to decide whether there
is a solution satisfying both networks.
\begin{example}
\label{ex:STC} The loose integration of qualitative size and topology
of \cite{gerevini2002combining}, which we call $\STC$, describes
the relation between two regions both in terms of topology and in
terms of their relative size; e.g., ``$x$ and $y$ are disjoint
and the size of $x$ is smaller than that of $y$''.

To reason on $\STC$, \cite{gerevini2002combining} generalized the
path-consistency algorithm into the \emph{bipath-consistency} algorithm,
which enforces path-consistency, i.e. applies the algebraic closure,
on both networks while simultaneously propagating their interdependencies,
by using interdependency operators, described in Tables~\ref{tab:interd=0000E9pendances-de_QS_RCC}~and~\ref{tab:interd=0000E9pendances-de_RCC_QS}.
Subclasses for which bipath-consistency decides consistency have been
found for several combinations of formalisms~\cite{gerevini2002combining,li2012reasoning,cohn2014reasoning}.
\end{example}

\begin{lyxcode}
\begin{table}
\begin{centering}
{\footnotesize{}%
\begin{tabular}{|c|c|c|c|}
\hline 
{\footnotesize$\correctTable$$b$} & {\footnotesize$<$} & {\footnotesize$=$} & {\footnotesize$>$}\tabularnewline
\hline 
\hline 
{\footnotesize$\correctTable$$\conv{\RCAH}{\PA}b$} & {\footnotesize$\dc\cup\ec\cup\po\cup\tpp\cup\ntpp$} & {\footnotesize$\dc\cup\ec\cup\po\cup\eq$} & {\footnotesize$\dc\cup\ec\cup\po\cup\tppi\cup\ntppi$}\tabularnewline
\hline 
\end{tabular}}{\footnotesize\par}
\par\end{centering}
\caption{Interdependency operator from $\protect\PA$ to the algebra of $\protect\RCCH$,
$\protect\RCAH$.\label{tab:interd=0000E9pendances-de_QS_RCC}}
\end{table}
\end{lyxcode}
\begin{table}
\begin{centering}
\begin{tabular}{|c|c|c|c|c|c|c|c|c|}
\hline 
$\correctTable$$b$ & $\dc$ & $\ec$ & $\po$ & $\eq$ & $\tpp$ & $\ntpp$ & $\tppi$ & $\ntppi$\tabularnewline
\hline 
\hline 
$\correctTable$$\conv{\RCAH}{\PA}(b)$ & $\B_{\PA}$ & $\B_{\PA}$ & $\B_{\PA}$ & $=$ & $<$ & $<$ & $>$ & $>$\tabularnewline
\hline 
\end{tabular}
\par\end{centering}
\caption{Interdependency operator from the algebra of $\protect\RCCH$, $\protect\RCAH$,
to $\protect\PA$\label{tab:interd=0000E9pendances-de_RCC_QS}}
\end{table}

The framework introduced in this paper encompasses loose integrations
(generalized to $m$ formalisms) as specific combinations, but does
not cover all ways of combining formalisms. \emph{Tight integrations}~\cite{wolfl2009combinations}
are more expressive than loose integrations, at the cost of drastically
increasing the number of relations. Another combination is that of
\cite{Meiri1996}, which deals with \emph{heterogeneous} entities
that are points and intervals. The corresponding relations are relations
between a point and an interval, two intervals, and two points, respectively.
In this combination, whose complexity has been studied in depth in
\cite{jonsson2004complexity}, there is only one relation per pair
of entities; in constrast, loose integrations feature several relations
(from different formalisms) between the same entities, which increases
expressiveness by allowing the use of complementary relations. This
complementarity is the main asset of loose integration (and its major
difficulty). Note that there also exist combinations with non-qualitative
formalisms \cite{Meiri1996,bennett2002multi}.

\subsection{Spatio-Temporal Formalisms\label{subsec:RW:Spatio-Temporal}}

Spatio-temporal formalisms are also combinations, which integrate
space and time information in particular ways.

\cite{westphal2013transition} proposed a method to reason about temporal
sequences of spatial information, which actually share the same structure
as loose integrations and are thus covered by our framework. They
model such sequences as tuples of constraint networks, each corresponding
to a time instant. They introduce two kinds of solutions, depending
on the desired dynamics of entities (moving continuously) over time.
Our framework covers the weaker ``$T_{2}$-solutions'', which guarantee
that between successive instants of the sequence, for each pair of
entities, only the relation of the first instant and then the relation
of the second instant hold: there is no intermediary relations between
the instants of the sequence.

\begin{figure}[t]
\centering{}\includegraphics[width=1\columnwidth]{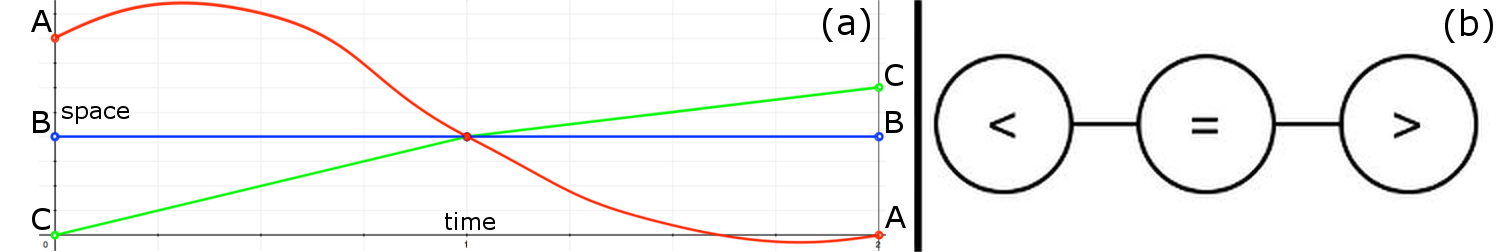} \caption{(a) An evolution of space points. (b) The neighborhood graph of the
point algebra.}
\label{fig:neighborhood}
\end{figure}

\begin{example}
\label{ex:TPC} The following is a temporal sequence of $3$ networks
describing spatial points moving along a line: $x\leq y<z$ at the
first instant, $x=y=z$ at the second, and then $x>y>z$ at the third
instant. This description has temporally continuous solutions without
intermediary relations between the instants, such as that of Figure~\ref{fig:neighborhood}~(a).
\end{example}

In fact, the $T_{2}$ condition forces relations at successive instants
to be ``neighbors'' according to the \emph{neighborhood graph} of
$\PA$~\cite{freksa1991conceptual}, shown in Figure~\ref{fig:neighborhood}~(b).
In this graph, for example, the only neighbor relation of ``$<$''
is ``$=$''. The $T_{2}$ condition thus ensures that, if $x<y$
at one instant, then at any neighbor instant, either $x<y$ or $x=y$.

\cite{gerevini2002qualitative} proposed a similar formalism based
on time intervals but with uncertainty on the scheduling of intervals,
so it is not encompassed by our framework.

\subsection{Multi-Scale Formalisms\label{subsec:RW:Multi-Scale}}

Multi-scale qualitative reasoning makes it possible to reason in the
presence of information having different precision (that is to say,
having different levels of detail). It turns out that we can consider
it as a combination, and that it is covered by our formal framework,
as we will see in the section~\ref{subsec:Semantics-of-Multi-algebra}.
In the context of multi-scale temporal reasoning, we consider several
discretizations of time, of different resolutions, called \emph{granularities}
or \emph{temporal scales}~\cite{euzenat2005time}. These scales are
most often partitions of the real line. They are formalized as preserving
order functions $g$ of $\mathbb{R}$ in $\mathbb{N}$. A scale $g$
associates with a point $x$ its \emph{representative} $g(x)$ at
this scale.
\begin{example}
The Gregorian calendar is an example of a set of scales. At the day
scale, the base unit is the day: the representative of an instant
at this scale is the day to which it belongs. As a result, at this
scale, all events occurring in the same day are considered to take
place at the same time, since they have the same representative.
\end{example}

One scale is said to be \emph{finer than} another if it has a ``better
resolution'': for example, in the Gregorian calendar, the month scale
is finer than that of the years - but that of the weeks is \emph{not}
finer than that of months (and vice versa), because any period of
time expressed in months can not be expressed exactly as a set of
weeks, as some weeks are indeed straddling two months. The set of
scales therefore forms a \emph{partial} order compared to the relation
``to be finer than''.

From this modeling of time, we can represent information defined at
different scales: two entities satisfy a constraint \emph{defined
at a scale} if the representatives of the two entities at this scale
satisfy the corresponding ``classical'' constraint. For example,
having $x<y$ at the minute scale means that the minute of $x$ is
before the minute in which $y$ occurs.

Conversion operators have been introduced to reason on multiple scales
to convert information from one scale to another~\cite{euzenat2005time}.
In the qualitative framework, the \emph{downward conversion} $\downgc$
allows propagating relations to a finer scale and \emph{upward conversion}
$\upgc$ allows propagating relations to a \emph{coarser} scale that
is to say less fine. Table~\ref{tab:conversions_PA} describes the
conversions of the point algebra. For example, since one instant may
be before another at the scale of seconds while taking place during
the same minute, the upward conversion verifies $\{=\}\in\upgc\{<\}$.
More precisely, $\upgc\{<\}=\{<,=\}$. Thus, if we know that $x<y$
at the scale of seconds, we can deduce that $x\leq y$ at the scale
of minutes. In a previous article~\cite{cohen2015algebra,cohen2015algebre},
we have proposed a formalization of qualitative multiscale temporal
reasoning in the context of the decision of satisfiability, formalization
which allows in particular to take into account the fact that an interval
can become a instant at a coarser scale and that an instant can become
an interval at a finer scale.

\begin{table}
\centering{}%
\begin{tabular}{cccc}
\toprule 
$b$ & $<$ & $=$ & $>$\tabularnewline
\midrule 
$\upgc b$ & $\leq$ & $=$ & $\geq$\tabularnewline
$\downgc b$ & $<$ & $\mathcal{B}_{\PA}$ & $>$\tabularnewline
\end{tabular}\caption{Upward and downward conversions of the point algebra\label{tab:conversions_PA}}
\end{table}

Note that there is also work on multiscale reasoning in the context
of spatial reasoning~\cite{euzenat2001granularity,li2007qualitative}.

\section{Symmetric Qualitative Formalisms\label{sec:Symmetrical-Qualitative-Formalis}}

Before introducing the multi-algebra framework, we must specify what
we call \emph{qualitative formalism} in this paper, in order to delimit
the formalisms concerned by the results of this formal framework.
We begin by introducing our own definition, that of \emph{symmetric}
qualitative formalisms, which includes some formalisms that are not
included in previous definitions and that are important in the context
of combinations. 

\subsection{Generalized Definition}

\cite{dylla2013algebraic} have been proposed a new definition of
qualitative formalism to include in particular the qualitative formalisms
with a \emph{weak converse}, i.e. invalidating the \emph{inversion
property} $\I(\sinv r)=\inv{\I(r)}$ of qualitative formalisms of
Ligozat and\emph{ }Renz \cite{ligozat2004qualitative}. However, this
definition does not include all the qualitative formalisms of the
literature, such as some formalisms found in the context of combinations.
As we will see later, the multi-scale formalisms and some spatio-temporal
combinations require a generalization of the definition of qualitative
formalism. However, we can not simply generalize the definition of
Dylla \emph{et al.}, because the framework of the multi-algebras we
propose does not apply, in its current form, to qualitative formalisms
with a weak converse (see section~\ref{subsec:Limitations}). We
then propose a definition of qualitative formalism, called symmetric,
on which we base our framework of multi-algebras. We call them symmetric
since the converse of a basic relation is a basic relation and the
converse of the converse of a relation is the same relation.

Note that the framework of multi-algebras should be generalized without
difficulty to weak converses. However, we do not place ourselves in
this general framework for two reasons. On the one hand, there are
only two formalisms with a weak converse in the literature. The first
\cite{goyal2000cardinal} is a fragment of the rectangle algebra (the
equivalent of Allen's algebra in 2D) and is therefore expressible
in it. The second is the calculation of the cardinal regions $\CRC$
\cite{navarrete2013spatial}. It does not have any identified tractable
subclass. Moreover, its combination with $\RCCH$ does not have any
tractable subclass containing the basic relations~\cite{liu2009combining}.
There is therefore no point for the moment in considering such a generalization.
On the other hand, considering a weak converse would lead to a significantly
more complicated framework.

We thus introduce, in the following definition, what we call ``qualitative
formalism'' in this paper. Informally, a symmetric qualitative formalism
is a formalism generalizing the definition of Ligozat and\emph{ }Renz,
having the same algebraic structure, but having a less restrictive
interpretation.
\begin{defn}
\label{def:Un-formalisme-qualitatif-symetrique}A \emph{symmetric
qualitative formalism} is a triplet $(\Aa,\U,\I)$, where $(\Aa,\sunion,\scompl{\cdot},\svide,\B,\scomp,\sinv{\cdot},\se)$
forms a finite non-associative algebra, $\U$ is a non-empty set,
the entity domain, and $\I$ is a function $\I\colon\Aa\to2^{\U\times\U}$,
the interpretation function, that satisfies, for all $r,r'\in\A$,
with $x\sinter y=\scompl{(\scompl x\sunion\scompl y)}$, the following
properties: \emph{\index{symmetric qualitative formalism}}
\begin{align*}
\begin{gathered}\I(\svide)=\vide\\
\I(\sinv r)=\inv{\I(r)}
\end{gathered}
 &  & \begin{gathered}\I(r\sinter r')=\I(r)\cap\I(r')\\
\I(r\sunion r')=\I(r)\cup\I(r')\\
\I(r\scomp r')\supseteq\bigl(\I(r)\comp\I(r')\bigr)\cap\I(\B)
\end{gathered}
\end{align*}
\end{defn}

The properties to be verified by the interpretation function are
indeed less restrictive than those of the Ligozat and Renz definition.
On the one hand, similar to the definition of \cite{dylla2013algebraic},
we no longer impose that the \emph{equality} relation of the algebra
to be interpreted as equality on $\U$. More formally, we no longer
require the satisfaction of $\I(\se)=\Id=\liste{(x,y)\in\U\times\U}{x=y}$.
The formalism on the qualitative sizes of the regions (example~\ref{ex:STC})
does not have equality on $\U$ and is therefore not a qualitative
formalism in the sense of Ligozat et Renz. As we will see later, qualitative
formalisms defined at a scale (as part of multi-scale reasoning) are
not qualitative formalisms in the sense of Ligozat and Renz, for the
same reason. Note that, contrary to our definition, the definition
of Dylla et al. does not require the presence of an \emph{algebraic
equality} (that is, a neutral element for the composition of $\Aa$).
Note also that all the qualitative formalisms of the literature, to
our knowledge, have an algebraic equality.

On the other hand, and contrary to the definition of Dylla et al.,
we include with this definition the formalisms whose basic relations
are not exhaustive on $\U$, that is to say, whose basic relations
do not form a partition of $\U\times\U$, but only a \emph{sub-partition}.
With such formalisms, certain pairs of entities do not correspond
to any basic relation. Therefore, basic relations do not allow to
represent all possible configurations of entities of $\U$.\index{sub-partition}

A dissociable example is $\RCCS$ \cite{gerevini2002combining}, that
corresponds to $\RCCH$ without the overlap relation. Its basic relations
obviously do not allow to represent all the possible configurations
of disks in the plane. Indeed, two overlapping disks do not even satisfy
the universal relation of $\RCCS$. The composition operator is also
more restrictive: if you have $\contrainte x{\un{\ec}}y$ and $\contrainte y{\un{\ntpp}}z$,
reasoning in $\RCCS$ allows to deduce that $\contrainte x{\deux{\ntpp}{\tpp}}z$,
which is not true for any regions $x,y,z$. However, all this is not
a problem if we know that the regions we consider cannot overlap,
which is the case, most of the time, in the context of applications
in geography (countries do not overlap) or in that of reasoning about
physical objects. The formalism $\DRA$ is another example of a subpartition~\cite{moratz2000qualitative}.
In these cases, algebra makes it possible to represent all the possible
configurations, and the reasoning is correct. The formalisms based
on sub-partitions thus serve to reason when certain basic relations
are never satisfied or, in other words, when the union of the basic
relations considered is not $\U\times\U$. Note that restrict $\U$
does not allow to formulate a formalism based on a sub-partition into
a formalism based on a partition, since the basic relations restrict
the \emph{couples} of considered entities. Thus, $\RCCS$ is actually
on the same universe as $\RCCH$. In addition, subclasses of the algebra
of a subpartition are not all subclasses of the algebra of the corresponding
partition (since their compositions differ). It is thus necessary
to study specifically the complexity of the subclasses of the subpartitions.

To include the formalisms based on subpartitions, we have relaxed
two conditions of the definition of Ligozat and Renz. On the one hand,
we no longer impose that the interpretation of the universal relation
$\B$ is necessarily $\U\times\U$. Note by $\Bw$ its interpretation
($\Bw=\I(\B)$): $\Bw$ is the set of considered couples of entities,
the only couples of entities that are assumed to be possible in the
world that we considered. In other words, in the context of a sub-partition,
we assume that all the couples of entities satisfy the relation $\Bw$.

On the other hand, we relaxed the condition $\I(r)\comp\I(r')\subseteq\I(r\scomp r')$
which guarantees that the abstract composition of the algebra is a
correct approximation of the true composition of relations on $\U$.
Since formalisms based on a sub-partition, such as $\RCCS$, do not
satisfy this condition, we replaced it with $(\I(r)\comp\I(r'))\cap\I(\B)\subseteq\I(r\scomp r')$.
Abstract composition can thus remove potentially valid couples of
entities from $\U\times\U$, as long as it does not remove couples
of entities from $\I(\B)=\Bw$, that is, of the subpartition. Assuming
that the universal relation $\Bw$ is verified by all the couples
of entities considered, which is generally the case when using a formalism
based on a subpartition, the composition is correct. There is no incorrect
inference by composition. Thus, the reasoning is carried out assuming
that the relation $\Bw$ is satisfied.

Note that the definition of weak composition \cite{renz2005weak}
is the same in this context. However, since $\Base$ is no longer
(necessarily) exhaustive, the weak composition is no longer (necessarily)
a superior approximation of the composition on $\U$, although it
is the superior approximation of $r\comp r'\cap\Bw$.

To summarize, the qualitative formalisms we consider in this paper
are those whose basic relations are disjoined two by two, but not
necessarily exhaustive and they have a \emph{strong} converse ($\begin{gathered}\sinv{\sinv r}=r\end{gathered}
$) and an algebraic equality (which is not necessarily interpreted
as equality on $\U$). 

\subsection{Properties of Symmetric Qualitative Formalisms\label{sect:def_formalismes_qualitatifs_sym}}

We are now interested in the properties of symmetric qualitative formalisms,
on which the proofs of the framework of multi-algebras are based.

We begin by defining the basic relations, that is to say the \emph{atoms}
of a non-associative algebra, as well as a certain number of notations.
\begin{defn}[\cite{hirsch2002relation}]
\label{def:atome}Let $\A$ be a non-associative algebra and $b\in\A$.

The relation $b$ is called an \emph{atom} or a \emph{basic relation}
of $\A$ if and only if for all $r\in\A$,\index{basic relation}
$r\sinter b=b\ou r\sinter b=\vide$.%
\end{defn}

\begin{defn}
\label{def:notation-appartenance_etc}Let $\A$ be a non-associative
algebra and $b,r,r'\in\A$. We use the following notations:
\begin{itemize}
\item We denote by $\Base$ the set of atomes of $\A$.\index{Base@$\Base$}
\item We denote $r\subseteq r'$ if and only if $r\cup r'=r'$. \index{relation refinining subseteq@\emph{relation refinining} $\subseteq$}
\item We denote $b\in r$ if and only if $b\subseteq r$ and $b$ is an
atom of $\A$.\index{in@$\in$}
\end{itemize}
\end{defn}

Note that we overloaded the membership notation, which is not ambiguous
in this article. Using these notations, we simply formulate the properties
of finite non-associative algebras and those of symmetric qualitative
formalisms.

To begin, we present the classical properties of finite non-associative
algebras.
\begin{lem}[\cite{hirsch2002relation}]
\label{lem:ppt-algebre-non-asso-finie}Let $\A$ be a finite non-associative
algebra and $r,r',r'',l,l'\in\A$. We have the following properties:
\[
r=\bigcup_{b\in r}b\qquad\qquad\qquad\sinv r=\sbigunion_{b\in r}\sinv b\qquad\qquad\qquad r\subseteq r'\implies\sinv r\subseteq\sinv{r'}
\]
\[
r\scomp r'=\sbigunion_{b\in r,b'\in r'}b\scomp b'\qquad\qquad\qquad(r\sinter r')\sunion r=r
\]
\[
r\subseteq r'\et r'\subseteq r''\implies r\subseteq r''\qquad\qquad\qquad r\subseteq r'\et l\subseteq l'\implies r\scomp l\subseteq r'\scomp l'
\]
\end{lem}

We are now interested in the properties of symmetric qualitative formalisms.
They satisfy most of the fundamental properties of qualitative formalisms
in the sense of the Ligozat and Renz definition, as shown in the following
lemma.
\begin{lem}
\label{lem:ppt-union_FQ-sym}Let $\F=(\A,\U,\I)$ be a symmetric qualitative
formalism and $r,r'\in\A$. We have the following properties:
\begin{center}
\[
\I(r)=\bigcup_{b\in r}\I(b)\qquad\qquad\qquad r\subseteq r'\implies\I(r)\subseteq\I(r')
\]
\par\end{center}

\end{lem}

\begin{proof}
We start by proving the first property. We have $\I(r)=\I(\bigcup_{b\in r}b)$,
by the lemme~\ref{lem:ppt-algebre-non-asso-finie}. Since $\F=(\A,\U,\I)$
is a symmetric qualitative formalism, we have $\I(r\sunion r')=\I(r)\cup\I(r')$
for all $r,r'\in\A$. Therefore, $\I(\bigcup_{b\in r}b)=\bigcup_{b\in r}\I(b)$.
Thus, $\I(r)=\bigcup_{b\in r}\I(b)$.

We prove the second property. suppose that $r\subseteq r'$. We have
by definition $r\sunion r'=r'$. Therefore, $\I(r\sunion r')=\I(r')$,
hence $\I(r)\cup\I(r')=\I(r')$. Thus, $\I(r)\subseteq\I(r')$.
\end{proof}

\section{The Framework of Multi-algebras: Representation, Reasoning, and Satisfiability\label{sec:Framework_Multi-algebras}}

This section introduces the core of the formal framework of multi-algebras.
In the next subsection, we define the \emph{multi-algebras}, objects
constituting the underlying structure of loose integrations, spatio-temporal
sequences, and multi-scale descriptions. We then show how to reason
with multi-algebra relations, and we give them formal semantics. Finally,
we define \emph{constraint networks} and \emph{algebraic closure}
in the context of multi-algebras.

\subsection{Projections and Multi-algebras}

We introduce in this section the basic blocks of the formal framework
of multi-algebras, starting with \emph{projection operators}. The
purpose of projection operators is to represent the interdependencies
between relations coming from different formalisms. A projection operator
simply associates with each relation of a formalism the relation that
corresponds to it in another formalism. Therefore, a projection operator
must distribute on the union and on the inversion. As a result, a
projection needs to be defined only on a subset of $\A$: the set
$\Base$ of basic relations.
\begin{defn}
\label{def:projection} Let $\A$ and $\A'$ be two non-associative
algebras. A \emph{projection operator}\index{projection operator}
is a function $\cconv\colon\A\to\A'$ which satisfies:
\begin{itemize}
\item $\forall r,r'\in\A\sep\cconv(r\sunion r')=\cconv r\sunion\cconv r'$.%
\item $\forall r\in\A\sep\cconv\sinv r=\sinv{\cconv r}$ .
\end{itemize}
\end{defn}

\begin{example}
Interdependency operators of loose integrations and the conversion
operators of multi-scale formalisms are examples of projections.
\end{example}

Note that a projection operator is an increasing function and that
it is calculated from the projections of the basic relations.
\begin{lem}
\label{lem:croissance_proj}Let $\cconv$ be a projection operator
from $\A$ to $\A'$ and let $r,r'\in\A$. We have the following properties:
\begin{itemize}
\item If $r\subseteq r'$ then $\cconv r\subseteq\cconv r'$,
\item $\cconv r=\sbigunion_{b\atome r}\cconv b$.
\end{itemize}
\end{lem}

\begin{proof}
We prove the first property. We suppose $r\sunion r'=r'$ and we show
$\cconv r\sunion\cconv r'=\cconv r'$. By the definition of a projection
operator, we have $\cconv r\sunion\cconv r'=\cconv(r\sunion r')$.
By hypothesis, we have thus $\cconv(r\sunion r')=\cconv r'$. Therefore,
$\cconv r\sunion\cconv r'=\cconv r'$.

We prove the second property. By Lemma~\ref{lem:ppt-algebre-non-asso-finie},
$\cconv r=\cconv\sbigunion_{b\atome r}b$. By definition of a projection
operator, we deduce that $\cconv\sbigunion_{b\atome r}b=\sbigunion_{b\atome r}\cconv b$.
Thus, $\cconv r=\sbigunion_{b\atome r}\cconv b$.
\end{proof}
We can now define multi-algebras, which are Cartesian products of
algebras -- each of them corresponding to one of the combined formalisms,
to one of the instants of the temporal sequence or to one of the scales
-- equipped with projection operators representing the interdependencies
between their relations.
\begin{defn}
\label{def:multi-algebre} A \emph{multi-algebra}\index{multi-algebra}
$\A$ is a Cartesian product of $m$ non-associative algebras $\{\A_{i}\}_{i=1}^{m}$
(with $m\in\mathbb{N}^{\ast}$), equipped with $m(m-1)$ projection
operators $\conv ij\colon\A_{i}\to\A_{j}$ (for all distinct $i,j\in\{1,\ldots,m\}$). 

A multi-algebra $\A$ is qualified as \emph{finite}, if each non-associative
algebra $\A_{i}$ is finite. 
\end{defn}

We identify $\A$ and $\A_{1}\times\cdots\times\A_{m}$. In particular,
$R\in\A$ means $R\in\A_{1}\times\cdots\times\A_{m}$. 

We call \emph{relations} the elements $R$ of $\A$, although they
are in fact $m$-tuples of relations ; $R_{i}$ denotes the (classical)
relation associated with $\A_{i}$ in $R$. 

Note that a multi-algebra where $m=1$ (a \guillemotleft{} mono-algebra
\guillemotright ) is exactly a non-associative algebra since it has
no projection operator. In the following, we denote by $\conv{\A_{i}}{\A_{j}}$
the projection operation $\conv ij$ when $\A_{i}$ and $\A_{j}$
are specified, for ease of reading.
\begin{defn}
A relation $R$ is called \emph{basic}\index{basic relation} when
all the relations $R_{i}$ are basic, i.e. $R\in\Base_{1}\times\dots\times\Base_{m}$. 

We denote by $\Base$ the set of basic relations of $\A$ : $\Base=\Base_{1}\times\dots\times\Base_{m}$,
\index{Base@$\Base$}

We denote by $\B$ the \emph{universal relation} of $\A$ : $\B=\ctrois{\B_{1}}{\ldots}{\B_{m}}$,
\index{B@$\B$}

We denote by $R\subseteq R'$ when $R_{i}\subseteq R'_{i}$ for each
$i\in\left\{ 1,\ldots,m\right\} $, \index{relation refinining subseteq@\emph{relation refinining} $\subseteq$}

Finally, we denote by $B\atome R$ when $B\subseteq R$ and $B\in\Base$.\index{atome@$\atome$}
\end{defn}

\begin{example}
\label{ex:QST_multialg} The multi-algebra corresponding to $\STC$
(see Example~\ref{ex:STC}) is the Cartesian product $\RCAH\times\PA$
of the algebra of $\RCCH$, denoted $\RCAH$, and of the point algebra
$\PA$ (here describing the relative sizes of the regions), equipped
with the interdependence operators of $\STC$ as projections (see
Tables~\ref{tab:interd=0000E9pendances-de_QS_RCC} et~\ref{tab:interd=0000E9pendances-de_RCC_QS}).
We call this multi-algebra $\STA$ (\emph{size and topology (multi)-algebra}).
One of its relation is $\cdeux{\tpp}{<\cup=}$. The projection of
$\tpp$ to $\PA$ is $<$ (what is denoted by $\conv{\RCAH}{\PA}\left(\tpp\right)=\mathord{<}$)
since $\tpp$ is the relation\emph{ tangential proper part} and that
a region strictly included in another is necessarily smaller. The
basic relations $B\in\cdeux{\tpp}{<\cup=}$ are $\cdeux{\tpp}<$ and
$\cdeux{\tpp}=$. Only $\cdeux{\tpp}<$ corresponds to a non-empty
set of pairs of regions.
\end{example}

\begin{example}
\label{ex:TPC_multialg} The Cartesian product $\PA^{m}$ can be used
to represent temporal sequences of relations between points on a line,
as in the example~\ref{ex:TPC}. Indeed, one can use the $i$-th
$\PA$ of the Cartesian product for the $i$-th instant ; $R_{i}$
is thus the relation at the instant $i$. We then equip $\PA^{m}$
with the projections imposing the respect of the neighborhood graph
$V$ of the figure~\ref{fig:neighborhood}. We thus have $\conv ij\left(<\right)=\left(\mathord{\mathord{<}\cup\mathord{=}}\right)$,
$\conv ij\left(\mathord{>}\right)=\left(\mathord{\mathord{>}\cup\mathord{=}}\right)$
and $\conv ij\left(\mathord{=}\right)=\B$ if the instants $i$ and
$j$ are neighbors (i.e. if $\left|i-j\right|=1$), and $\forall b\in\Base,\ \conv ijb=\B$
(i.e. no constraint) if the instants $i$ and $j$ are not neighbors.
We denote by $\TPA$ this multi-algebra (\emph{temporalized point
(multi-)algebra with neighborhood}). The Cartesian product of $\TPA$
is $\PA\times\PA\times\PA$ in the context of three instants, i.e.
when $m=3$. The relation $\ctrois{\mathord{<}\cup\mathord{=}}=>$
of this multi-algebra represents a possible sequence of relations
for the three instants.
\end{example}

\begin{example}
\label{ex:GPC_multialg} $\PA^{m}$ can also be used to represent
multi-scale temporal descriptions. Consider scales totally ordered
by the finesse relation $\preceq$. In this case, the $i$-th $\PA$
corresponds to the $i$-th scale, thus $R_{i}$ is the relation at
the scale $g_{i}$. For example, the relation $\cquatre <{\mathord{\mathord{<}\cup\mathord{=}}}{\mathord{\mathord{<}\cup\mathord{=}}}=$
of $\PA^{4}$ describes a possible temporal relation between two instants
at four different scales (seconds, minutes, hours, and days, for example).
This relation means that the first instant is before the second at
the scale of seconds, is not after this one at the scale of minutes
and at the scale of hours, and that they take place during the same
day. The projections of the associated multi-algebra, denoted $\SPA^{m}$
(\emph{scaled point (multi-)algebra}), correspond to the upward conversions
$\upgc$ and the downward conversion $\downgc$ introduced in the
section~\ref{subsec:RW:Multi-Scale}. More precisely, for all $b\in\Base_{\PA}$,
we have $\conv ijb=\upgc b$ and $\conv jib=\downgc b$ if the scale
$i$ is finer than the scale $j$ (i.e. if $g_{i}\preceq g_{j})$.
The basic relations $B\atome\cquatre <{\mathord{\mathord{<}\cup\mathord{=}}}{\mathord{\mathord{<}\cup\mathord{=}}}=$
are $\cquatre <<<=$, $\cquatre <=<=$, $\cquatre <<==$ and $\cquatre <===$.
\end{example}

\subsection{Reasoning about Multi-algebra Relations}

We can reason about the relations of a multi-algebra by applying the
classical rules component by component. For example, in $\STC$ (see
Example~\ref{ex:QST_multialg}), if $\contrainte x{\cdeux{\tpp}{\leq}}y$
and $\contrainte y{\cdeux{\dc}=}z$, then $\contrainte x{\cdeux{\dc}{\leq}}z$
(since $\tpp\wcomp\dc$ is $\dc$ and $\leq\wcomp=$ is $\leq$).
It is thus natural to introduce operators of composition $\scomp$,
intersection $\sinter$, converse $\sinv{\cdot}$, and union $\sunion$
on the relations of multi-algebras, which are based on the corresponding
operators of combined algebras and work simply component by component.
\begin{defn}
\label{def:multi-operateur}Let $\A$ be a multi-algebra. We call
respectively \emph{composition}\index{composition}, \emph{converse}\index{converse@\emph{converse}},
\emph{intersection\index{intersection}}, and \emph{union}\index{union},
the operators $\scomp$, $\sinv{\cdot}$, $\sinter$, and $\sunion$
on $\A$ defined, for all $R,R'\in\A$, by :
\begin{itemize}
\item $\forall i\sep(R\scomp R')_{i}=R_{i}\scomp R_{i}'$ ;
\item $\forall i\sep(\sinv R)_{i}=\sinv{(R_{i})}$ ;
\item $\forall i\sep(R\sinter R')_{i}=R_{i}\sinter R_{i}'$ ; 
\item $\forall i\sep(R\sunion R')_{i}=R_{i}\sunion R_{i}'$.
\end{itemize}
\end{defn}

Combination operators are useful for applying classical concepts to
our generalized framework; however, they are not sufficient for reasoning,
which also requires propagating interdependencies within each relation
through the use of projections. The following definition introduces
the property of being closed under projection and the closure by projection
of a multi-algebra relation.
\begin{defn}
\label{def:cloture-par-proj}Let $\A$ be a finite multi-algebra and
$R\in\A$.

The relation $R$ is \emph{closed under projection\index{closed under projection}}
if for all distinct $i,j\in\left\{ 1,\ldots,m\right\} $, the property
$R_{j}\subseteq\conv ijR_{i}$ is satisfied.

The \emph{\index{projection closure@\emph{projection closure}} projection
closure} of $R\in\A$, denoted by $\cconv R$, is obtained from $R$
by iteratively replacing each $R_{j}$ by $R_{j}\sinter(\conv ijR_{i})$
for all distinct $i,j\in\left\{ 1,\ldots,m\right\} $ until a fixed
point is reached.
\end{defn}

\begin{rem}
The result of a projection closure does not depend on the order of
the substitutions ``$R_{j}\leftarrow R_{j}\sinter(\conv ijR_{i})$''.
The operator is thus well defined.
\end{rem}

\begin{example}
\label{ex:proj closure} Let us return to the multi-algebra of $\STC$
introduced in the example~\ref{ex:QST_multialg}. Since we have $\conv{\RCAH}{\PA}\tpp=\mathord{<}$
and $\conv{\PA}{\RCAH}\mathord{<}=\tpp$, the projection closure of
$\cdeux{\tpp}{\leq}$, $\cconv\cdeux{\tpp}{\leq}$, is $\cdeux{\tpp}<$.
For the same reason, that of $\cdeux{\tpp}>$ is $\cconv\cdeux{\tpp}>=(\vide,\vide)$,
which proves that this relation is not feasible. Indeed, a region
cannot be inside another when it has a larger size.

For the multi-algebra $\TPA$ (example~\ref{ex:TPC_multialg}) with
$m=3$, the projection closure of $\ctrois <{\neq}>$ is $\ctrois{\vide}{\vide}{\vide}$.
This proves that this temporal sequence of relations between two points
at $3$ successive instants is not feasible. Indeed, two points on
a straight line cannot invert their relative position continuously
without there being a time when they are equal. Note, however, that
the relations $\cdeux <{\neq}$ and $\cdeux{\neq}>$ ($\TPA$ with
$m=2$) are each independently feasible. Thus, the \emph{local consistency}
of a relation does not ensure its \emph{global consistency}.
\end{example}

\subsection{Semantics of Multi-algebra Relations\label{subsec:Semantics-of-Multi-algebra}}

Thanks to these operators, we can now give to the relations of multi-algebras
their own semantics, adopting an approach similar to the classical
framework. We qualify the multi-algebras endowed with this semantics
as \textquotedblleft \emph{sequential}\textquotedblright{} formalisms
because their relations are sequences of classical relations.
\begin{defn}
\label{defn:combined formalisms} A \emph{(qualitative) sequential
formalism }\footnote{We used the term \textquotedbl loosely combined formalism\textquotedbl{}
when we introduced this definition in the literature~\cite{cohen2017checking,cohen2017decision}.}\emph{\index{sequential formalism}} is a triple $(\A,\U,\I)$, where
$\A$ is a finite multi-algebra, $\U$ is a non-empty set, the \emph{universe},
and $\I$ is a function $\I\colon\A\to2^{\U\times\U}$, the \emph{interpretation
function}, satisfying, for all $R,R'\in\A$, the following properties:
\begin{align*}
\begin{gathered}\begin{aligned}\I(\cconv R) & =\I(R)\\
\I(\sinv R) & =\overline{\I(R)}
\end{aligned}
\\
\I\bigl(\ctrois{\svide}{\ldots}{\svide}\bigr)=\vide
\end{gathered}
 &  & \begin{aligned}\I(R\scomp R') & \supseteq\bigl(\I(R)\circ\I(R')\bigr)\cap\I\left(\B\right)\\
\I(R\sinter R') & =\I(R)\cap\I(R')\\
\I(R) & =\bigcup_{B\atome R}\I(B)\text{.}
\end{aligned}
\end{align*}
\end{defn}

This definition is a natural generalization of the notion of symmetric
qualitative formalism that we proposed in Section~\ref{sec:Symmetrical-Qualitative-Formalis},
with one exception. It is adapted for $m$-tuples of relations, and
the condition on $\I(\cconv R)$ was added to ensure that the projection
is correct (i.e. it does not remove valid entities pairs). The tricky
point is that the property $\I(R\sunion R')=\I(R)\cup\I(R')$ has
been replaced by $\I(R)=\bigcup_{B\atome R}\I(B)$, property which
is equivalent in the context of finite non-associative algebras,
but which is not equivalent in the framework of multi-algebras. In
fact, the property $\I(R\sunion R')=\I(R)\cup\I(R')$ is usually invalidated
(we will see it in Example~\ref{ex:QST_formalisme}). 

In the sequel, as in the classical setting, we denote by $\contrainte xRy$
when $(x,y)\in\I(R)$.
\begin{example}
\label{ex:TPC_formalisme} We call \emph{temporalized point calculus
by the neighborhood $\TPC$\index{temporalized point calculus TPC@temporalized point calculus $\TPC$}}
the sequential formalism of Example~\ref{ex:TPC}, allowing us to
represent temporal sequences over $\PA$ in the context of continuity
without intermediary relation. The multi-algebra of $\TPC$ is the
one described in Example~\ref{ex:TPC_multialg}, i.e. $\TPA$ with
$\PA^{m}$ as Cartesian product for sequences of length $m$. Its
semantics require choosing the value $\tau_{i}\in\mathbb{R}$ of each
instant $i$ of the sequence. The universe $\U$ is then the set of
continuous evolutions of points on the real line during the sequence,
i.e. the set $\mathcal{C}^{0}([\tau_{0},\tau_{m}])$ of continuous
functions from $[\tau_{0},\tau_{m}]$ to $\mathbb{R}$. The interpretation
of a relation of $\TPC$ is the set of pairs of continuously evolving
points on $\mathbb{R}$ over the time of the sequence, satisfying
at each key instant (each $\tau_{i}$) the corresponding relation
and satisfying no other relation between these instants. More formally,
noting $\I_{\PA}$ the classical interpretation function of the point
algebra on the real line, and $\mathtt{B}(p,p')$ the basic relation
of $\PA$ verified by two points $p$ and $p'$ of $\mathbb{R}$,
the interpretation of a relation $R$ by $\I$ is the set of pairs
$(x,y)$ of functions from $\mathcal{C}^{0}([\tau_{0},\tau_{m}])$
which satisfies the following conditions. 
\begin{itemize}
\item The relations at key instants are satisfied:
\[
\forall i\in\{1,\ldots,m\}\sep\bigl(x(\tau_{i}),y(\tau_{i})\bigr)\in\I_{\PA}(R_{i})\text{.}
\]
\item Between two successive key instants $\tau_{i}$ and $\tau_{i+1}$,
there is an instant of transition $\tau$ such that before $\tau$
the satisfied relation is that of the previous key instant, and after
$\tau$ the satisfied relation is that of the next key instant:
\[
\forall i\in\{1,\ldots,m-1\}\minisep\exists\tau\in[\tau_{i},\tau_{i+1}]\bigwedge\left\{ \begin{aligned} & \forall t\in\left[\tau_{i},\tau\right[\sep\mathtt{B}(x(t),y(t))=\mathtt{B}(x(\tau_{i}),y(\tau_{i}))\\
 & \forall t\in\left]\tau,\tau_{i+1}\right]\sep\mathtt{B}(x(t),y(t))=\mathtt{B}(x(\tau_{i+1}),y(\tau_{i+1}))\\
 & \mathtt{B}(x(\tau),y(\tau))\in\{\mathtt{B}(x(\tau_{i}),y(\tau_{i})),\mathtt{B}(x(\tau_{i+1}),y(\tau_{i+1}))\}\text{.}
\end{aligned}
\right.
\]
\end{itemize}
Note that $\TPC$ is based on a sub-partition (see Section~\ref{sec:Symmetrical-Qualitative-Formalis}).
Indeed, a pair of any continuous functions does not always satisfy
a basic relation of $\TPC$ for a fixed choice of the $\tau_{i}$.
But, for any pair of continuous functions, there exists a value for
$m$ and a set of $\tau_{i}$ such that the pair satisfy a basic relation
of $\TPC$.
\end{example}

Loose integration (generalized to $m$ formalisms) is a particular
case of sequential formalism, built from existing formalisms, in which
the projections and the interpretation function have a certain form.
\begin{defn}
\label{def:int=0000E9gration-l=0000E2che}\emph{The loose integration\index{loose integration@\emph{loose integration}}}
of $m$ symmetric qualitative formalisms on the same universe $\U$,
$(\mathcal{\A}_{1},\U,\I_{1}),\dots,(\mathcal{\A}_{m},\U,\I_{m})$,
is the triple $(\A,\U,\I)$, where $\A$ is the multi-algebra whose
Cartesian product is $\A_{1}\times\dots\times\A_{m}$ verifying the
following property, for all distinct $i,j\in\{1,\ldots,m\}$ and $b\in\Base{}_{i}$
: 
\[
\conv ijb=\sbigunion\liste{b'\in\Base_{j}}{\I_{i}(b)\cap\I_{j}(b')\neq\vide}
\]
and where $\I$ is defined by $\forall R\in\A\sep\I(R)=\bigcap_{i}\I_{i}(R_{i})$.
\end{defn}

We are going to show that a loose integration is indeed a sequential
formalism. We introduce for that the following lemma which allows
us to prove that certain formalisms of different natures are also
sequential formalisms.
\begin{namedthm}[Combination Lemma]
\label{lem:formalisme-decompos=0000E9}Let $\{(\A_{i},\U,\I_{i})\}_{i=1}^{m}$
be a set of symmetric qualitative formalisms over the same universe
and let $\mathfrak{R}'$ be a symmetric relation ($\mathfrak{R}'=\inv{\mathfrak{R}'})$.
Let, for each distinct $i,j\in\{1,\ldots,m\}$, $\conv ij$ be a projection
from $\A_{i}$ to $\A_{j}$ satisfying $\I_{i}(b)\cap\I_{j}(\B_{j})\cap\mathfrak{R}'\subseteq\I_{j}(\conv ijb)$
for all $b\in\B_{i}$.

The triplet $(\A,\U,\I)$ where 
\begin{itemize}
\item $\A$ is the following multi-algebra: $\A_{1}\times\cdots\times\A_{m}$
equipped with the projections $\conv ij$ with distinct $i,j\in\{1,\ldots,m\}$,
\item $\I$ is the function from $\A$ to $2^{\U\times\U}$ defined by $\I(R)=\bigcap_{i=1}^{m}\I_{i}(R_{i})\cap\mathfrak{R}'$,
\end{itemize}
is a sequential formalism.
\end{namedthm}
\begin{proof}
Before proving that the triple $(\A,\U,\I)$ is a sequential formalism,
we will show that the projections of $\A$ satisfy $\I_{i}(r)\cap\I_{j}(\B_{j})\cap\mathfrak{R}'\subseteq\I_{j}(\conv ijr)$
for all $r\in\A_{i}$ and for all distinct $i,j\in\{1,\ldots,m\}$.
Let distinct $i,j\in\{1,\ldots,m\}$ and let $r\in\A_{i}$. Given
that $(\A_{i},\U,\I_{i})$ is a symmetric qualitative formalism (Definition~\ref{def:Un-formalisme-qualitatif-symetrique})
and that $\conv ij$ is a projection (Definition~\ref{def:projection}),
we have the following inequality:

\[
\begin{array}{ccccc}
\I_{i}(r)\cap\I_{j}(\B_{j})\cap\mathfrak{R}' & = & \I_{i}(\bigcup_{b\in r}b)\cap\I_{j}(\B_{j})\cap\mathfrak{R}' & \text{since} & r=\sbigunion_{b\in r}b\text{ (Lem.\,\ref{lem:ppt-algebre-non-asso-finie})}\\
 & = & (\bigcup_{b\in r}\I_{i}(b))\cap\I_{j}(\B_{j})\cap\mathfrak{R}' & \text{since} & \I_{i}(r\sunion r')=\I_{i}(r)\cup\I_{i}(r')\\
 & = & \bigcup_{b\in r}(\I_{i}(b)\cap\I_{j}(\B_{j})\cap\mathfrak{R}') & \text{by} & \text{distributivity}\\
 & \subseteq & \bigcup_{b\in r}\I_{j}(\conv ijb) & \text{since} & \I_{i}(b)\cap\I_{j}(\B_{j})\cap\mathfrak{R}'\subseteq\I_{j}(\conv ijb)\\
 & = & \I_{j}(\bigcup_{b\in r}\conv ijb) & \text{since} & \I_{i}(r\sunion r')=\I_{i}(r)\cup\I_{i}(r')\\
 & = & \I_{j}(\conv ijr) & \text{since} & \conv ij(r\cup r')=\conv ijr\cup\conv ijr'\text{.}
\end{array}
\]

We now show that the triple $(\A,\U,\I)$ constitutes a sequential
formalism. $\A$ is indeed a multi-algebra, $\U$ is a non-empty set,
and $\I$ is a function from $\A$ to $2^{\U\times\U}$. We now check
the properties of $\I$. 
\begin{itemize}
\item $\I(\cconv R)=\I(R)$ :
\end{itemize}
Let distinct $i,j\in\{1,\ldots,m\}$. We show that closing under the
projections $\conv ij$ does not change the interpretation:

\[
\begin{array}{ccccc}
\I(R) & = & \bigcap_{k=1}^{m}\I_{k}(R_{k})\cap\mathfrak{R}'\\
 & = & \bigcap_{k=1}^{m}\I_{k}(R_{k})\cap\I_{j}(\conv ijR_{i})\cap\mathfrak{R}' & \text{since} & \I_{i}(R_{i})\cap\I_{j}(R_{j})\cap\mathfrak{R}'\\
 &  &  &  & \subseteq\I_{i}(R_{i})\cap\I_{j}(\B_{j})\cap\mathfrak{R}'\\
 &  &  &  & \subseteq\I_{j}(\conv ijR_{i})\\
 & = & \bigcap_{k\in\{1,\ldots,m\}\backslash\{j\}}\I_{k}(R_{k})\cap\I_{j}(R_{j})\cap\I_{j}(\conv ijR_{i})\cap\mathfrak{R}'\\
 & = & \bigcap_{k\in\{1,\ldots,m\}\backslash\{j\}}\I_{k}(R_{k})\cap\I_{j}(R_{j}\cap\conv ijR_{i})\cap\mathfrak{R}' & \text{since} & \I_{j}(r)\cap\I_{j}(r')=\I_{j}(r\cap r')\\
 & = & \I((R_{1},\ldots,R_{j-1},R_{j}\cap\conv ijR_{i},R_{j+1},\ldots,R_{m})) & \text{since} & \I(R)=\bigcap_{i}\I_{i}(R_{i})\cap\mathfrak{R}'\text{.}
\end{array}
\]
Thus, repeating this operation until a fixed point is reached, i.e.
closing under projection (Definition~\ref{ex:proj closure}) does
not change the interpretation. Therefore, $\I(R)=\I(\cconv R)$.
\begin{itemize}
\item $\I(\sinv R)=\overline{\I(R)}$ :
\end{itemize}
\[
\begin{array}{ccccc}
\I(\sinv R) & = & \bigcap_{i}\I_{i}((\sinv R)_{i})\cap\mathfrak{R}' & \text{since} & \I(R)=\bigcap_{i}\I_{i}(R_{i})\cap\mathfrak{R}'\\
 & = & \bigcap_{i}\I_{i}(\sinv{(R_{i})})\cap\mathfrak{R}' & \text{by} & \text{the definition\,\ref{def:multi-operateur}}\\
 & = & \bigcap_{i}\inv{\I_{i}(R_{i})}\cap\mathfrak{R}' & \text{since} & \I_{i}(\sinv r)=\inv{\I_{i}(r)}\\
 & = & \bigcap_{i}\inv{\I_{i}(R_{i})}\cap\inv{\mathfrak{R}'} & \text{by} & \text{property of }\mathfrak{R}'\\
 & = & \inv{\bigcap_{i}\I_{i}(R_{i})\cap\mathfrak{R}'} & \text{since} & \text{converse distributes on intersection}\\
 & = & \inv{\I(R)} & \text{\text{since}} & \I(R)=\bigcap_{i}\I_{i}(R_{i})\cap\mathfrak{R}'\text{.}
\end{array}
\]
\begin{itemize}
\item $\I\bigl(\ctrois{\svide}{\ldots}{\svide}\bigr)=\vide$ :
\end{itemize}
$\I\bigl(\ctrois{\svide}{\ldots}{\svide}\bigr)=\bigcap_{i}\I_{i}(\svide)\cap\mathfrak{R}'=\bigcap_{i}\vide\cap\mathfrak{R}'=\vide$
since each $(\A_{i},\U,\I_{i})$ is a qualitative formalism (Definition~\ref{def:Un-formalisme-qualitatif-symetrique}).
\begin{itemize}
\item $\I(R\scomp R')\supseteq\bigl(\I(R)\circ\I(R')\bigr)\cap\I(\ctrois{\B_{1}}{\ldots}{\B_{m}})$
:
\end{itemize}
We prove this property from the properties of symmetric qualitative
formalisms (Definition~\ref{def:Un-formalisme-qualitatif-symetrique})
:

\[
\begin{array}{cccc}
 & \left(\I(R)\circ\I(R')\right)\cap\I\left(\ctrois{\B_{1}}{\ldots}{\B_{m}}\right)\\
\subseteq & \left(\left(\bigcap_{i}\I_{i}(R_{i})\cap\mathfrak{R}'\right)\circ\left(\bigcap_{j}\I_{j}(R'_{j})\cap\mathfrak{R}'\right)\right)\cap\bigcap_{i}\I_{i}(\B_{i})\cap\mathfrak{R}' & \text{since} & \I(R)=\bigcap_{i}\I_{i}(R_{i})\cap\mathfrak{R}'\\
\subseteq & \left(\left(\bigcap_{i}\I_{i}(R_{i})\right)\circ\left(\bigcap_{j}\I_{j}(R'_{j})\right)\right)\cap\bigcap_{i}\I_{i}(\B_{i})\cap\mathfrak{R}'\\
\subseteq & \bigcap_{i}\left(\I_{i}(R_{i})\circ\bigcap_{j}\I_{j}(R'_{j})\right)\cap\bigcap_{i}\I_{i}(\B_{i})\cap\mathfrak{R}' & \text{since} & (r\cap s)\comp t\subseteq(r\comp t)\cap(s\comp t)\\
\subseteq & \bigcap_{i}\bigcap_{j}\left(\I_{i}(R_{i})\circ\I_{j}(R'_{j})\right)\cap\bigcap_{i}\I_{i}(\B_{i})\cap\mathfrak{R}' & \text{since} & t\comp(r\cap s)\subseteq(t\comp r)\cap(t\comp s)\\
\subseteq & \bigcap_{i}\left(\I_{i}(R_{i})\circ\I_{i}(R'_{i})\right)\cap\bigcap_{i}\I_{i}(\B_{i})\cap\mathfrak{R}'\\
\subseteq & \bigcap_{i}\left(\left(\I_{i}(R_{i})\circ\I_{i}(R'_{i})\right)\cap\I_{i}(\B_{i})\right)\cap\mathfrak{R}'\\
\subseteq & \bigcap_{i}\I_{i}(R_{i}\scomp R_{i}')\cap\mathfrak{R}' & \text{since} & \left(\I_{i}(r)\circ\I_{i}(r')\right)\cap\I_{i}(\B_{i})\subseteq\I_{i}(r\scomp r')\\
\subseteq & \bigcap_{i}\I_{i}\left((R\scomp R')_{i}\right)\cap\mathfrak{R}' & \text{by} & \text{the definition of }\scomp\text{ (déf.\,\ref{def:multi-operateur})}\\
\subseteq & \I(R\scomp R') & \text{since} & \I(R)=\bigcap_{i}\I_{i}(R_{i})\cap\mathfrak{R}'\text{.}
\end{array}
\]

\begin{itemize}
\item $\I(R\sinter R')=\I(R)\cap\I(R')$ :
\end{itemize}
\[
\begin{array}{ccccc}
\I(R\sinter R') & = & \bigcap_{i}\I_{i}((R\sinter R')_{i})\cap\mathfrak{R}' & \text{since} & \I(R)=\bigcap_{i}\I_{i}(R_{i})\cap\mathfrak{R}'\\
 & = & \bigcap_{i}\I_{i}(R_{i}\sinter R'_{i})\cap\mathfrak{R}' & \text{by} & \text{the definition of }\sinter\text{ (Def.\,\ref{def:multi-operateur})}\\
 & = & \bigcap_{i}\I_{i}(R_{i})\cap\I_{i}(R'_{i})\cap\mathfrak{R}' & \text{since} & \I_{i}(r\sinter r')=\I_{i}(r)\cap\I_{i}(r')\\
 & = & \bigcap_{i}\I_{i}(R_{i})\cap\mathfrak{R}'\cap\bigcap_{i}\I_{i}(R'_{i})\cap\mathfrak{R}'\\
 & = & \I(R)\cap\I(R') & \text{since} & \I(R)=\bigcap_{i}\I_{i}(R_{i})\cap\mathfrak{R}'\text{.}
\end{array}
\]
\begin{itemize}
\item $\I(R)=\bigcup_{B\atome R}\I(B)$ :
\end{itemize}
\[
\begin{array}{ccccc}
\I(R) & = & \bigcap_{i}\I(R_{i})\cap\mathfrak{R}' & \text{by} & \text{definition}\\
 & = & \bigcap_{i}\bigcup_{b_{i}\atome R_{i}}\I(b_{i})\cap\mathfrak{R}' & \text{since} & \I_{i}(r\sunion r')=\I_{i}(r)\cup\I_{i}(r')\\
 & = & \bigcup\liste{\I(b_{1})\cap\dots\cap\I(b_{m})}{b_{1}\in R_{1},\ldots,b_{m}\in R_{m}}\cap\mathfrak{R}' & \text{by} & \text{distributivity}\\
 & = & \bigcup\liste{\I(B_{1})\cap\dots\cap\I(B_{m})}{B\in R}\cap\mathfrak{R}' & \text{by} & \text{the definition of \textquotedblleft\ }B\in R\text{ \textquotedblright}\\
 & = & \bigcup_{B\atome R}\bigcap_{i}\I(B_{i})\cap\mathfrak{R}'\\
 & = & \bigcup_{B\atome R}(\bigcap_{i}\I(B_{i})\cap\mathfrak{R}')\\
 & = & \bigcup_{B\atome R}\I(B) & \text{since} & \I(R)=\bigcap_{i}\I(R_{i})\cap\mathfrak{R}'\text{.}
\end{array}
\]
The triplet $(\A,\U,\I)$ is thus a sequential formalism.
\end{proof}
We now show that loose integrations are sequential formalisms, based
on the previous lemma, proving that their projections are correct.
\begin{lem}
\label{lem:projections_int=0000E9gration_saines}Let $(\A,\U,\I)$
be the loose integration of $m$ symmetric qualitative formalisms
on the same universe $\U$, $(\mathcal{\A}_{1},\U,\I_{1}),\dots,(\mathcal{\A}_{m},\U,\I_{m})$,
and let distinct $i,j\in\{1,\ldots,m\}$ and $r\in\A_{i}$.

The property $\I_{i}(r)\cap\I_{j}(\B_{j})\subseteq\I_{j}(\conv ijr)$
is satisfied.
\end{lem}

\begin{proof}
Let distinct $i,j\in\{1,\ldots,m\}$ and $r\in\A_{i}$. Let $(x,y)\in\I_{i}(r)\cap\I_{j}(\B_{j})$.
There exists $b\in r$ and $b'\in\Base_{j}$ such that $(x,y)\in\I_{i}(b)\cap\I_{j}(b')$,
since $\I_{i}(r)=\bigcup_{b\in r}\I_{i}(b)$ and $\I_{j}(\B_{j})=\bigcup_{b'\in\Base_{j}}\I_{j}(b')$
(Lemma~\ref{lem:ppt-union_FQ-sym}). Given that $\conv ijb=\sbigunion\liste{b'\in\Base_{j}}{\I_{i}(b)\cap\I_{j}(b')\neq\vide}$
(Definition~\ref{def:int=0000E9gration-l=0000E2che}), we have $b'\in\conv ijb$.
Thus $\I_{j}(b')\subseteq\I_{j}(\conv ijb)$ (Lemma~\ref{lem:ppt-union_FQ-sym}).
But $(x,y)\in\I_{i}(b)\cap\I_{j}(b')$, hence $(x,y)\in\I_{j}(\conv ijb)$.
Since $b\in r$, we have $\conv ijb\subseteq\conv ijr$ (Lemma~\ref{lem:croissance_proj}).
Therefore, $\I_{j}(\conv ijb)\subseteq\I_{j}(\conv ijr)$ (Lemma~\ref{lem:ppt-union_FQ-sym}).
Thus $(x,y)\in\I_{j}(\conv ijr)$, d'où $\I_{i}(r)\cap\I_{j}(\B_{j})\subseteq\I_{j}(\conv ijr)$.
\end{proof}
\begin{prop}
\label{prop:L'int=0000E9gration-l=0000E2che_l=0000E2chement-combin=0000E9}The
loose Integration $(\A,\U,\I)$ of $m$ symmetrical qualitative formalisms
on the same universe $\U$, $(\mathcal{\A}_{1},\U,\I_{1}),\dots,(\mathcal{\A}_{m},\U,\I_{m})$,
is a sequential formalism.
\end{prop}

\begin{proof}
We apply the combination lemma~\ref{lem:formalisme-decompos=0000E9}
with $\mathfrak{R}'=\U\times\U$, which trivially verifies $\mathfrak{R}'=\inv{\mathfrak{R}'}$.
We actually have $\{(\A_{i},\U,\I_{i})\}_{i=1}^{m}$ a set of symmetrical
qualitative formalisms on the same universe. The function $\I$ is
indeed a function from $\A$ to $2^{\U\times\U}$ satisfying $\I(R)=\bigcap_{i=1}^{m}\I_{i}(R_{i})=\bigcap_{i=1}^{m}\I_{i}(R_{i})\cap\U\times\U=\bigcap_{i=1}^{m}\I_{i}(R_{i})\cap\mathfrak{R}'$.
To conclude, we have
\[
\I_{i}(r)\cap\I_{j}(\B_{j})\cap\mathfrak{R}'=\I_{i}(r)\cap\I_{j}(\B_{j})\subseteq\I_{j}(\conv ijr)
\]
 for all $r\in\A_{i}$, for all distinct $i,j\in\{1,\ldots,m\}$ (Lemma~\ref{lem:projections_int=0000E9gration_saines}). 
\end{proof}
\begin{example}
\label{ex:QST_formalisme} The size-topology combination\index{size-topology combination STC@size-topology combination $\STC$}
$\STC$ (example~\ref{ex:STC}) is the loose integration of the formalism
$\RCCH$ and the formalism $\PA$ interpreted over the region sizes
(i.e., $\U$ is the set of measurable regions of $\mathbb{R}^{n}$
for a given measure $\mu$, $\varphi\left(<\right)=\liste{\left(r,r'\right)\in\U^{2}}{\mu\left(r\right)<\mu\left(r'\right)}$,
and $\varphi\left(=\right)=\liste{\left(r,r'\right)\in\U^{2}}{\mu\left(r\right)=\mu\left(r'\right)}$
). Its multi-algebra is that of Example~\ref{ex:QST_multialg}. Take
relations $\cdeux{\tpp}{\vide}$ and $\cdeux{\vide}<$ of $\STC$,
each of them has a trivially empty interpretation. However, the interpretation
of the relation $\cdeux{\tpp}{\vide}\cup\cdeux{\vide}<=\cdeux{\tpp}<$
is not the empty set. The classic property $\I(R\sunion R')=\I(R)\cup\I(R')$
is invalidated. 
\end{example}

\begin{example}
\label{ex:GPC_formalisme} We call \emph{multi-scale point calculus\index{multi-scale point calculus SPC@\emph{multi-scale point calculus $\SPC$}},}
denoted by \emph{$\SPC$}, the sequential formalism allowing us to
represent multi-scale temporal descriptions based on the point algebra
(see Section~\ref{subsec:RW:Multi-Scale}). We place ourselves within
the framework of a calendar where the scales are totally ordered by
the fineness relation $\preceq$. For $m$ scales, the multi-algebra
of $\SPC$ is simply $\SPA$ (see Example~\ref{ex:GPC_multialg})
with $\PA^{m}$ as Cartesian product. $\SPC$ is in fact the loose
integration of $m$ formalisms based on the point algebra, each interpreting
relations at different time scales. To fully define $\SPC$, we only
need in addition to define the interpretation at each scale. The interpretation
function $\I_{k}$ of the scale $g_{k}$ associates with each relation
$r$ the pairs of points whose representatives at the scale $g_{k}$
satisfy $r$, i.e. $\I_{k}(r)=\liste{(x,y)\in\mathbb{R}^{2}}{(g_{k}\left(x\right),g_{k}\left(y\right))\in\I_{\PA}(r)}$
with $g_{k}\left(x\right)$ the representative of the instant $x$
at the scale $g_{k}$ and $\I_{\PA}$ the classic interpretation of
$\PA$. 
\end{example}

We recover some properties of the classical framework.
\begin{lem}
\label{lem:ppt_relations}Let $(\A,U,\I)$ be a sequential formalism
and $R,R'\in\A$.

If $R\subseteq R'$ then $\I(R)\subseteq\I(R')$.
\end{lem}

\begin{proof}
We have $\I(R)=\bigcup_{B\in R}\I(B)$ and $\I(R')=\bigcup_{B\in R'}\I(B)$.
Let $B\in R$. We have $B\in R'$, since $R\subseteq R'$. Thus, $\I(B)\subseteq\I(R')$.
Therefore, $\I(R)\subseteq\I(R')$.
\end{proof}
To conclude this section, note that the reasoning introduced in the
previous section is justified by the semantics of sequential formalisms.
We have the following properties:
\begin{prop}
\label{prop:implication-MA}Let $(\A,U,\I)$ be a sequential formalism,
$R,R'\in\A$ and $x,y,z\in\U$. We have the following equivalences:
\begin{align*}
\contrainte xRy & \iff\contrainte y{\sinv R}x\\
\contrainte xRy\et\contrainte x{R'}y & \iff\contrainte x{\left(R\sinter R'\right)}y\\
\contrainte xRy & \iff\contrainte x{\left(\cconv R\right)}y\text{.}
\end{align*}
In addition, by assuming $\contrainte x{(\B_{1},\ldots,\B_{m})}z$,
we have the following implication:
\[
\contrainte xRy\et\contrainte y{R'}z\implies\contrainte x{\left(R\scomp R'\right)}z
\]
\end{prop}

\begin{proof}
To prove this proposition, we use the fact that $(\A,U,\I)$ is a
sequential formalism (Definition~\ref{defn:combined formalisms}). 

We prove the first equivalence:
\[
\begin{array}{ccc}
\contrainte xRy & \iff & (x,y)\in\I(R)\\
 & \iff & (y,x)\in\inv{\I(R)}\\
 & \iff & (y,x)\in\I(\sinv R)\\
 & \iff & \contrainte y{\sinv R}x\text{.}
\end{array}
\]

We prove the second equivalence:

\[
\begin{array}{ccc}
\contrainte xRy\et\contrainte x{R'}y & \iff & (x,y)\in\I(R)\et(x,y)\in\I(R')\\
 & \iff & (x,y)\in\I(R)\cap\I(R')\\
 & \iff & (x,y)\in\I(R\sinter R')\\
 & \iff & \contrainte x{\left(R\sinter R'\right)}y\text{.}
\end{array}
\]

We prove the last equivalence:

\[
\begin{array}{ccc}
\contrainte xRy & \iff & (x,y)\in\I(R)\\
 & \iff & (x,y)\in\I(\cconv R)\\
 & \iff & \contrainte x{\left(\cconv R\right)}y\text{.}
\end{array}
\]

We prove the implication (assuming $\contrainte x{(\B_{1},\ldots,\B_{m})}z$):
\[
\begin{array}{ccc}
\contrainte xRy\et\contrainte y{R'}z & \implies & (x,y)\in\I(R)\et(y,z)\in\I(R')\\
 & \implies & (x,z)\in\I(R)\circ\I(R')\\
 & \implies & (x,z)\in\left(\I(R)\circ\I(R')\right)\cap\I\left((\B_{1},\ldots,\B_{m})\right)\\
 & \implies & (x,z)\in\I(R\scomp R')\\
 & \implies & \contrainte x{\left(R\scomp R'\right)}z\text{.}
\end{array}
\]
\end{proof}

\subsection{Satisfiability of Multi-algebra Relations}

The semantics of relations induces a notion of satisfiability, which
we formulate in the following definition.
\begin{defn}
\label{defn:coherence_relations} Let $\F=(\A,\U,\I)$ be a sequential
formalism.

A relation $R$ of a multi-algebra $\A$ is said \emph{satisfiable\index{satisfiable@\emph{satisfiable}}
(for $\F$)} if $\I(R)\neq\vide$.
\end{defn}

Contrary to classical formalisms where any non-empty relation is satisfiable,
the relations -- even basic -- of multi-algebras can be unsatisfiable,
because of interdependencies. The definition of a sequential formalism
(definition~\ref{defn:combined formalisms}) allows us nevertheless
to keep the following elementary property:
\begin{prop}
\label{prop:equivalence_coherence_relation} Let $(\A,\U,\I)$ a sequential
formalism.

A relation R of a multi-algebra $\A$ is satisfiable if and only if
it contains a basic satisfiable relation, i.e.
\[
\I(R)\neq\vide\iff\exists B\in R\sep\I(B)\neq\vide
\]
\end{prop}

\begin{proof}
Let $R\in\A$. Suppose that $\I(R)\neq\vide$. Given that $\I(R)=\bigcup_{B\atome R}\I(B)$
(Definition of a sequential formalism ; Def.~\ref{defn:combined formalisms}),
$\bigcup_{B\atome R}\I(B)\neq\vide$ and thus $\exists B\in R\sep\I(B)\neq\vide$.

Conversely, suppose that $\exists B\in R\sep\I(B)\neq\vide$. Since,
$\I(R)=\bigcup_{B\atome R}\I(B)$, thus we have $\I(R)\neq\vide$.
\end{proof}
We have illustrated in Example~\ref{ex:proj closure} the fact that
some relations were unsatisfiable because of interdependencies, exhibiting
basic and non-basic relations whose closure under projection is empty.
Closing a relation under projection can thus help to detect its inconsistency.
If $\cconv R$ is empty, we can conclude that $R$ is unsatisfiable.
Otherwise, $\cconv R$ is said \emph{$\cconv$}-consistent.
\begin{defn}
\label{defn:arc-consistency} A relation $R$ of a multi-algebra is
\emph{$\cconv$-consistent}\index{cconv -consistent@$\cconv$-consistent}
if :
\begin{itemize}
\item $R$ is \index{closed under projection}closed under projection: $R=\cconv R$
and 
\item $R$ is not \emph{\index{trivially inconsistent}trivially inconsistent}:
$\forall i\sep R_{i}\neq\svide$.
\end{itemize}
\end{defn}

\begin{prop}
\label{prop:basique-coherente_conv-coherent} Let $(\A,\U,\I)$ be
a sequential formalism and $R\in\A$. We have the following three
properties:
\begin{itemize}
\item If $\cconv R$ is trivially inconsistent, then $R$ is unsatisfiable.
\item If $R$ is satisfiable, then $\cconv R$ is $\cconv$-consistent.
\item All satisfiable basic relation are $\cconv$-consistent.
\end{itemize}
\end{prop}

\begin{proof}
We first prove the first assertion. Let $R\in\A$ such that $\cconv R$
is trivially inconsistent. There exists $i\in\left\{ 1,\ldots,m\right\} $
such that $(\cconv R)_{i}=\svide$. But $\cconv R$ is closed under
projection, by definition. Thus, for all $j\in\left\{ 1,\ldots,m\right\} $
such that $i\neq j$, we have $(\cconv R)_{j}\subseteq\conv ij(\cconv R)_{i}=\conv ij\svide$.
Since $\conv ijr=\sbigunion_{b\atome r}\conv ijb$ (Lemma~\ref{lem:croissance_proj}),
$\conv ij\svide=\svide$. Thus, $(\cconv R)_{j}=\svide$ for all $j\in\left\{ 1,\ldots,m\right\} $.
Therefore, $\I(\cconv R)=\I\left(\left(\vide,\ldots,\vide\right)\right)=\vide$
(since $(\A,\U,\I)$ is a sequential formalism ; Definition~\ref{defn:combined formalisms}).
Since $\I(\cconv R)=\I(R)$ (Definition~\ref{defn:combined formalisms}),
we have $\I(R)=\vide$. The relation $R$ is thus unsatisfiable.

The second assertion follows directly from the contraposition of the
first assertion.

We now show the third assertion. Let $B\in\Base$ such that $B$ is
satisfiable. We have $\I(B)\neq\vide$. Thus $\I(\cconv B)=\I(B)\neq\vide$
(Definition~\ref{defn:combined formalisms}). Therefore, there exists
$i\in\{1,\ldots,m\}$ such that $(\cconv B)_{i}\neq\svide$ (Definition~\ref{defn:combined formalisms}).
Since $\cconv B$ is closed under projection, we have for all $i\in\{1,\ldots,m\}$
the property $(\cconv B)_{i}\neq\svide$. Given that $\cconv B\subseteq B$
(Definition~\ref{def:cloture-par-proj}), that $B$ is basic and
that $\cconv B$ is not trivially inconsistent, we have $B=\cconv B$.
Thus, $B$ is \emph{$\cconv$}-consistent.
\end{proof}
However, although projection operators can remove inconsistencies
with respect to pairs of components, satisfiability ultimately depends
on the interpretation function. For example, in a loose integration,
one can have $\I_{i}(R_{i})\cap\I_{j}(R_{j})\neq\vide$ for each pair
$\{i,j\}$, while having $\bigcap_{i}\I_{i}(R_{i})=\vide$ (since
in general $A\cap B\neq\vide\et B\cap C\neq\vide\et A\cap C\neq\vide\not\implies A\cap B\cap C\neq\vide$).
The projections being correct, the relation $\cconv(R_{1},\dots,R_{n})$
will not necessarily be trivially inconsistent, and yet $\I\left(\left(R_{1},\dots,R_{n}\right)\right)=\vide$.
Therefore, in general, $\cconv$-consistency does not imply satisfiability,
even for basic relations. Nevertheless, for some formalisms, all the
$\cconv$-consistent basic relations are satisfiable. With these formalisms,
reasoning can be done in a purely algebraic way. In Section~\ref{sub:arconsistency_consistency},
we identify sufficient conditions for projections to decide the satisfiability
of relations. 

\subsection{Multi-Algebra Networks and Algebraic Closure\label{subsec:multi-reseau_cloture}}

In the context of sequential formalisms, descriptions are \emph{qualitative
constraint networks over multi-algebras}, i.e. constraint networks
whose relations belong to a multi-algebra and are thus $m$-tuples
of classic relations. For example, the network $N$ on the left of
the figure~\ref{fig:RdC_TPA} corresponds to the temporal sequence
of 3 networks over the point algebra described in Example~\ref{ex:TPC}
; the relation $N^{xy}$ thus corresponds to the sequence of relations
between $x$ and $y$. 

Constraint networks over multi-algebras are an alternative but equivalent
representation to the biconstraint networks of the framework of loose
integrations and to temporal sequences of constraint networks. In
fact, a network $N$ over a multi-algebra $\A$ can be seen equivalently
as a $m$-tuple of classical networks $N_{i}$ over the algebra $\A_{i}$,
as shown to the right in Figure~\ref{fig:RdC_TPA}, in which $N_{i}$
corresponds to the state of the network at the instant $i$. This
is formalized by the following definition.

\begin{figure}
\begin{centering}
\includegraphics[width=0.9\columnwidth]{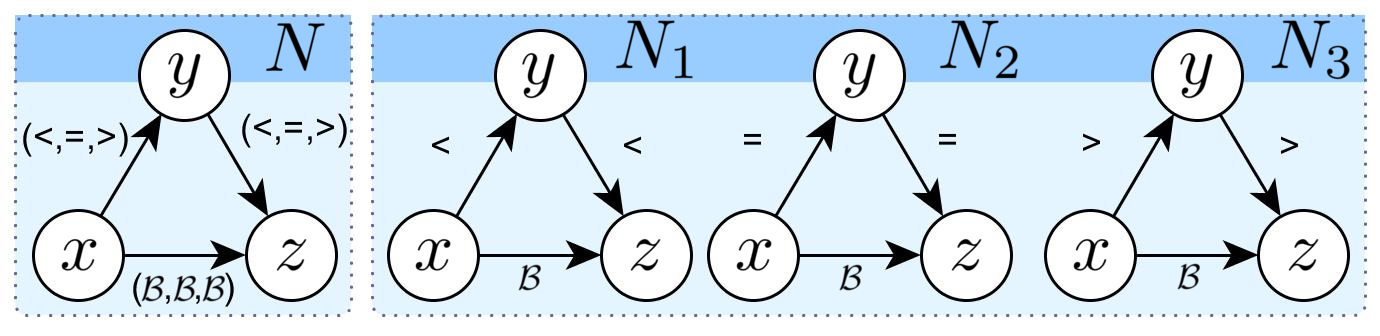}
\par\end{centering}
\caption{The network $N$ over $\protect\PA^{3}$ (left) and its three slices
(right)\label{fig:RdC_TPA}}
\end{figure}

\begin{defn}
\label{def:network slice} Let $N=(\E,\CdeN)$ be a network over a
multi-algebra $\A$. The \emph{$i$-th slice\index{slice} of $N$},
denoted by $N_{i}$ (or $N_{\A_{i}}$), is the network $(\E,\CdeN_{i})$
over $\A_{i}$, where $\CdeN_{i}=\liste{\contrainte x{R_{i}}y}{\contrainte xRy\in\CdeN}$.
\end{defn}

We directly adapt to networks over multi-algebras certain classical
notions, namely that of \emph{scenario} as well as those of \emph{solution}
and \emph{satisfiability}.
\begin{defn}
A \emph{scenario\index{scenario}} over a multi-algebra $\A$ is a
network over $\A$ satisfying $N^{xy}\in\Base$ for all distinct $x,y\in\E$.
\end{defn}

\begin{defn}
\label{def:coherence_multi-reseau}Let $\F=(\A,\U,\I)$ be a sequential
formalism and $N$ be a network over $\A$.
\begin{itemize}
\item A family $\{u_{x}\}_{x\in\mathtt{E}}$ is called \emph{solution\index{solution}}
of $N$ (for $\F$) if it satisfies $(u_{x},u_{y})\in\I(N^{xy})$,
for all distinct $x,y\in\E$. 
\item The network $N$ is said \emph{satisfiable\index{satisfiable}} (for
$\F$) if there exists a solution of $N$ (for $\F$) .
\item The network $N$ is said unsatisfiable\emph{\index{unsatisfiable}}
(for $\F$) if it is not satisfiable (for $\F$). 
\end{itemize}
\end{defn}

For example, a solution of the network of Figure~\ref{fig:RdC_TPA}
(for the formalism $\TPC$) is the evolution represented in Figure~\ref{fig:neighborhood}a. 

Note that Definition~\ref{defn:combined formalisms} of sequential
formalism induces the following elementary properties:
\begin{prop}
\label{prop:equivalence_coherence_reseau} Let $\F=(\A,\U,\I)$ be
a sequential formalism. Let $N$ and $\Cl{N'}$ be two networks over
$\A$. We have the following properties:
\begin{itemize}
\item The network $N$ is satisfiable if and only if there exists a satisfiable
scenario $S$ such that $S\subseteq N$.
\item If $N\subseteq N'$, then all solution of $N$ is a solution of $N'$.
\item If $N\subseteq N'$, then the satisfiability of $N$ implies the satisfiability
of $N'$.
\end{itemize}
\end{prop}

\begin{proof}
Let $S$ be a scenario and $N$ be a network over the same multi-algebra,
with $S\subseteq N$. Suppose that $S$ is satisfiable. Thus, for
all $z\in\E$, there exists $u_{z}\in\U$ such that for all distinct
$x,y\in\E$, $(u_{x},u_{y})\in\I(S^{xy})$. But, $\I(S^{xy})\subseteq\I(N^{xy})$,
since $S^{xy}\subseteq N^{xy}$ (Lemma~\ref{lem:ppt_relations}).
Therefore, $N$ is satisfiable. Conversely, suppose $N$ is satisfiable.
Therefore, there exists $\{u_{x}\}_{x\in\E}\subseteq\U$ such that
for all disctinct $x,y\in\E$ $(u_{x},u_{y})\in\I(N^{xy})$. But $\I(N^{xy})=\bigcup_{B\atome N^{xy}}\I(B)$,
thus there exists $B(x,y)\atome N^{xy}$ such that $(u_{x},u_{y})\in\I(B(x,y))$.
Therefore, the scenario $S$ defined by $S^{xy}=B(x,y)$ for all distinct
$x,y\in\E$, which is a scenario of $N$ ($S\subseteq N$), is indeed
satisfiable.

We now prove the second assertion of the proposition. We suppose that
$N\subseteq N'$. Let $\{u_{x}\}_{x\in\E}\subseteq\U$ be a solution
of $N$. We have, for all disctinct $x,y\in\E$, $(u_{x},u_{y})\in\I(N^{xy})$.
Since $N\subseteq N'$, for all disctinct $x,y\in\E$, $N^{xy}\subseteq(N')^{xy}$,
and thus $\I(N^{xy})\subseteq\I(N')^{xy}$ (Lemma~\ref{lem:ppt_relations}).
Therefore, for all distinct $x,y\in\E$, $(u_{x},u_{y})\in\I(N')^{xy}$.
Thus, $\{u_{x}\}_{x\in\E}\subseteq\U$ is a solution of $N'$. 

We prove the third assertion. We suppose that $N\subseteq N'$ and
that $N$ is satisfiable. Thus the network $N$ has a solution $\{u_{x}\}_{x\in\E}\subseteq\U$.
By the second assertion, this solution is a solution of $N'$. Therefore,
$N'$ is satisfiable.
\end{proof}
We now generalize the other notions of the classical framework of
qualitative formalisms.
\begin{defn}
\label{def:prop-networks} Let $N=(\E,\CdeN)$ be a network over a
multi-algebra. 
\begin{itemize}
\item $N$ is said \emph{trivially inconsistent\index{trivially inconsistent}}
if $\exists i\in\{1,\ldots,m\}\minisep\exists x,y\in\E\sep N_{i}^{xy}=\svide$.
\item $N$ is said\emph{ \index{algebraically closed}algebraically closed
if} :
\begin{itemize}
\item $N$ is \index{closed under composition}\emph{closed under composition},
which means $N^{xz}\subseteq N^{xy}\scomp N^{yz}$ for all distinct
$x,y,z\in\E$ ;
\item $N$ is \emph{closed under projection}\index{closed under projection},
which means $N^{xy}=\cconv N^{xy}$ for all distinct $x,y\in\E$.
\end{itemize}
\item $N$ is said\emph{ \index{algebraically consistent}}algebraically
consistent if it is algebraically closed and is not trivially inconsistent.
\item $N$ is said\emph{ $\scomp$-consistent}, if each $N_{i}$ is algebraically
consistent (closed under composition and not trivially inconsistent).
\end{itemize}
\end{defn}

\begin{defn}
\label{def:La-cl=0000F4ture-alg=0000E9brique}Let $N$ be a network
over a finite multi-algebra. 

The \emph{algebraic closure\index{algebraic closure} of $N$} is
the (algebraically closed) network obtained by applying the following
operations:
\[
\begin{array}{c}
N^{xz}\leftarrow\cconv N^{xz}\\
N^{xz}\leftarrow N^{xz}\sinter(N^{xy}\scomp N^{yz})
\end{array}
\]
for all distinct $x,y,z\in\E$, until a fixed point is reached.
\end{defn}

For a network over a multi-algebra, being \textquotedblleft algebraically
closed\textquotedblright{} consists in being closed both under composition
and under projection. It can be imposed by closing alternately each
relation $N^{xy}$ under projection and each slice $N_{i}$ under
composition until a fixed point is reached. This properly generalizes
the classical framework, since any relation of a mono-algebra is closed
under projection (by vacuity: there is no projection). Moreover, the
notion of algebraic closure for a multi-algebra generalizes to $m$
formalisms the bipath-consistency algorithm of loose integrations
of two formalisms~\cite{gerevini2002combining}. Since the operators
are correct (the reasoning is correct), the algebraic closure of a
network must be algebraically consistent for this network to be satisfiable. 
\begin{prop}
\label{prop:trivialement-incoherent_incoherent}Let $\F=(\A,\U,\I)$
be a sequential formalism, $N$ be a network over $\A$, and $\Cl N$
be the algebraic closure of $N$. We have the following properties:
\begin{itemize}
\item $\Cl N\subseteq N$.
\item All solution of $N$ is a solution of its algebraic closure $\Cl N$.
\item All solution of $\Cl N$ is a solution of $N$.
\item $N$ is satisfiable if and only if its algebraic closure $\Cl N$
is satisfiable.
\item If $N$ is trivially inconsistent, then $N$ is unsatisfiable.
\end{itemize}
\end{prop}

\begin{proof}
We prove the five points of the proposition.
\begin{enumerate}
\item We have $\Cl N\subseteq N$, since $\cconv R\subseteq R$ and $R\sinter(R'\scomp R'')\subseteq R$
for all $R,R',R''\in\A$ (Definitions~\ref{def:La-cl=0000F4ture-alg=0000E9brique}~and~\ref{def:cloture-par-proj}).
\item Let $\{u_{x}\}_{x\in\E}\subseteq\U$ be a solution of $N$. We have
$(u_{x},u_{y})\in\I(N^{xy})$ for all distinct $x,y\in\E$. But $(u_{x},u_{y})\in\I(N^{xy})$
for all distinct $x,y\in\E$ implies that $(u_{x},u_{y})\in\I\left(\Cl N\right)^{xy}$
for all distinct $x,y\in\E$, since $\contrainte xR{y\et\contrainte y{R'}z\et}\contrainte x{R''}z\implies\contrainte x{(R\scomp R')\sinter R''}z$
and $\contrainte xRy\implies\contrainte x{\cconv R}y$, for all $x,y,z\in\U$
and all $R,R',R''\in\A$ (Proposition~\ref{prop:implication-MA}).
Thus, $\{u_{x}\}_{x\in\E}$ is a solution of $\Cl N$. 
\item Let $\{u_{x}\}_{x\in\E}\subseteq\U$ be a solution of $\Cl N$. Since
$\Cl N\subseteq N$, $\{u_{x}\}_{x\in\E}$ is a solution of $N$ (Proposition~\ref{prop:equivalence_coherence_reseau}).
\item The fourth assertion follows directly from the two previous ones.
\item Let $N$ be a trivially inconsistent network. There therefore exists
$i\in\{1,\ldots,m\}$ and $x,y\in\E$ such that $N_{i}^{xy}=\svide$.
Consider its algebraic closure $\Cl N$. This verifies the property
$\Cl N^{xy}=\ctrois{\svide}{\ldots}{\svide}$, since 
\[
\Cl N_{j}^{xy}=N{}_{j}^{xy}\cap\conv ijN{}_{i}^{xy}=N{}_{j}^{xy}\cap\conv ij\vide=N{}_{j}^{xy}\cap\vide=\vide
\]
for all $j\in\{1,\ldots,m\}$ distinct from $i$ (Definitions~\ref{def:cloture-par-proj}~and~\ref{def:projection}).
Thus, given that $\F$ is a sequential formalism (Definition~\ref{defn:combined formalisms}),
we have $\I(N^{xy})=\vide$ for all distinct $x,y\in\E$. $\Cl N$
is thus unsatisfiable (Definition~\ref{def:coherence_multi-reseau}).
By the previous assertion of this proposition, $N$ is therefore unsatisfiable.
\end{enumerate}
\end{proof}
We can therefore use algebraic closure to filter out inconsistent
networks, in addition to inferring new knowledge. If to algebraically
close a network $N$ gives a trivially inconsistent network, $N$
is necessarily unsatisfiable.
\begin{example}
\label{ex:a-consistency} The network $N$ over $\STC$ of Figure~\ref{fig:RdC_IL}a
is algebraically consistent. However, if we remove the relation \guillemotleft{}
$=$ \guillemotright{} from $N_{\PA}^{xy}$, the network is no longer
algebraically closed because it is no longer closed under projection,
since $\{\ntppi,\eq\}\nsubseteq\conv{\PA}{\RCAH}\{>\}$ (Table~\ref{tab:interd=0000E9pendances-de_QS_RCC}).
Furthermore, even if $N$ is algebraically consistent and although
$N_{\PA}$ and $N_{\RCAH}$ are satisfiable, $N$ is actually unsatisfiable.
Indeed, on the one hand \guillemotleft{} $=$ \guillemotright de $N_{\PA}^{xy}$
does not belong to any satisfiable scenario of $N_{\PA}$ (see Example~\ref{fig:contre_ex_van_beek}
or \cite{van1989approximation}). On the other hand, $\ntppi$ of
$N^{xy}$ is also not feasible, because it does not belong to any
satisfiable scenario of $N_{\RCAH}$. The only remaining basic relation,
$\contrainte x{\cdeux{\eq}>}y$, is unsatisfiable, since its closure
under projection is the empty relation $(\vide,\vide)$.

However, by adding $\dc$ in $N_{\RCAH}^{yz}$, the network remains
algebraically consistent, but becomes satisfiable. Figure~\ref{fig:RdC_IL}b
shows one of its satisfiable scenarios.
\end{example}

\begin{figure}
\begin{centering}
\includegraphics[width=1\columnwidth]{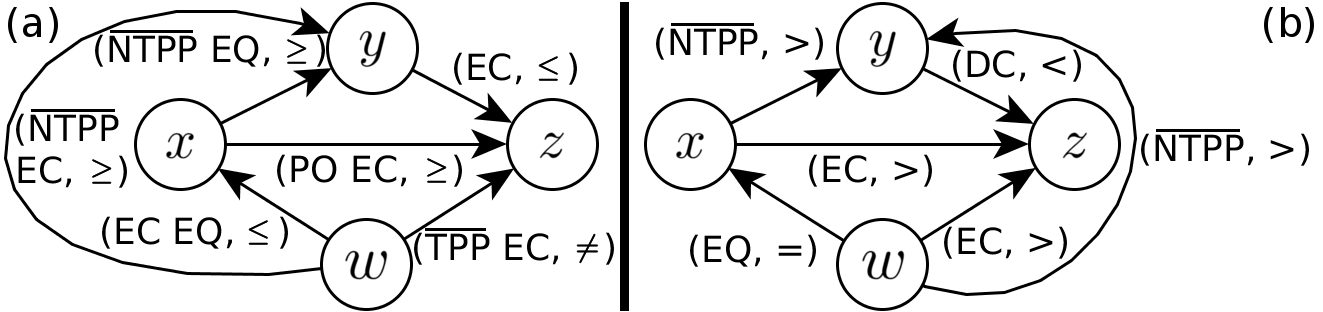}
\par\end{centering}
\caption{(a) a nework over $\protect\STC$ which is algebraically consistent,
but nonetheless unsatisfiable ; (b) a satisfiable scenario over $\protect\STC$\label{fig:RdC_IL}}
\end{figure}

Note that since the reasoning is correct, any satisfiable scenario
is algebraically closed:
\begin{prop}
\label{prop:sc=0000E9nario-coh=0000E9rent_clos}Let $\F=(\A,\U,\I)$
be a sequential formalism and $S$ be a scenario over $\A$. 

If $S$ is satisfiable then $S$ is algebraically closed.
\end{prop}

\begin{proof}
Let $S$ be a satisfiable scenario. Its algebraic closure $\Cl S$
is thus satisfiable (Proposition~\ref{prop:trivialement-incoherent_incoherent}).
Therefore, $\forall i\forall x,y\sep\Cl S_{i}\neq\svide$ (Proposition~\ref{prop:trivialement-incoherent_incoherent}).
Given that $\Cl S\subseteq S$ (Definition~\ref{def:prop-networks}),
that $S$ is basic, and that $\Cl S$ is not trivially inconsistent,
we have $S=\Cl S$. Thus, $S$ is algebraically closed.
\end{proof}
If we have the converse of this proposition, that is, if any algebraically
closed scenario is satisfiable, we can algebraically decide the satisfiability
of a network. It suffices for this to look for an algebraically closed
scenario among its scenarios (Proposition~\ref{prop:equivalence_coherence_reseau}).
Any sequential formalism satisfying this property is said to be\emph{
complete}\index{complete}. Under this assumption, we obtain the following
complexity result:
\begin{prop}
\label{prop:si_clos_coherent_NP}Let $\F=(\A,\U,\I)$ be a sequential
formalism whose algebraically closed scenarios are consistent.

The satisfiability decision of networks over $\A$ (for $\F$) is
in $\NP$.
\end{prop}

\begin{proof}
By Propositions~\ref{prop:sc=0000E9nario-coh=0000E9rent_clos}~and~\ref{prop:equivalence_coherence_reseau},
a network $N$ is satisfiable if and only if it is refined by an algebraically
closed scenario $S$. This scenario constitutes a \emph{certificate}
of the satisfiability of this network. Indeed, on the one hand, checking
that a scenario $S$ of the multi-algebra is a scenario of $N$ (i.e.
$S\subseteq N$) and that it is algebraically closed can be done in
polynomial time. On the other hand, any scenario of the multi-algebra
is of polynomial size with respect to the network $N$. The decision
of satisfiability is therefore in $\NP$.
\end{proof}
To search for an algebraically closed scenario within a network, in
order to decide if it is satisfiable, branch and bound techniques
can be used. For example, we can adapt the techniques of the classic
framework~\cite{Ladkin1992,nebel1997solving} 

\section{Tractability of the Satisfiability Decision in the Multi-algebras
Context\label{sec:Tractability}}

In this section, we study the \emph{tractability} of the satisfiability
decision problem of networks over multi-algebras. Since the problem
is $\NP$-complete for many formalisms, we proceed as in the classical
framework: we focus on subsets of multi-algebras, and in particular
on \emph{subclasses}. We first study satisfiability of\emph{ $\cconv$-consistent}
relations, which is essential for the algebraic closure to decide
satisfiability of networks. We then present two theorems providing
conditions on a subclass that are sufficient to guarantee satisfiability
of its algebraically consistent networks. Subclasses verifying these
conditions are, in particular, tractable. More precisely, we identify,
on the one hand, conditions for a combination of tractable classical
subclasses to be tractable. On the other hand, we identify conditions
for inheriting the tractability from a smaller combined subclass,
through an adaptation of the classical technique of reduction by refinement.
These two results are complementary and make it possible to prove
the tractability of large subclasses.

\subsection{Algebraically Tractable Subclasses\label{partial_projection}}

We start by introducing types of subsets of multi-algebras, adapted
from the classical case, namely \emph{subclasses} and in particular
\emph{basic subclasses}. We also define two other concepts of particular
subsets: \emph{$\icclo$ subsets} which generalize subclasses when
they do not contain the universal relation $\B$ and the abstract
equality $\mathrm{e}$ and the \emph{atomizable subsets} whose satisfiable
relations are refinable by a satisfiable basic relation from that
subset, which generalize basic subsets.
\begin{defn}
\label{def:subclass} A \emph{subclass\index{subclass}} of a multi-algebra
$\A$ is a set of relations $\S\subseteq\A$ which is closed under
composition, intersection, and inversion.

A subset $\S$ of a multi-algebra $\A$ is a basic subset\emph{\index{basic subset}}
if $\S$ contains all the basic relations (that is $\Base\subseteq\S$).

A subset $\S$of a multi-algebra $\A$ is $\icclo$ if for all $R,R',R''\in\S$,
we have $\left(R\scomp R'\right)\cap R''\in\S$.

A subset $\S$of a multi-algebra $\A$ is \emph{atomizable} if for
all satisfiable $R\in\S$, there exists a satisfiable basic relation
$B\in\S$ such that $B\subseteq R$.
\end{defn}

For example, $\HH\times\PA$ -- with $\HH$ one of the maximal tractable
basic subclass of $\RCCH$~\cite{renz1999maximal} -- is a basic
subclass of the multi-algebra $\STA$ of $\STC$ (whose Cartesian
product is $\RCAH\times\PA$). Let us recall that basic subclasses
are particularly interesting subclasses since all the scenarios over
their multi-algebras are scenarios over these subclasses.

The following notion of \emph{slice} of a subset of a multi-algebra
is in a way an inverse operator of the Cartesian product. It will
allow us to inherit certain tractability results from the classical
framework to that of multi-algebras.
\begin{defn}
\label{def:subclass slice} The \emph{$i$-th slice\index{slice}}
of a subset $\S$ of a multi-algebra $\A$, denoted by $\S_{i}$ (or
$\S_{\A_{i}}$), is the subset of $\A_{i}$ defined by $\S_{i}=\liste{R_{i}}{R\in\S}$.
\end{defn}

Note that $\S$ is always a subset of $\S_{1}\times\dots\times\S_{m}$,
the Cartesian product of its slices.

Recall also that, as in the classical framework, we can use algebraic
closure to detect inconsistent networks. It provides a satisfiability
decision procedure that is polynomial and correct (see Section~\ref{subsec:multi-reseau_cloture}),
but incomplete. Our objective in this section is to identify subclasses
for which the procedure is complete.
\begin{defn}
\label{def:algetract} A subset $\S$ of a multi-algebra of a sequential
formalism is said to be \emph{algebraically tractable\index{algebraically tractable}}
when, for any network $N$ over $\S$, $N$ is satisfiable if and
only if the algebraic closure of $N$ is not trivially inconsistent.
\end{defn}

For a subset to be algebraically tractable, a fundamental necessary
condition is that its algebraically closed scenarios are satisfiable;
it is in a way the basic case, which depends on the interpretation
function of the sequential formalism.

One might think that if all algebraically consistent networks over
a subclass $\S$ are satisfiable, then $\S$ is algebraically tractable.
This property is yet not sufficient because, contrary to the classical
framework, the algebraic closure of a network $N$ over a subclass
$\S$ is not necessarily over $\S$. Thus, the property does not imply
anything on the satisfiability of $N$. This property is however sufficient
in the cases where the projection closure of any relation of $\S$
remains in $\S$. More formally:
\begin{defn}
\label{def:sous-classe_clos-projection}A subset $\S$ of a finite
multi-algebra is said \emph{$\cconv$-closed} if $\forall R\in\S\sep\cconv R\in\S$.
\index{cconv -closed@$\cconv$-closed}
\end{defn}

\begin{prop}
\label{prop:alge_tractable} Let $\F=(\A,\U,\I)$ be a sequential
formalism.

A $\cconv$-closed $\icclo$ subset of $\A$ over which algebraically
consistent networks are satisfiable is algebraically tractable.
\end{prop}

\begin{proof}
Let $\S$ be a $\cconv$-closed $\icclo$ subset over which algebraically
consistent networks are satisfiable. Let $N$ be a network over $\S$
and $\Cl N$ be the algebraic closure of $N$. $\Cl N$ is over $\S$
since $\S$ is a $\cconv$-closed $\icclo$ subset. We have two possibilities:
\begin{itemize}
\item Either $\Cl N$ is trivially inconsistent. In this case, $\Cl N$
is unsatisfiable (Proposition~\ref{prop:trivialement-incoherent_incoherent}).
Therefore, $N$ is unsatisfiable (Proposition~\ref{prop:trivialement-incoherent_incoherent}). 
\item or $\Cl N$ is not trivially inconsistent. $\Cl N$ is therefore algebraically
consistent and thus satisfiable. Therefore, $N$ is satisfiable (Proposition~\ref{prop:trivialement-incoherent_incoherent}). 
\end{itemize}
Thus, for all $N$ over $\S$, $N$ is satisfiable if and only if
the algebraic closure of $N$ is not trivially inconsistent. Therefore,
$\S$ is algebraically tractable.
\end{proof}
The following lemma gives examples of $\cconv$-closed subclasses
based on classic subclasses (see Table~\emph{\ref{tab:def_classic_subclass}}).
\begin{lem}
\label{lem:STC-ss-classe-conv-close}The following subclasses of $\STC$
: 
\begin{itemize}
\item $\RCAHs\times\PAs$,
\item $\HH\times\PA$,
\item $\QH\times\PA$, and
\item $\CH\times\PA$
\end{itemize}
are $\cconv$-closed.

In particular, let $\S\times\S'$ be one of these subclasses, we have
the following properties:
\end{lem}

\begin{itemize}
\item for all $r\in\S$, $\conv{\RCAH}{\PA}r\in\S'$ ;
\item for all $r'\in\S'$, $\conv{\PA}{\RCAH}r'\in\S$.
\end{itemize}
\begin{proof}
We will show that these subclasses of the form $\text{\ensuremath{\S\times\S}'}$
are $\cconv$-closed, by showing the following two properties:$\forall r\in\S\sep\conv{\RCAH}{\PA}r\in\S'$
and $\forall r'\in\S'\sep\conv{\PA}{\RCAH}r'\in\S$. Indeed, by these
properties, to close a relation $R\in\S\times\S'$ under projection
gives a relation $(r,r')$ such that $r\in\S$ and $r'\in\S'$ and
thus $\cconv R=(r,r')\in\S\times\S'$. Therefore, $\S\times\S'$ is
indeed $\cconv$-closed.

We start by verifying that for all $r\in\PA$, we have $\conv{\PA}{\RCAH}r\in\RCAHs$
($\RCAHs$ is described in Table~\ref{tab:def_classic_subclass}):
\begin{itemize}
\item $\cconv\left(<\right)=\dc\cup\ec\cup\po\cup\tpp\cup\ntpp\in\RCAHs$
;
\item $\cconv\left(=\right)=\dc\cup\ec\cup\po\cup\eq\in\RCAHs$ ;
\item $\cconv\left(<\cup=\right)=\dc\cup\ec\cup\po\cup\tpp\cup\ntpp\cup\eq\in\RCAHs$
;
\item $\cconv\left(<\cup>\right)=\dc\cup\ec\cup\po\cup\tpp\cup\ntpp\cup\tppi\cup\ntppi\in\RCAHs$
;
\item $\cconv\B_{\PA}=\B_{\RCAH}\in\RCAHs$ ;
\item the subclasses are closed under inversion, by the definition of projections
($\cconv\sinv b=\sinv{\cconv b}$ ), we have the same result for the
converse relations.
\end{itemize}
Thus, we have the satisfaction of $\conv{\PA}{\RCAH}r\in\RCAHs$ for
all $r\in\PA$, and therefore in particular for all $r\in\PAs$. 

To show that for all $r\in\RCAHs$ we have $\conv{\RCAH}{\PA}r\in\PAs$,
it suffices to show that for any relation $r\in\RCAHs$, we have $\conv{\RCAH}{\PA}r\notin\deux{<\cup=}{=\cup>}$.
There are only three relations of $\RCCH$ satisfying $\conv{\RCAH}{\PA}r=\left(<\cup=\right)$.
Those are $\tpp\cup\eq$, $\ntpp\cup\eq$, and $\tpp\cup\ntpp\cup\eq$
(since $\mathord{>}\in\conv{\RCAH}{\PA}r$ if and only if $r\cap\{\dc,\ec,\po,\tppi,\ntppi\}\neq\vide$)
; and they are not in $\RCAHs$. By symmetry, there are only three
relations of $\RCCH$ satisfying $\conv{\RCAH}{\PA}r=\left(=\cup>\right)$,
namely $\tppi\cup\eq$, $\ntppi\cup\eq$, and $\tppi\cup\ntppi\cup\eq$
; and they are not in $\RCAHs$. 

Thus, for all $r\in\PAs$, we have $\conv{\PA}{\RCAH}r\in\RCAHs$
and for all $r\in\RCAHs$, we have $\conv{\RCAH}{\PA}r\in\PAs$. Therefore
$\RCAHs\times\PAs$ is $\cconv$-closed.

We are now interested in the other subclasses:
\begin{itemize}
\item On the one hand, from Table~\ref{tab:def_classic_subclass}, for
all $r\in\PA$, we have $\conv{\PA}{\RCAH}r\in\HH$, $\conv{\PA}{\RCAH}r\in\QH$,
and $\conv{\PA}{\RCAH}r\in\CH$ (in particular since for all $r\in\PA$,
we have $\po\in\conv{\PA}{\RCAH}r$, $\ntpp\in\conv{\PA}{\RCAH}r\implies\tpp\in\conv{\PA}{\RCAH}r$
, and $\ntppi\in\conv{\PA}{\RCAH}r\implies\tppi\in\conv{\PA}{\RCAH}r$).
\item On the other hand, for all $r\in\RCA$, and therefore, in particular,
for all $r$ in $\HH$, $\QH$, and $\CH$, we have $\conv{\RCAH}{\PA}r\in\PA$.
\end{itemize}
Thus, the subclasses $\HH\times\PA$, $\QH\times\PA$ et $\CH\times\PA$
are $\cconv$-closed.
\end{proof}
The question that now arises is the following: given a $\cconv$-closed
subclass, under what conditions is it guaranteed that the algebraically
consistent networks of this subclass are satisfiable? A first, obvious
condition is the satisfiability of $\cconv$-consistent relations.

\subsection{Satisfiability of $\protect\cconv$-Consistent Relations\label{sub:arconsistency_consistency}}

We are therefore interested, in this section, in conditions which
ensure that the $\cconv$-consistent relations of a multi-algebra
are consistent. For that, it is necessary that the basic relations
closed under projection are satisfiable. However, this condition is
not sufficient. Intuitively, what can compromise satisfiability is
the fact that there are too many interdependencies between the formalisms
of the combination. Conversely, if the interdependencies are weak
enough, satisfiability can be guaranteed. For example, the interdependencies
of $\TPC$ are weak, in the sense that they impose restrictions only
between neighboring instants -- the past does not directly constrain
the distant future. To formally define this idea, we start by introducing
the inverse operator of a projection, which describes the interdependencies
of this projection in the opposite direction.
\begin{defn}
\label{def:projection-reciproque} Let $\conv ij$ be a projection
operator from a finite non-associative algebra $\A_{i}$ to another
finite non-associative algebra $\A_{j}$. 

The inverse projection operator\index{inverse projection operator}
$\convinv ij$ of the operator $\conv ij$ is the projection operator
from $\A_{j}$ to $\A_{i}$ defined by 
\[
\forall b\in\Base_{i}\minisep\forall b'\in\Base_{j}\sep b\subseteq\convinv ijb'\iff b'\subseteq\conv ijb\text{.}
\]
For example, for $\SPC^{m}$ (Example~\ref{ex:GPC_multialg}), the
inverse projection of $\conv ij$ is simply the projection operator
$\conv ji$. Note, however, that the property $\conv ijb=\convinv jib$
is not satisfied by the projections of all the multi-algebras (we
will see in Section~\ref{sub:pblm_cloture} the interest of such
multi-algebras).
\end{defn}

We are now going to be able to formalize the fact that interdependencies
are not \textquotedblleft too strong\textquotedblright{} within a
multi-algebra, by imposing that they have a tree structure.
\begin{defn}
\label{def:anti-arborescence}An\emph{ anti-tree structure\index{anti-tree structure}}
$\arbo$ on a set $V$ is an oriented graph $(V,E)$ such that there
exists a particular node $r\in V$, called \emph{root}\index{root},
satisfying that for any node $v\in V$, distinct from $r$, there
is only one path from $v$ to $r$, i.e. there is a unique integer
$K\geq2$ and a unique family $\{v_{k}\}_{k=1}^{K}\subseteq V$ such
that $v_{1}=v$, $v_{K}=r$ and that for all $k\in\{1,\ldots,K-1\}$,
$(v_{k},v_{k+1})\in E$. 

We denote by $v\inArbo v'$ when there is an edge from the node $v$
to the node $v'$ in $\mathtt{A}$, i.e. when $(v,v')\in E$. If there
is no ambiguity, we denote it more simply $v\inA v'$.
\end{defn}

Thus, the idea of the following definition is that by removing redundant
interdependencies, one obtains a tree structure, which implies that
the interdependencies are local and do not conflict. All of the interdependencies
are then expressed by the projections of the tree structure.
\begin{defn}
\label{defn:arbo} Let $\A$ be a finite multi-algebra whose Cartesian
product is $\A_{1},\ldots,\A_{m}$.

A \emph{plenary\index{plenary}} anti-tree structure of $\A$ is an
anti-tree structure $\arbo$ on $\{1,\dots,m\}$ such that for all
distinct $i,j\in\{1,\ldots,m\}$, with $i=k_{0}\inA\dots\inA k_{s}\inAinv\dots\inAinv k_{s+t+1}=j$
the shortest oriented chain between $i$ and $j$ in $\mathtt{A}$,
the following property holds:
\[
\forall b\in\Base_{i}\sep\conv ijb\supseteq\convinv j{k_{s+t}}\cdots\convinv{k_{s+1}}{k_{s}}\conv{k_{s-1}}{k_{s}}\cdots\conv i{k_{1}}b
\]

The multi-algebra $\A$ is a \emph{tree\index{tree}} multi-algebra
if it has a plenary anti-tree structure. 
\end{defn}

The condition imposed by this definition describes that between two
algebras $\A_{i}$ and $\A_{j}$, the composition of the projections
of the plenary anti-tree structure imposes interdependencies at least
as strong as the projection $\conv ij$. The projection $\conv ij$
is therefore redundant and therefore dispensable: the projections
of the plenary anti-tree structure summarize the interdependencies.
\begin{example}
The multi-algebra of $\TPC$ is in fact a tree multi-algebra, with
for plenary anti-tree structure the temporal order of instants: $\mathtt{A}=(\{1,\ldots,m\},E)$
with $E=\liste{(i,i+1)}{i\in\{1,\ldots,m-1\}}$. Indeed, the projection
towards a neighboring instant is identical whether this instant is
future or past ($\conv i{i+1}r=\conv{i+1}ir=\convinv i{i+1}r$ for
all $r\in\PA$), and the projections to all other instants give $\B$
($\conv ijr=\B_{\PA}$ for all $r\in\PA$ when $\left|i-j\right|>1$),
which means that there is no direct interdependencies. More precisely,
the formula of the definition ``$\forall b\in\Base_{i}\sep\conv ijb\supseteq\convinv j{k_{s+t}}\cdots\convinv{k_{s+1}}{k_{s}}\conv{k_{s-1}}{k_{s}}\cdots\conv i{k_{1}}b$''
is satisfied for all distinct $i,j\in\{1,\ldots,m\}$: we have $\conv i{i+1}b\supseteq\conv i{i+1}b$
and $\conv{i+1}ib\supseteq\convinv i{i+1}b$ since $i\inA i+1$ and
also $\conv ijb=\B_{\PA}\supseteq\convinv j{j+1}\cdots\convinv{i-1}ib$
if $i-j>1$ and $\conv ijb=\B_{\PA}\supseteq\conv{j-1}j\cdots\conv i{i+1}b$
if $j-i>1$ (the anti-tree is linear). To sum up, in this example,
the interdependencies are not strong, because a projection between
distant instants constrains less than the composition of the projections
of the intermediate neighboring instants.
\end{example}

Note that the multi-algebras of loose integrations of two formalisms
are all tree multi-algebras, because $\conv ijb=\convinv jib$ for
all $b\in\Base_{i}$. It is thus in particular the case for the integration
$\STC$ (Example~\ref{ex:QST_formalisme}).

In fact, since the projections of the plenary anti-tree structure
express all the interdependencies, being closed uniquely under these
projections implies being closed under all the projections.
\begin{lem}
\label{lem:clos_partiel_arbo}Let $\A$ be a finite tree multi-algebra,
$\mathtt{A}$ be one of its plenary anti-tree structure, and $B$
be one of its basic relations. 

If for all distinct $k,\ell\in\left\{ 1,\ldots,m\right\} $, we have
$k\inA\ell\implies B_{\ell}\subseteq\conv k{\ell}B_{k}$, then $B$
is closed under projection (we have $B_{j}\subseteq\conv ijB_{i}$
for all distinct $i,j\in\left\{ 1,\ldots,m\right\} $). 
\end{lem}

\begin{proof}
We assume that if $k\inA\ell$ then we have $B_{\ell}\subseteq\conv k{\ell}B_{k}$
and we show that for all distinct $i,j\in\left\{ 1,\ldots,m\right\} $,
we have $B_{j}\subseteq\conv ijB_{i}$.

Let distinct $i,j\in\{1,\ldots,m\}$ and $i=k_{0}\inA\dots\inA k_{s}\inAinv\dots\inAinv k_{s+t+1}=j$
the shortest chain between $i$ and $j$ in the anti-tree structure
$\arbo$. Given that $i=k_{0}\inA k_{1}$, we have $B_{k_{1}}\subseteq\conv i{k_{1}}B_{i}$.
More generally, since $k_{p}\inA k_{p+1}$, we have $B_{k_{p+1}}\subseteq\conv{k_{p}}{k_{p+1}}B_{k_{p}}$
for all $p\in\{0,\ldots,s-1\}$. Given that the projection operators
are increasing (Lemma~\ref{lem:croissance_proj}), we have $\conv{k_{1}}{k_{2}}B_{k_{1}}\subseteq\conv{k_{1}}{k_{2}}(\conv i{k_{1}}B_{i})$.
By transitivity, $B_{k_{2}}\subseteq\conv{k_{1}}{k_{2}}(\conv i{k_{1}}B_{i})$,
since $B_{k_{2}}\subseteq\conv{k_{1}}{k_{2}}B_{k_{1}}$. Analogously,
we have $\conv{k_{2}}{k_{3}}B_{k_{2}}\subseteq\conv{k_{2}}{k_{3}}(\conv{k_{1}}{k_{2}}(\conv i{k_{1}}B_{i}))$,
and thus $B_{k_{3}}\subseteq\conv{k_{2}}{k_{3}}(\conv{k_{1}}{k_{2}}(\conv i{k_{1}}B_{i}))$
since $B_{k_{3}}\subseteq\conv{k_{2}}{k_{3}}B_{k_{2}}$. By induction,
we thus deduce $B_{k_{s}}\subseteq\conv{k_{s-1}}{k_{s}}\dots\conv i{k_{1}}B_{i}$. 

Moreover, as $k_{s}\inAinv k_{s+1}$, we have $B_{k_{s}}\subseteq\conv{k_{s+1}}{k_{s}}B_{k_{s+1}}$,
and thus $B_{k_{s+1}}\subseteq\convinv{k_{s+1}}{k_{s}}B_{k_{s}}$
(since $\convinv{k_{s+1}}{k_{s}}$ is the inverse projection of $\conv{k_{s+1}}{k_{s}}$
; Definition~\ref{def:projection-reciproque}). By Lemma~\ref{lem:croissance_proj}),
we have $\convinv{k_{s+1}}{k_{s}}(B_{k_{s}})\subseteq\convinv{k_{s+1}}{k_{s}}(\conv{k_{s-1}}{k_{s}}\dots\conv i{k_{1}}B_{i})$.
By transitivity, we deduce $B_{k_{s+1}}\subseteq\convinv{k_{s+1}}{k_{s}}(\conv{k_{s-1}}{k_{s}}\dots\conv i{k_{1}}B_{i})$.
By a similar induction, we prove that $B_{j}\subseteq\convinv j{k_{s+t}}\dots\convinv{k_{s+1}}{k_{s}}\conv{k_{s-1}}{k_{s}}\dots\conv i{k_{1}}B_{i}$,
since $B_{k_{s+p+1}}\subseteq\convinv{k_{s+p+1}}{k_{s+p}}B_{k_{s+p}}$
for all $0\leq p\leq t$.

Since $\arbo$ is a plenary anti-root structure of $\A$, we know
by definition that
\[
\convinv j{k_{s+t}}\dots\convinv{k_{s+1}}{k_{s}}\conv{k_{s-1}}{k_{s}}\dots\conv i{k_{1}}B_{i}\subseteq\conv ijB_{i}\text{.}
\]
We deduce by transitivity that $B_{j}\subseteq\conv ijB_{i}$. This
being true for all distinct $i,j\in\{1,\ldots,m\}$, we conclude that
$B$ is closed under projection.
\end{proof}
Added to the necessary condition on basic relations, the tree structure
condition makes it possible to guarantee the satisfiability of $\cconv$-consistent
relations:
\begin{prop}[\cite{cohen2017decision}]
\label{prop:arborescence} Let $\A$ be a tree multi-algebra of a
sequential formalism whose basic relations closed under projection
are satisfiable. 

Its $\cconv$-consistent relations are satisfiable.
\end{prop}

\begin{proof}
Let $R\in\A$ be a $\cconv$-consistent relation and $\arbo=(\{1,\ldots,m\},E)$
be one of the plenary anti-root structure of $\A$, whose root is
denoted by $r$. We will build a particular basic relation $B\subseteq R$.
Let $j$ such that $j\inA r$ (i.e. that $(j,r)\in E$) ; since $R_{r}\subseteq\conv jrR_{j}$,
for all $b_{r}\atome R_{r}$ there exists $b_{j}\atome R_{j}$ such
that $b_{r}\subseteq\conv jrb_{j}$. If we choose one of the basic
relation $b_{r}$, then we can choose a corresponding basic relation
$b_{j}$, and this for all $j$ such that $j\inA r$. Let $j\in\left\{ 1,\ldots,m\right\} $,
now consider $i\in\left\{ 1,\ldots,m\right\} $ such that $i\inA j$.
Similarly, for all $b_{j}\atome R_{j}$ there exists $b_{i}\atome R_{i}$
such that $b_{j}\subseteq\conv ijb_{i}$. Thus, we can choose, for
all $i$ such that $i\inA j$, a $b_{i}$ corresponding to $b_{j}$
chosen previously -- and this for all the $j$ considered previously.
We can repeat this process by traversing all the anti-tree in order
to obtain a basic relation $B=(b_{1},\ldots,b_{m})$ satisfying the
following property: for all distinct $i,j\in\{1,\dots,m\}$, if $i\inA j$
then $B_{j}\subseteq\conv ijB_{i}$. Thus, by Lemma~\ref{lem:clos_partiel_arbo},
$B$ is closed under projection. Since $B$ is basic, it is therefore
satisfiable, by hypothesis. Therefore, since $B\subseteq R$, we conclude
that $R$ is satisfiable (Proposition~\ref{prop:equivalence_coherence_relation}).
\end{proof}
\begin{example}
Thus, the $\cconv$-consistent relations of $\TPC$ (temporalized
point calculus) and $\STC$ (size and topology combination) are satisfiable,
since their multi-algebras are tree multi-algebras. For these formalisms,
closing a relation under projection is enough to decide its satisfiability. 
\end{example}

\begin{example}
We have the same result for $\SPC$ (the scaled point algebra ; see
Example~\ref{ex:GPC_formalisme}): its plenary anti-tree structure
corresponds to the fineness order of scales. However, if we generalize
$\SPC$ by considering that the scales are not totally ordered by
the fineness relation $\preceq$, we lose in general this result.
The multi-algebra of generalized $\SPC$ is a tree multi-algebra when
scales form a partial tree order with respect to the fineness relation.
In other words, the multi-algebra is a tree multi-algebra when for
any pair of scales $(g,g')$ such that $g$ is not finer than $g'$
and $g'$ is not finer than $g$, there is no considered scale $g''$
which is coarser, that is to say less fine, than these two scales.
An example is the case of the Gregorian calendar.
\end{example}

We have identified conditions for $\cconv$-consistent relations to
be satisfiable. However, these conditions are obviously not sufficient
to ensure the satisfiability of algebraically consistent networks.
In the following, we therefore propose two complementary theorems
identifying conditions ensuring algebraic tractability. The first
theorem gives us conditions for a basic subclass to inherit the tractability
of its slices, while the second theorem gives us conditions for the
tractability to be inherited from another subset by refinement.

\subsection{Inheriting Tractability by Slicing\label{sect:filters_for_tractability}}

In this section, we are interested in conditions ensuring that a basic
subclass $\S$ is tractable by relying on the tractability of each
of its slices $\S_{i}$. One of these conditions is that each slice
must be minimal (Definition~\ref{def:classe_minimale}) which is
a particular case of tractability. Indeed, consider a basic subclass
$\S$ whose each slice $\S_{i}$ is minimal. Each basic relation in
any relation of an algebraically consistent network over $\S_{i}$
can be extended in a satisfiable scenario. The idea is to build a
scenario for the network over $S$ from the building of scenarios
over the $\S_{i}$. But, there is in fact no guarantee that it is
possible. Additional conditions are required to ensure that the individual
processes for obtaining the scenarios work well together with respect
to projections. On the one hand, $\cconv$-consistent relations (Definition~\ref{defn:arc-consistency})
must be satisfiable, which is algebraic tractability for relations
(assuming that $\S$ is $\cconv$-closed). On the other hand, the
compatibility between this hypothesis and the minimality of slices
$\S_{i}$ requires the following property.
\begin{defn}
\label{defn:simple} A network is said \emph{dissociable\index{dissociable}}\footnote{\emph{In the original paper, we use instead the term ``simple''.}}\emph{}
when closing it under projection then under composition makes it either
algebraically consistent or trivially inconsistent.

A subset $\S$ of a finite multi-algebra is said to be \emph{dissociable}
when every network over $\S$ is dissociable.
\end{defn}

We will see in Section~\ref{sect:conditions_plus_fortes} that the
simplicity of a subclass is easy to establish, in particular thanks
to the properties of composition and projection.

Using these properties, we formulate our first theorem:
\begin{namedthm}[Slicing theorem]
\label{thm:combination} Let $\S$ be a $\icclo$ atomizable subset
of a sequential formalism $\F=(\A,\U,\I)$ whose algebraically closed
scenarios over $\A$ are satisfiable. If the following conditions
are satisfied:
\begin{itemize}
\item (D1) each slice $\S_{i}$ is minimal ;
\item (D2) $\S$ is dissociable ; 
\item (D3) all $\cconv$-consistent relation $R\in\S$ is satisfiable
\end{itemize}
then algebraically consistent networks over $\S$ are satisfiable.

Moreover, if $\S$ is $\cconv$-closed then $\S$ is algebraically
tractable.
\end{namedthm}
\begin{proof}
Let $N$ be an algebraically consistent network over $\S$, and let
$x$ and $y$ be distinct variables such that $N^{xy}$ is not basic.
We know that $N^{xy}$ is satisfiable (D3). It therefore contains
at least one satisfiable basic relation $B=(b_{1},\dots,b_{m})$ such
that $B\in\S$ (because $\S$ is atomizable). The basic relation $B$
is therefore in particular $\cconv$-consistent (Proposition~\ref{prop:basique-coherente_conv-coherent}).
We then refine $N^{xy}$ by $B$. Obviously, the modified network
$N$, which we denote by $N'$, is still closed under projection,
and it is always over $\S$.

On the other hand, since each slice $N_{i}$ of the initial network
is algebraically consistent and by (D1), there exists an algebraically
closed scenario $S_{i}\subseteq N'_{i}$ such that $S_{i}^{xy}=b_{i}$.
Thus, the closure of $N'$ under composition is not trivially inconsistent.
Therefore, the closure of $N'$ under composition is algebraically
consistent, since $\S$ is dissociable (D2) and $N'$ is closed under
projection.

The closure of $N'$ under composition is an algebraically consistent
network over $\S$ (since $N'$ is over $\S$ and $\S$ is $\icclo$).
We can therefore apply the procedure again from the beginning, thus
making iteratively basic and satisfiable all the relations of the
initial network. The procedure necessarily ends with an algebraically
closed scenario $S$, which is therefore satisfiable by assumption
of the theorem.Therefore, since $S\subseteq N$, $N$ is satisfiable
(Proposition~\ref{prop:equivalence_coherence_reseau}). We have thus
proved the first conclusion of the theorem.

To complete the proof, we further assume that $\S$ is $\cconv$-closed.
By Proposition~\ref{prop:alge_tractable}, we prove the second conclusion
of the theorem.
\end{proof}
This theorem can be used to prove the tractability of a $\cconv$-closed
basic subclass $\S$ constructed from known tractable basic subclasses.
It also offers an efficient way to decide the satisfiability of a
network over $\S$, requiring only a projection closure followed by
a composition closure. In fact, this theorem implies \emph{separability}
in the sense of \cite{li2012reasoning}. Finally, the proof describes
an efficient algorithm implementing a procedure for building a satisfiable
scenario from any algebraically consistent network over $\S$.

\subsection{Inheriting Tractability by Refinement}

In this section, we present another set of conditions ensuring the
tractability of a $\cconv$-closed subclass $\S$. The idea, inspired
by the classic technique of reduction by a refinement~\cite{renz1999maximal},
is to find a subset $\mathcal{S}'$ whose algebraically consistent
networks are satisfiable, as well as a refinement from $\S$ to $\mathcal{S}'$
which preserves the algebraic consistency of networks. A natural idea
to find such a refinement is to combine classic refinements $h_{1},\dots,h_{m}$
in a refinement on the multi-algebra. We call \textquotedblleft multi-refinement\textquotedblright{}
the refinements having this specific form:
\begin{defn}
\label{def:affinement}A \emph{refinement}\index{refinement} of a
subset of a multi-algebra is a function $H\colon\S\to\A$ such that,
for all $R\in\S$:
\begin{itemize}
\item $H(R)\subseteq R$ ;
\item $H(R)$ is not trivially inconsistent if $R$ is not trivially inconsistent.
\end{itemize}
A refinement $H$ of a subset $\S$ of a multi-algebra is a \emph{multi-refinement\index{multi-refinement}}
if it is of the form $H\colon R\mapsto(h_{1}(R_{1}),\dots,h_{m}(R_{m}))$,
where each $h_{i}$ is a refinement of $\S_{i}$. We then denote $H=(h_{1},\dots,h_{m})$. 
\end{defn}

We call \emph{algebraic stability through a refinemen}t this condition
of preservation of algebraic consistency when applying a refinement:
\begin{defn}
\label{defn:stable} Let $N$ be a network over a multi-algebra $\A$
and $H$ be a refinement of $\A$.

We denote by $H(N)$ the network obtained from $N$ by simultaneously
replacing each relation $N^{xy}$ by $H(N^{xy})$.

A subset $\S$ of a multi-algebra is \emph{algebraically stable\index{algebraically stable}
through a refinement $H$} if for any algebraically consistent network
$N$ over $\S$, the refined network $H(N)$ is always algebraically
consistent.
\end{defn}

For example, $\HH\times\PA$ is algebraically stable through $H=(\hHH,\hmax)$,
with $\hHH$ the refinement of $\HH$ (Definition~\ref{def:Affinement-RCC8}).
We will show in Section~\ref{sect:conditions_plus_fortes} that stability,
like simplicity (Definition~\ref{def:Affinement-RCC8}), can be decided
easily by enumeration.

Our second theorem formalizes the reduction mechanism between subsets
of multi-algebras.
\begin{namedthm}[Refinement theorem]
\label{thm:affinement} Let $\F=(\A,\U,\I)$ be a sequential formalism,
$\S$ and $\S'$ be two subsets of $\A$, and $H$ be a refinement
from $\S$ to $\S'$. If the following conditions are satisfied:
\begin{itemize}
\item (A1) $\S$ is algebraically stable through $H$ ; 
\item (A2) algebraically consistent networks over $\S'$ are satisfiable
; 
\end{itemize}
then algebraically consistent networks over $\S$ are also satisfiable.

If, in addition, $\S$ is a $\cconv$-closed $\icclo$ subset, then
$\S$ is algebraically tractable.
\end{namedthm}
\begin{proof}
Let $N$ be an algebraically consistent network over $\S$. By the
condition (A1), the network $H(N)$ is algebraically consistent. It
is thus satisfiable, by condition (A2), since it is over $\S'$. We
deduce that $N$ is also satisfiable, since $H(N)\subseteq N$ (Proposition~\ref{prop:equivalence_coherence_reseau}).

The second result is a direct corollary of the first (by Proposition~\ref{prop:alge_tractable}).
\end{proof}

\section{Properties to Apply Tractability Results\label{sec:Properties-to-Apply}\label{sect:conditions_plus_fortes}}

We are now interested in stronger conditions which make it possible
to easily verify the assumptions of the tractability theorems. The
majority of assumptions relate to networks. Therefore, they cannot
be verified by enumeration (there is an infinity of networks). Thus,
we identify stronger conditions, relating to relations of the multi-algebra,
making possible a proof by enumeration (there is a finite number of
relations). We also study a technique that allows us to get around,
in some cases, the non-satisfaction of the assumptions of theorems.
To conclude, we summarize these results in two corollaries whose conditions
are more easily verifiable and we illustrate them by recovering the
known results of the size-topology combination $\STC$.

\subsection{Properties of slices and bi-slices}

We start by formalizing the restrictions $\A_{i}\times\A_{j}$ of
a multi-algebra $\A$, as well as the corresponding restrictions of
its subclasses and its networks, which we call \emph{bi-slices}. The
properties of bi-slices and slices will allow us to obtain results
on the subclasses of $\A$.
\begin{defn}
\label{def:bi-tranche}Let $\A$ be a multi-algebra of Cartesian product
$\A_{1}\times\cdots\times\A_{m}$, $\S$ be a subset of $\A$, $N=(\E,\CdeN)$
be a network over $\S$, $R\in\A$, and distinct $i,j\in\{1,\ldots,m\}$.

The \emph{bi-slice\index{bi-slice}} $\A_{i,j}$ of $\A$ is the multi-algebra
of Cartesian product $\A_{i}\times\A_{j}$ equipped with the projections
$\conv ij$ and $\conv ji$ from $\A$.

The \emph{bi-slice} $R_{i,j}$ of $R$ is the relation of $\A_{i,j}$
defined by $R_{i,j}=(R_{i},R_{j})$.

The \emph{bi-slice} $\S_{i,j}$ of $\S$ is the subset of $\A_{i,j}$
defined by $\liste{R_{i,j}}{R\in\S}$.

The \emph{bi-slice} $N_{i,j}$ of $N$ is the network over $\S_{i,j}$
defined by $(\E,\CdeN_{i,j})$ with $\CdeN_{i,j}=\liste{\contrainte x{R_{i,j}}y}{\contrainte xRy\in\CdeN}$.
\end{defn}

Note the following properties on the slices of a subclass.
\begin{prop}
\label{prop:equiv_S_Si}Let $\S$ be a subset of a finite multi-algebra.
We have the following properties:
\begin{itemize}
\item If $\S$ is a subclass then each slice $\S_{i}$ is a subclass.
\item If $\S$ is a basic subset and if $\Base_{1}\times\cdots\times\Base_{m}\neq\vide$
then each slice $\S_{i}$ is a basic subset.
\item If each slice $\S_{i}$ is a subclass then $\S_{1}\times\cdots\times\S_{m}$
is a subclass.
\item If each slice $\S_{i}$ is a basic subset then $\S_{1}\times\cdots\times\S_{m}$
is a basic subset.
\end{itemize}
\end{prop}

\begin{proof}
We suppose that $\S$ is a subclass. We show that each $\S_{i}$ is
a subclass. Let $i\in\{1,\ldots,m\}$ and $r,r'\in\S_{i}$. We show
that $r\sinter r'\in\S_{i}$, $r\scomp r'\in\S_{i}$, and $\sinv r\in\S_{i}$.
Since $r,r'\in\S_{i}$, there exists $R,R'\in\S$ such that $R_{i}=r$
and $R'_{i}=r'$, by the definition of $\S_{i}$ (Definition~\ref{def:subclass slice}).
Since $\S$ is a subclass, $R\sinter R'\in\S$, $R\scomp R'\in\S$,
and $\sinv R\in\S$. By the definition of operateurs (Definition~\ref{def:multi-operateur}),
$(R\sinter R')_{i}=R_{i}\sinter R_{i}'=r\sinter r'$, $(R\scomp R')_{i}=R_{i}\scomp R_{i}'=r\scomp r'$,
and $(\sinv R)_{i}=\sinv{R_{i}}=\sinv r$. Thus, $r\sinter r'\in\S_{i}$,
$r\scomp r'\in\S_{i}$, and $\inv r\in\S_{i}$, by the Definition
of $\S_{i}$.

We suppose that $\S$ is a basic subset. We show that each $\S_{i}$
is a basic subset. Let $i\in\{1,\ldots,m\}$, i.e., we show that for
all $b\in\B_{i}$, $b\in\S_{i}$. Since $\S$ is a basic subset, for
all $B\in\Base=\Base_{1}\times\cdots\times\Base_{m}$, $B\in\S$.
Let $b\in\Base_{i}$, there exists $B\in\Base$ such that $B_{i}=b$.
Indeed, since $\Base_{1}\times\cdots\times\Base_{m}\neq\vide$, there
exists $(b_{1},\ldots,b_{m})\in\Base$ and therefore $(b_{1},\ldots,b_{i-1},b,b_{i+1},\ldots,b_{m})\in\Base$.
Since $B\in\S$, $B_{i}=b\in\S_{i}$, by the definition of $\S_{i}$.
Thus, $\S_{i}$ is a basic subset.

We assume that each slice $\S_{i}$ is a subclass, and we show that
$\S_{1}\times\cdots\times\S_{m}$ is a subclass. Let $R,R'\in\S_{1}\times\cdots\times\S_{m}$.
For all $i\in\{1,\ldots,m\}$, $(R\sinter R')_{i}=R_{i}\sinter R'_{i}\in\S_{i}$,
$(R\scomp R')_{i}=R_{i}\scomp R'_{i}\in\S_{i}$, and $(\sinv R)_{i}=\sinv{R_{i}}\in\S_{i}$
(since $R_{i}\in\S_{i}$, $R'_{i}\in\S_{i}$ and $\S_{i}$ is a subclass).
Therefore, $(R\sinter R')\in\S_{1}\times\cdots\times\S_{m}$, $(R\scomp R')\in\S_{1}\times\cdots\times\S_{m}$,
and $\sinv R\in\S_{1}\times\cdots\times\S_{m}$.

We assume that each slice $\S_{i}$ is a basic subset, and we show
that $\S_{1}\times\cdots\times\S_{m}$ is a basic subset, i.e. we
show that $\Base\subseteq\S_{1}\times\cdots\times\S_{m}$. For all
$i\in\{1,\ldots,m\}$, we have $\Base_{i}\subseteq\S_{i}$. Thus,
we have $\Base=\Base_{1}\times\cdots\times\Base_{m}\subseteq\S_{1}\times\cdots\times\S_{m}$.
\end{proof}
Let us now notice the following properties on the relations and networks
and their bi-slices.
\begin{lem}
\label{lem:bi-relations_equiv}Let $\A$ be a finite multi-algebra
and $R\in\A$. We have the following properties:
\begin{itemize}
\item $R$ is closed under projection if and only if for all distinct $i,j\in\left\{ 1,\ldots,m\right\} $,
$R_{i,j}$ is closed under projection.
\item $R$ is $\cconv$-consistent if and only if for all distinct $i,j\in\left\{ 1,\ldots,m\right\} $,
$R_{i,j}$ is $\cconv$-consistent.
\end{itemize}
\end{lem}

\begin{proof}
We prove the first assertion. Let $R\in\A$. We suppose that $R$
is a relation closed under projection. We have for all distinct $i,j\in\{1,\ldots,m\}$
, $R_{j}\subseteq\conv ijR_{i}$ (Definition~\ref{def:cloture-par-proj}).
Thus, for all distinct $i,j\in\{1,\ldots,m\}$, $R_{j}\subseteq\conv ijR_{i}$
and $R_{i}\subseteq\conv jiR_{j}$. We deduce by definition that for
all distinct $i,j\in\{1,\ldots,m\}$, $R_{i,j}=(R_{i},R_{j})$ is
closed under projection (Definition~\ref{def:bi-tranche}).

Now suppose that each $R_{i,j}$ is closed by projection. We thus
have, for all distinct $i,j\in\{1,\ldots,m\}$, $R_{j}\subseteq\conv ijR_{i}$
and $R_{i}\subseteq\conv jiR_{j}$. We deduce directly that $R$ is
closed under projection.

The second assertion follows directly from the first.
\end{proof}
\begin{prop}
\label{prop:bi-tranche_equiv}Let $N$ be a network over a finite
multi-algebra. The following properties are satisfied:
\begin{itemize}
\item $N$ is closed under projection if and only if for all distinct $i,j\in\left\{ 1,\ldots,m\right\} $,
$N_{i,j}$ is closed under projection.
\item $N$ is algebraically closed if and only if for all distinct $i,j\in\left\{ 1,\ldots,m\right\} $,
$N_{i,j}$ is algebraically closed.
\item $N$ is algebraically consistent if and only if for all distinct $i,j\in\left\{ 1,\ldots,m\right\} $,
$N_{i,j}$ is algebraically consistent.
\end{itemize}
\end{prop}

\begin{proof}
Let $N$ be a network over a une multi-algebra. We prove the first
assertion.

Let distinct $i,j\in\{1,\ldots,m\}$. Suppose that $N$ is closed
under projection. For all distinct $x,y\in\E$, $N^{xy}$ is closed
under projection, therefore $N_{i,j}^{xy}$ is closed under projection
(Lemma~\ref{lem:bi-relations_equiv}). Thus, $N_{i,j}$ is closed
under projection. Conversely, suppose that $N_{i,j}$ is closed under
projection. For all distinct $x,y\in\E$, $N_{i,j}^{xy}$ is closed
under projection, and thus $N^{xy}$ is closed under projection (Lemma~\ref{lem:bi-relations_equiv}).
Therefore, $N$ is closed under projection.

The proofs of the other two assertions are analogous, and are based
on the fact that $R\subseteq R'\scomp R''$ if and only if for all
$i\in\{1,\ldots,m\}$, $R_{i}\subseteq R_{i}'\scomp R_{i}''$ (by
Definition~\ref{def:multi-operateur}) and the fact that $R$ is
not trivially inconsistent if and only if for all $i\in\{1,\ldots,m\}$,
$R_{i}\neq\vide$ (Definition~\ref{defn:arc-consistency}).
\end{proof}

\subsection{$\protect\cconv$-Closed Subclasses}

We are now interested in verifying the $\cconv$-closure of particular
subsets.
\begin{prop}
\label{prop:ss-classe_clos-local-clos-global}Let $\S$ be a subclass
of a finite multi-algebra such that $\S=\S_{1}\times\dots\times\S_{m}$.

If for all distinct $i,j\in\{1,\ldots,m\}$ and for all $r\in\S_{i}$,
we have $\conv ijr\in\S_{j}$, then $\S$ is $\cconv$-closed.
\end{prop}

\begin{proof}
Let $R\in\S$, we show that $\cconv R\in\S$.

Closing a relation under projection amounts to applying the following
refinement operation: $R_{j}\leftarrow R_{j}\cap\conv ijR_{i}$, for
all distinct $i,j\in\{1,\ldots,m\}$ until a fixed point is reached
(which is the relation $\cconv R$). Let $K$ be the number of steps
needed to reach this fixed point and let $R^{k}$ be the state of
$R$ at the step $k$, we have:
\[
R_{j}^{k+1}=R_{j}^{k}\cap\bigcap_{i=1,i\neq j}^{m}\conv ijR_{i}^{k}
\]
for $j\in\{1,\ldots,m\}$ and $k\in\{1,\ldots,K-1\}$ with $R^{1}=R$
and $R^{K}=\cconv R$. 

We prove that $R_{j}^{k}\in\S_{j}$ for all $j\in\{1,\ldots,m\}$
by induction on $k\in\{1,\ldots,K\}$. The property holds at the rank
$k=1$, indeed $R^{1}=R\in\S=\S_{1}\times\dots\times\S_{m}$ and thus
$R_{j}^{1}\in\S_{j}$ for all $j\in\{1,\ldots,m\}$. We assume the
property holds at the rank $k$, $1\leq k<K$, and we show that it
holds at the rank $k+1$. Let $j\in\{1,\ldots,m\}$. By induction
hypothesis, $R_{j}^{k}\in\S_{j}$ is satisfied and for all $i\in\{1,\ldots,m\}$
distinct from $j$, $R_{i}^{k}\in\S_{i}$. Thus, we have the property
$\conv ijR_{i}^{k}\in\S_{j}$ for all $i\in\{1,\ldots,m\}$ distinct
from $j$, by hypothesis of the statement. We deduce the property
$R_{j}^{k}\cap\bigcap_{i=1,i\neq j}^{m}\conv ijR_{i}^{k}\in\S_{j}$
since $\S_{j}$ is a subclass (if $r\in\S_{j}$ and $r'\in\S_{j}$
then $r\sinter r'\in\S_{j}$ ; $\S_{j}$ is a subclass by Proposition~\ref{prop:equiv_S_Si}).
Thus, we conclude that $R_{j}^{k+1}\in\S_{j}$.

Thus, for all $j\in\{1,\ldots,m\}$, $(\cconv R)_{j}=R_{j}^{K}\in\S_{j}$.
Therefore, $\cconv R\in\S$, since $\S=\S_{1}\times\dots\times\S_{m}$.
\end{proof}

\subsection{Simplicity and Distributivity}

We are now interested in verifying the simplicity (Definition~\ref{defn:simple}),
to be able to apply the slicing theorem.

First of all, it is enough to enumerate all three-variable networks
over each bi-slice $\S_{i,j}$ to verify that a basic subclass $\S$
is dissociable, since being closed is a local property:
\begin{prop}
\label{prop:simple} Let $\S$ be a $\cconv$-closed $\icclo$ subset
of a finite multi-algebra. 

If for all distinct $i,j\in\{1,\ldots,m\}$, all networks over $\S_{i,j}$
with only three variables are dissociable, then $\S$ is dissociable. 
\end{prop}

\begin{proof}
The proof of this result comes directly from the definition of closure
under projection and closure under composition. Recall that to be
closed under projection is to be closed under projection for each
pair $\{i,j\}$, and that to be closed under composition is to be
closed under composition for each triplet of variables.

Suppose the assumptions hold, and consider a network $N$ over $\S$.
We want to show that closing $N$ under projection then under composition
makes it either trivially inconsistent or algebraically consistent.
Suppose that the closure of $N$ under projection then composition
is not trivially inconsistent. We show that it is algebraically consistent.

We close $N$ under projection ($N$ is thus again over $\S$ since
$\S$ is $\cconv$-closed). If then we close $N$ locally under composition
for three variables, each bi-slice $N_{i,j}$ remains closed under
projection (and $N$ remains over $\S$). Indeed, $N_{i,j}$ restricted
to three variables is dissociable (since $N_{i,j}$ is over $\S_{i,j}$)
and $N_{i,j}$ is closed under projection (Proposition~\ref{prop:bi-tranche_equiv}).
As it cannot be trivially inconsistent after composition by hypothesis,
closure under composition of a bi-slice $N_{i,j}$ restricted to three
variables is therefore algebraically consistent and in particular
closed under projection. Thus, $N_{i,j}$ is closed under projection
after a local closure of $N$ under composition for three variables.
Since each bi-slice $N_{i,j}$ remains closed under projection, this
is the case for the entire $N$ (Proposition~\ref{prop:bi-tranche_equiv}).
We can therefore perform other local composition closures until we
have a network fully closed under composition. It will remain closed
under projection after each step and will therefore still be closed
at the end, because $N$ remains over $\S$ since $\S$ is a $\icclo$
subset.

In conclusion, if the closure of $N$ under projection then under
composition is not trivially inconsistent, it is algebraically consistent.
By definition, $N$ is dissociable. Since this result applies to any
network over $\S$, it proves that $\S$ is dissociable.
\end{proof}
However, in the context of formalisms having a large number of relations,
this verification may not be feasible in a reasonable time. We then
propose a stronger property, i.e. implying simplicity, whose verification
is only quadratic depending on the number of relations, namely $(m-1)\cdot\sum_{i\in\{1,\ldots,m\}}|S_{i}|^{2}$
instead of $\sum_{\begin{array}{c}
i,j\in\{1,\ldots,m\}\\
i<j
\end{array}}|S_{i}|^{3}\cdot|S_{j}|^{3}$. This property is particularly remarkable because it is the superdistributivity
of the projection operators on composition and on intersection. We
call \textquotedblleft $\cconv$-distributives\textquotedblright{}
the subclasses verifying it:
\begin{defn}
\label{def:distrib}A subset $\S$ of a multi-algebra is said \emph{$\cconv$-distributive\index{cconv -distributivity@\emph{$\cconv$-}distributivity}}
if it satisfies the following conditions for all distinct $i,j\in\{1,\ldots,m\}$:
\begin{itemize}
\item $\conv ij$ is \emph{superdistributive} over $\scomp$ on $\S_{i}$,
that is to say that for all relations $r,r'\in\S_{i}$ such that $r\scomp r'\neq\emptyset$,
we have the property $(\conv ijr)\scomp(\conv ijr')\subseteq\conv ij(r\scomp r')$;
\item $\conv ij$ is \emph{superdistributive} over $\sinter$ on $\S_{i}$,
that is to say that for all relations $r,r'\in\S_{i}$ such that $r\sinter r'\neq\emptyset$,
we have the property $(\conv ijr)\sinter(\conv ijr')\subseteq\conv ij(r\sinter r')$. 
\end{itemize}
\end{defn}

Note that a similar property, the distributivity of composition over
intersection, has already been identified as implying the tractability
of basic subclasses of classical formalisms~\cite{long2015distributive}. 
\begin{example}
The upward conversion operator of the point algebra is superdistributive
over $\scomp$ and $\sinter$ on $\PAs$, but the downward conversion
is not superdistributive (over composition), which prevents this subclass
from being $\cconv$-distributive.
\end{example}

Since the superdistributivity of $\conv ij$ over $\sinter$ and $\scomp$
ensures that composition closure preserves the fact of being closed
under projection for satisfiable networks, simplicity is satisfied:
\begin{prop}
\label{prop:distrib-to-simple} Any $\cconv$-closed $\cconv$-distributive
subclass of a finite multi-algebra is dissociable.
\end{prop}

\begin{proof}
Let $\S$ be a $\cconv$-closed $\cconv$-distributive subclass. By
Proposition~\ref{prop:simple}, it suffices to prove that each subclass
$\S_{i}\times\S_{j}$ is dissociable. Let distinct $i,j\in\{1,\ldots,m\}$
and $N'$ be a network over $\S_{i}\times\S_{j}$ ; we will show that
$N'$ is dissociable thanks to $\cconv$-distributivity. Note first
that since $\S_{i}\times\S_{j}$ is a $\cconv$-closed subclass, by
applying the operators, the networks over $\S_{i}\times\S_{j}$ remain
over this subclass, which ensures that the property of $\cconv$-distributivity
is always satisfied at any time during an algebraic closure. Let $N$
be the closure of $N'$ under projection. Suppose that the closure
under composition of $N$ is not trivially inconsistent (otherwise,
$N'$ is dissociable by definition), and show that performing a composition
propagation leaves $N$ closed under projection. By symmetry (compared
to $\conv ji$ and $\conv ij$), we just need to prove that after
a composition propagation the network is still closed under the projection
operator $\conv ij$, which means:
\[
\forall x,y,z\in\mathtt{E}\sep N_{j}^{xz}\sinter(N_{j}^{xy}\scomp N_{j}^{yz})\subseteq\conv ij(N_{i}^{xz}\sinter(N_{i}^{xy}\scomp N_{i}^{yz}))\text{.}
\]

Let $x$, $y$, and $z$ be distinct variables of $N$. Since we
assume that the closure under composition of $N$ cannot be trivially
inconsistent, we have the property $N_{i}^{xz}\sinter(N_{i}^{xy}\scomp N_{i}^{yz})\neq\svide$
and in particular the property $N_{i}^{xy}\scomp N_{i}^{yz}\neq\svide$.
Since $N$ is closed under projection, $N_{j}^{XY}\subseteq\conv ijN_{i}^{XY}$
for all distinct $X,Y\in\{x,y,z\}$. By monotony of $\scomp$ (Lemma~\ref{lem:ppt-algebre-non-asso-finie}),
we obtain $N_{j}^{xy}\scomp N_{j}^{yz}\subseteq\conv ijN_{i}^{xy}\scomp\conv ijN_{i}^{yz}$.
By superdistributivity of $\conv ij$ over $\scomp$, we deduce $N_{j}^{xy}\scomp N_{j}^{yz}\subseteq\conv ij(N_{i}^{xy}\scomp N_{i}^{yz})$.
Since $N_{j}^{xz}\subseteq\conv ijN_{i}^{xz}$, we have $N_{j}^{xz}\sinter(N_{j}^{xy}\scomp N_{j}^{yz})\subseteq\conv ij(N_{i}^{xz})\sinter\conv ij(N_{i}^{xy}\scomp N_{i}^{yz})$.
By superdistributivity of $\conv ij$ over $\sinter$ and because
$N_{i}^{xy}\scomp N_{i}^{yz}\in\S_{i}$ (since $\S$ is a subclass),
the property $N_{j}^{xz}\sinter(N_{j}^{xy}\scomp N_{j}^{yz})\subseteq\conv ij(N_{i}^{xz}\sinter(N_{i}^{xy}\scomp N_{i}^{yz}))$
is thus satisfied. Therefore, closing $N'$ under projection then
under composition gives an algebraically consistent network if it
is not trivially inconsistent. By definition, $N'$ is thus dissociable.

Since for each distinct $i,j\in\{1,\ldots,m\}$, all networks over
$\S_{i}\times\S_{j}$ are dissociable, all networks over $\S_{i,j}$
with three variables are dissociable. Thus, by Proposition~\ref{prop:simple},
$\S$ is dissociable. 
\end{proof}
Note that for most qualitative formalisms (those which are \emph{uniform}
in the sense of the following definition), it is sufficient that projection
is superdistributive over composition on basic relations for being
superdistributive on the entire formalism.
\begin{defn}
A non-associative algebra is said \emph{uniform\index{uniform}} if
for all $b,b'\in\Base$, the property $b\scomp b'\neq\svide$ is satisfied.
\end{defn}

The majority of qualitative formalisms in the literature are uniform.
However, the formalisms relating to entities of different natures~\cite{inants2016qualitative},
such as points and intervals, are not uniform.
\begin{example}
The point algebra and $\RCCH$ are uniform. The point and interval
algebra of Meiri \cite{Meiri1996} and its generalization~\cite{cohen2015algebra}
are not uniform.
\end{example}

\begin{prop}
\label{prop:distrib_basique_non-basique}Let $\A$ be a finite uniform
non-associative algebra and $\cconv$ be a projection from $\A$ to
another non-associative algebra.

If $\cconv$ is superdistributive over $\scomp$ on $\Base$ then
$\cconv$ is superdistributive over $\scomp$ on $\A$. 
\end{prop}

\begin{proof}
We prove this proposition using the properties of projections and
non-associative algebras. Let $r,r'\in\A$. We have: 
\[
\begin{array}{ccccc}
\cconv r\scomp\cconv r' & = & (\sbigunion_{b\in r}\cconv b)\scomp(\sbigunion_{b'\in r'}\cconv b') &  & \text{(Lemma\,\ref{lem:croissance_proj})}\\
 & = & \sbigunion_{b\in r}\sbigunion_{b'\in r'}\cconv b\scomp\cconv b' & \text{since} & r\scomp(r'\sunion r'')=(r\scomp r')\sunion(r\scomp r'')\\
 & \subseteq & \sbigunion_{b\in r}\sbigunion_{b'\in r'}\cconv(b\scomp b') & \text{by} & \text{uniformity and over-distributivity}\\
 & = & \cconv\sbigunion_{b\in r}\sbigunion_{b'\in r'}(b\scomp b') & \text{since} & \cconv r\sunion\cconv r'=\cconv(r\sunion r')\\
 & = & \cconv(r\scomp r') &  & \text{(Lemma\,\ref{lem:ppt-algebre-non-asso-finie})}
\end{array}
\]

Thus, for all $r,r'\in\A$, the property $\cconv r\scomp\cconv r'\subseteq\cconv(r\scomp r')$
is satisfied. We thus have superdistributivity over composition.
\end{proof}
From this result, we deduce the superdistributivity of the projection
from $\RCAH$ to $\PA$ over composition, for the size-topology combination
$\STC$. 
\begin{lem}
\label{lem:sur-distrib-compo-STC} The projection $\cconv_{\RCAH}^{\PA}$
from $\RCAH$ to $\PA$ in $\STC$ is superdistributive over $\scomp$
on $\RCAH$.
\end{lem}

\begin{proof}
$\RCAH$ (the algebra of $\RCCH$) is uniform. By Proposition~\ref{prop:distrib_basique_non-basique},
to prove this lemma, it is sufficient to prove that $\cconv_{\RCAH}^{\PA}$
is superdistributive over $\scomp$ on $\Base_{\RCAH}$.

Consider the following pairs $(b,b')\in\left(\Base_{\RCAH}\right)^{2}$:
\begin{itemize}
\item $b\in\{\tpp,\ntpp\}$ and $b'\in\{\tpp,\ntpp\}$, we have $b\scomp b'\subseteq\tpp\cup\ntpp$,
$\cconv\tpp=\cconv\ntpp=\cconv\left(\tpp\cup\ntpp\right)=\mathord{<}$,
and $\left(\mathord{<}\scomp\mathord{<}\right)=\mathord{<}$, hence
$\cconv b\scomp\cconv b'\subseteq\cconv(b\scomp b')$ ;
\item $b=\eq$ and $b'\in\{\tpp,\ntpp\}$, we have $\cconv\tpp=\cconv\ntpp=\mathord{<}$,
$\cconv\eq$ is $\mathord{=}$, $\eq\scomp b'=b'$, and $\left(\mathord{=}\scomp\mathord{<}\right)=\mathord{<}$,
hence $\cconv b\scomp\cconv b'\subseteq\cconv(b\scomp b')$ ;
\item $b\in\{\tpp,\ntpp\}$ and $b'=\eq$, we have analogously $\cconv b\scomp\cconv b'\subseteq\cconv(b\scomp b')$
;
\item $b=\eq$ and $b'=\eq$, we have $\cconv\eq$ is $\mathord{=}$, $\eq\scomp\eq=\text{\ensuremath{\eq}}$,
and $\mathord{=}\scomp\mathord{=}$ is $\mathord{=}$, hence $\cconv b\scomp\cconv b'\subseteq\cconv(b\scomp b')$.
\end{itemize}
By symmetry, $\cconv b\scomp\cconv b'\subseteq\cconv(b\scomp b')$
is satisfied for respectively $b\in\{\tppi,\ntppi\}$ and $b'\in\{\tppi,\ntppi\}$,
$b=\eq$ and $b'\in\{\tppi,\ntppi\}$, $b\in\{\tppi,\ntppi\}$ and
$b'=\eq$. Indeed, we have the property $\cconv\sinv{b'}\scomp\cconv\sinv b\subseteq\cconv(\sinv{b'}\scomp\sinv b)$
if the property $\cconv b\scomp\cconv b'\subseteq\cconv(b\scomp b')$
is satisfied, since $\sinv{b\scomp b'}=\sinv{b'}\scomp\sinv b$, $\sinv{r\sunion r'}=\sinv r\sunion\sinv{r'}$
(by définition of a non-associative algebra) and $\cconv\sinv b=\sinv{\cconv b}$
(Definition~\ref{def:projection}).

For all other pairs $(b,b')\in\left(\Base_{\RCAH}\right)^{2}$ of
basic relation from $\RCAH$, we have $\trois{\dc}{\ec}{\po}\cap(b\wcomp b')\neq\vide$.
Thus, $\cconv(b\scomp b')=\B_{\PA}$ (since $\cconv\dc=\cconv\ec=\cconv\po=\B_{\PA}$),
hence $\cconv b\scomp\cconv b'\subseteq\cconv(b\scomp b')$. 
\end{proof}
On the other hand, the superdistributivity of this projection over
intersection is satisfied only on small subclasses, such as $\RCAHs$
(see Table~\ref{tab:def_classic_subclass}).
\begin{lem}
\label{lem:sur-distrib-inter-RCCHs}The projection $\cconv_{\RCAH}^{\PA}$
from $\RCAH$ to $\PA$ in $\STC$ is superdistributive over $\sinter$
on $\RCAHs$.
\end{lem}

\begin{proof}
We show that if $r\cap r'\neq\vide$, then $\cconv r\cap\cconv r'\subseteq\cconv(r\cap r')$.
Note first that there are only three non-basic relations of $\RCAHs$
which do not contain $\po$ : $\dc\cup\ec$, $\tpp\cup\ntpp$, and
$\tppi\cup\ntppi$.

Let $r,r'\in\RCAHs$, with $r\cap r'\neq\vide$. We have five cases
(not disjoint but exhaustive by commutativity of the intersection)
to analyze:
\begin{itemize}
\item The case where $r$ is basic. Therefore we have $r\cap r'=r$ . Thus,
$\cconv(r\cap r')=\cconv r$, hence $\cconv r\cap\cconv r'\subseteq\cconv(r\cap r')$.
\item The case $\po\subseteq r$ and $\po\subseteq r'$. Therefore we have
$\cconv(r\cap r')\supseteq\cconv\po=\B_{\PA}$, hence $\cconv r\cap\cconv r'\subseteq\cconv(r\cap r')$.
\item The case $r=\dc\cup\ec$. Therefore, we have either $r\cap r'=\dc$,
$r\cap r'=\ec$, or $r\cap r'=\dc\cup\ec$. Thus, $\cconv(r\cap r')=\cconv\dc=\cconv\ec=\cconv\left(\dc\cup\ec\right)=\B_{\PA}$,
hence $\cconv r\cap\cconv r'\subseteq\cconv(r\cap r')$.
\item The case $r=\tpp\cup\ntpp$. Therefore we have either $r\cap r'=\tpp$,
$r\cap r'=\ntpp$, or $r\cap r'=\tpp\cup\ntpp$. But $\cconv\tpp=\cconv\ntpp=\cconv\left(\tpp\cup\ntpp\right)=\mathord{<}$.
Hence, $\cconv r\cap\cconv r'\subseteq\cconv(r\cap r')$.
\item The case $r=\tppi\cup\ntppi$. It is symmetric to the previous case.
\end{itemize}
Thus, if $r\cap r'\neq\vide$, then $\cconv r\cap\cconv r'\subseteq\cconv(r\cap r')$.
\end{proof}
We thus obtain the following result:
\begin{cor}
\label{cor:sur-distrib-RCCHs}The projection $\cconv_{\RCAH}^{\PA}$
of $\STC$ is superdistributive over composition and over intersection
on $\RCAHs$.
\end{cor}

\begin{proof}
By direct application of Lemmas~\ref{lem:sur-distrib-compo-STC}~and~\ref{lem:sur-distrib-inter-RCCHs}.
\end{proof}
Unfortunately, we do not have the superdistributivity of the projection
$\cconv_{\PA}^{\RCAH}$ over composition on $\Base_{\PA}$. Therefore,
no basic subclass of $\STC$ can be $\cconv$-distributive.

\subsection{Stabilities and Invariance}

We are now interested in checking algebraic stability (Definition~\ref{defn:stable})
in order to be able to apply the refinement theorem. We first notice
that a subset is algebraically stable if it is both stable for composition
and stable for projection.
\begin{defn}
A subset $\S$ of a multi-algebra is said \emph{composition stable
through a refinement $H$ \index{composition stable}} when for any
network $N$ over $\S$, if $N$ is $\scomp$-consistent then $H(N)$
is also $\scomp$-consistent.

Moreover, $\S$ is said \emph{projection stable through a refinement
$H$ \index{projection stable}} when for any relation $R\in\S$,
if $R$ is $\cconv$-consistent then $H(R)$ is also $\cconv$-consistent. 
\end{defn}

The composition stability is in fact a particular case of the reduction
by refinement of the classical framework~\cite{renz1999maximal}.

The lemma below follows directly from the definitions.
\begin{lem}
\label{lem:trois_stabilites} A subset of a multi-algebra which is
projection stable and composition stable through a refinement $H$
is algebraically stable through $H$. 
\end{lem}

We therefore only need to independently check composition stability
and projection stability. Again, these properties can be verified
by enumeration. For composition, it suffices to enumerate the three-variable
networks over all slices $\S_{i}$. For projection, it suffices to
enumerate the relations of all the bi-slices $\S_{i,j}$.
\begin{prop}
\label{prop:stable} Let $\S$ be a subset of a finite multi-algebra
and $H=(h_{1},\dots,h_{m})$ be a multi-refinement of $\S$. 
\begin{enumerate}
\item If, for any slice $\S_{i}$ and all $\scomp$-consistent three-variable
network $N$ over $\S_{i}$, $h_{i}(N)$ is still $\scomp$-consistent,
then $\S$ is composition stable through $H$.
\item If, any bi-slice $\S_{i,j}$ is projection stable \emph{through} $(h_{i},h_{j})$,
then $\S$ is projection stable \emph{through} $H$.
\end{enumerate}
\end{prop}

\begin{proof}
(1) Let $N$ be a $\scomp$-consistent network over $\S$ ; we show
that $H(N)$ is $\scomp$-consistent. We know by hypothesis that if
we refine by $h_{i}$ a $\scomp$-consistent three-variable network
over $\S_{i}$, then it remains $\scomp$-consistent. On the one hand,
we deduce that $H(N)$ is not trivially inconsistent. On the other
hand, the closure under composition being local to each triplet of
variables and to each slice, $H(N)$ is necessarily closed under composition.
More precisely, for all distinct variables $x,y,z$ of $N$ and all
$i\in\{1,\text{\ensuremath{\ldots}},m\}$, we have $N_{i}^{xz}\subseteq N_{i}^{xy}\scomp N_{i}^{yz}$.
We thus know that after the refinement, we have $h_{i}(N_{i}^{xz})\subseteq h_{i}(N_{i}^{xy})\scomp h_{i}(N_{i}^{yz})$
for all $i\in\left\{ 1,\ldots,m\right\} $. By definition, $H(N)$
is thus closed under composition. Therefore, $H(N)$ is $\scomp$-consistent. 

(2) Recall that $\S_{i,j}$ is projection stable through $(h_{i},h_{j})$
means that for all $\cconv$-consistent relation $(r,r')\in\S_{i,j}$,
the relation $(h_{i}(r),h_{j}(r'))$ is also $\cconv$-consistent.
Let $R$ be a $\cconv$-consistent relation of $\S$. Thus, for all
distinct $i,j\in\left\{ 1,\ldots,m\right\} $, the relation $R_{i,j}=(R_{i},R_{j})\in\S_{i,j}$
is $\cconv$-consistent (Lemma~\ref{lem:bi-relations_equiv}). By
assumption, for all distinct $i,j\in\left\{ 1,\ldots,m\right\} $,
the relation $(h_{i},h_{j})\left(R_{i,j}\right)=(h_{i}(R_{i}),h_{j}(R_{j}))$
is thus $\cconv$-consistent. Therefore, $H(R)=(h_{1}(R_{1}),\ldots,h_{m}(R_{m}))$
is $\cconv$-consistent (Lemma~\ref{lem:bi-relations_equiv}).
\end{proof}
Note that composition stability is often already established, because
the tractable subclasses of the literature being reducible by refinement
are in fact more precisely composition stable \emph{through} the same
refinement. Recall that the classical refinements \emph{$\hmax$,}
$\hHH$, $\hQH$, and $\hCH$ are defined in Table~\ref{tab:affinement_PA}
and in Definition~\ref{def:Affinement-RCC8}.
\begin{prop}[\cite{renz1999maximal}]
\label{prop:stab-compo-PA} \label{prop:stab-compo-RCC} 
\begin{itemize}
\item The point algebra $\PA$ is composition stable through\emph{ $\hmax$}.
\item $\HH$, $\QH$, and $\CH$ are respectively composition stable through
$\hHH$, $\hQH$, and $\hCH$. 
\end{itemize}
\end{prop}

Moreover, although algebraic stability is an assumption of the refinement
theorem, for tree multi-algebras (Definition~\ref{defn:arbo}), projection
stability implies the condition (D3) of the slicing theorem (by Proposition~\ref{prop:arborescence},
when basic relations closed under projection are satisfiable). 

We have the following stability results for the projection.
\begin{lem}
\label{lem:stab-projection-STC}$\HH\times\PA$, $\QH\times\PA$,
and $\CH\times\PA$ are respectively projection stable through $H=(\hHH,\hmax)$,
 $H=(\hQH,\hmax)$, and $H=(\hCH,\hmax)$.
\end{lem}

\begin{proof}
We will prove that these subclasses $\S$ are projection stable through
$H$, that is, for any $\cconv$-consistent relation $r\in\S$, $H(R)$
is also $\cconv$-consistent (see Table~\ref{tab:def_classic_subclass}
for the definitions of $\HH$, $\QH$, and $\CH$ ; see Definition~\ref{def:Affinement-RCC8}
for those of their classic refinement: respectively $\hHH$, $\hQH$,
and $\hCH$ , and see Table~\ref{tab:affinement_PA} for definition
of $\hmax$). Note first the following property, denoted by (1) in
this proof: for all $b\in\Base_{\RCAH}$ and $r\in\PA$, if $\vide\neq r$
and $r\subseteq\conv{\RCAH}{\PA}b$ then $\cdeux br$ is $\cconv$-consistent
(since $b'\in\conv{\RCAH}{\PA}b\iff b\in\conv{\PA}{\RCAH}b'$, see
Tables~\ref{tab:interd=0000E9pendances-de_QS_RCC}~and~\ref{tab:interd=0000E9pendances-de_RCC_QS}).
Note further that for all $r\in\PA$, if $r\neq\vide$ then $\hmax(r)\neq\vide$.

Consider $\HH\times\PA$ and $\QH\times\PA$, which have the same
refinement ($\hHH=\hQH$). Let $(r,r')$ be a $\cconv$-consistent
relation in $\HH\times\PA\cup\QH\times\PA$. We have eight cases to
analyze:
\begin{itemize}
\item $\dc\in r$: in that case, $H\left((r,r')\right)=\cdeux{\dc}{\hmax(r')}$
which is necessarily $\cconv$-consistent, since $\conv{\RCAH}{\PA}\dc=\B_{\PA}$
and by (1).
\item $\ec\in r$ and $\dc\notin r$: in that case, $H\left((r,r')\right)=\cdeux{\ec}{\hmax(r')}$
which is necessarily $\cconv$-consistent, since $\conv{\RCAH}{\PA}\ec=\B_{\PA}$
and by (1).
\item $\po\in r$ and $\left(\dc\cup\ec\right)\cap r=\vide$: in that case,
$H\left((r,r')\right)=\cdeux{\po}{\hmax(r')}$ which is necessarily
$\cconv$-consistent, since $\conv{\RCAH}{\PA}\po=\B_{\PA}$ and by
(1).
\item $\tpp\in r$ and $\left(\dc\cup\ec\cup\po\right)\cap r=\vide$: in
that case, $\tppi\notin r$ and $\ntppi\notin r$ since $\HH\cap\mathcal{N}=\vide$
and $\QH\cap\mathcal{N}=\vide$ with $\mathcal{N}=\liste{r\in\RCAH}{\po\nsubseteq r\et r\cap\left(\tpp\cup\ntpp\right)\neq\vide\et r\cap\left(\tppi\cup\ntppi\right)\neq\vide)}$
(see Table~\ref{tab:def_classic_subclass}). We therefore necessarily
have $\tpp\in r\subseteq\trois{\tpp}{\ntpp}{\eq}$. Since $(r,r')$
is $\cconv$-consistent, either $r'=\mathord{<}$ or $r'=\left(\mathord{<}\cup\mathord{=}\right)$.
Therefore, $H\left((r,r')\right)=\cdeux{\tpp}<$, which is $\cconv$-consistent.
\item $\tppi\in r$ and $\left(\dc\cup\ec\cup\po\right)\cap r=\vide$: we
have the symmetric case ($H\left((r,r')\right)=(\tppi,>)$).
\item $\eq\in r$ and $\left(\dc\cup\ec\cup\po\cup\tpp\cup\tppi\right)\cap r=\vide$:
in that case, $r=\eq$ since $\ntpp\notin r$ and $\ntppi\notin r$
because $\HH\cap\liste r{(\eq\cup\ntpp\subseteq r\et\tpp\nsubseteq r)\ou(\eq\cup\ntppi\subseteq r\et\tppi\nsubseteq r)}=\vide$
and $\QH\cap\liste r{\eq\subseteq r\et\po\nsubseteq r\et r\cap\left(\tpp\cup\ntpp\cup\tppi\cup\ntppi\right)\neq\vide}=\vide$.
Thus $r'$ is $=$. Therefore $H\left((r,r')\right)=\cdeux{\eq}=$,
which is $\cconv$-consistent.
\item $\ntpp\in r$ and $\left(\dc\cup\ec\cup\po\cup\tpp\cup\tppi\cup\eq\right)\cap r=\vide$:
in that case, $r=\ntpp$ since $\ntppi\notin r$ (because $\HH\cap\mathcal{N}=\vide$
et $\QH\cap\mathcal{N}=\vide$). Thus $r'=\mathord{<}$. Therefore
$H\left((r,r')\right)=\cdeux{\ntpp}<$, which is $\cconv$-consistent.
\item We have the symmetric case for $\ntppi\in r$ and $\left(\dc\cup\ec\cup\po\cup\tpp\cup\tppi\cup\eq\right)\cap r=\vide$.
\end{itemize}
Thus, $\HH\times\PA$ and $\QH\times\PA$ are projection stable through
$\hHH=\hQH$ .

We are now analyzing $\CH\times\PA$:
\begin{itemize}
\item $\dc\in r$: in that case, $H\left((r,r')\right)=\cdeux{\dc}{\hmax(r')}$
which is necessarily $\cconv$-consistent, since $\conv{\RCAH}{\PA}\dc=\B_{\PA}$
and by (1).
\item $\po\in r$ and $\dc\notin r$: in that case, $H\left((r,r')\right)=\cdeux{\po}{\hmax(r')}$
which is necessarily $\cconv$-consistent, since $\conv{\RCAH}{\PA}\po=\B_{\PA}$
and by (1).
\item $\ntpp\in r$ and $\left(\dc\cup\po\right)\cap r=\vide$: in that
case, $\tppi\notin r$ and $\ntppi\notin r$ since $\CH\cap\mathcal{N}=\vide$
with $\mathcal{N}$ as previously defined. Moreover, $\ec\notin r$
since $\CH\cap\liste r{\ec\subseteq r\et\po\nsubseteq r\et r\cap\left(\tpp\cup\ntpp\cup\tppi\cup\ntppi\cup\eq\right)\neq\vide}=\vide$.
Therefore, $\ntpp\in r\subseteq\{\tpp,\ntpp,\eq\}$. Since $(r,r')$
is $\cconv$-consistent, either $r'=\mathord{<}$ or $r'=\left(\mathord{<}\cup\mathord{=}\right)$
Thus, $H\left((r,r')\right)=\cdeux{\ntpp}<$ which is $\cconv$-consistent.
\item $\ntppi\in r$ and $\left(\dc\cup\po\right)\cap r=\vide$: we have
the symmetric case ($H\left((r,r')\right)=\cdeux{\ntppi}>$).
\item $\tpp\in r$ and $\left(\dc\cup\po\cup\ntpp\cup\ntppi\right)\cap r=\vide$:
the case is similar to $\ntpp$.
\item $\tppi\in r$ and $\left(\dc\cup\po\cup\ntpp\cup\ntppi\right)\cap r=\vide$:
we have the symmetric case with respect to $\tpp$.
\item $\ec\in r$ and $\left(\dc\cup\po\cup\ntpp\cup\ntppi\cup\tpp\cup\tppi\right)\cap r=\vide$:
in that case, $r=\ec$ since $\eq\notin r$ (because $\CH\cap\liste r{\un{\ec}\subseteq r\et\un{\po}\nsubseteq r\et r\cap\cinq{\tpp}{\ntpp}{\tppi}{\ntppi}{\eq}\neq\vide}=\vide$).
Thus $H\left((r,r')\right)=\cdeux{\ec}{\hmax(r')}$ which is necessarily
$\cconv$-consistent, since $\conv{\RCAH}{\PA}\ec=\B_{\PA}$ and by
(1).
\item $\eq\in r$ and $\left(\dc\cup\po\cup\ntpp\cup\ntppi\cup\tpp\cup\tppi\cup\ec\right)\cap r=\vide$:
in that case, $r=\eq$. Thus, $r'$ is $=$. Therefore $H\left((r,r')\right)=\cdeux{\eq}=$,
which is $\cconv$-consistent,.
\end{itemize}
Thus, $\CH\times\PA$ is projection stable through $\hCH$.
\end{proof}
Finally, we are briefly interested in a particular case of projection
stability, which we call \emph{projection invariance through multi-refinement}.
This property can be decided in time $\mathcal{O}(\sum_{i}|\S_{i}|)$,
contrary to projection stability which can be decided in time $\mathcal{O}(\sum_{i,j}|\S_{i,j}|)$.
This property is natural and is verified by particular subclasses
of a certain number of combinations. It therefore deserves to be underlined.
\begin{defn}
A subset $\S$ of a multi-algebra is \emph{$\cconv$-invariant through
a multi-refinement} $H=(h_{1},\ldots,h_{m})$ if for all distinct
$i,j\in\{1,\ldots,m\}$ and all $r\in\S_{i}$, we have $\conv ijr=\conv ijh_{i}(r)$.
\end{defn}

\begin{prop}
\label{prop:invariance_to_stability} A subset $\cconv$-invariant
through a multi-affinement $H=(h_{1},\ldots,h_{m})$ is projection
stable through $H$ .
\end{prop}

\begin{proof}
Let $R$ be a $\cconv$-consistent relation : for all distinct $i,j\in\left\{ 1,\ldots,m\right\} $,
we have $R_{j}\subseteq\conv ijR_{i}$. On the one hand, $R_{j}\supseteq h_{j}(R_{j})=(H(R))_{j}$,
and on the other hand, by $\cconv$-invariance, $\conv ijR_{i}=\conv ijh_{i}(R_{i})=\conv ij(H(R))_{i}$.
Therefore, $(H(R))_{j}\subseteq\conv ij(H(R))_{i}$ for all distinct
$i,j\in\left\{ 1,\ldots,m\right\} $. Therefore, $H(R)$ is $\cconv$-consistent.
\end{proof}
For example, one can easily verify by enumeration that for $\STC$,
the basic subclass $\QH\times\PA$ is $\cconv$-invariante through
$H=(h_{\QH},\id)$, with $h_{\QH}$ the classical refinement of $\QH$
and $\id$ the \emph{identity} function. 

\subsection{Projection Closure Problem and Weakening of Formalisms\label{sub:pblm_cloture}}

We have presented a number of conditions ensuring tractability in
the previous sections. However, what can we do when one of these conditions
is invalidated? What to do in particular when considering subclasses
which are not $\cconv$-closed and whose closure under projection
is not a tractable subclass? Does this imply that the subclass in
question is intractable, as in the classical framework for closure
under composition, intersection, and inversion~\cite{ligozat2013qualitative}
? This is actually not always the case. In this section, we will study
a counterexample. What to do also when the subclass that we consider
is neither dissociable nor stable? We answer these questions in this
section by proposing a technique that allows us to circumvent, in
some cases, the fact that some of the assumptions of the tractability
theorems are not satisfied. In other words, this technique allows
us to increase the scope of tractability theorems.

To solve the fact that a subclass is not $\cconv$-closed, the idea
we propose consists in weakening the projections of the multi-algebra,
i.e. making them less restrictive, in order to the affected subclass
becomes $\cconv$-closed. Weakening the projections can also make
it possible to obtain other properties such as $\cconv$-distributivity
and therefore simplicity, or even projection stability. A weakened
projection is still correct. In other words, it will not remove valid
entities pairs. We will see that the tractability results obtained
for the weakened projections also apply to the initial multi-algebra.
However, excessive weakening of projections can cause algebraic closure
to lose the ability to decide satisfiability, even for scenarios.
We must therefore find a balance in the weakening of projections,
which perhaps does not exist. To distinguish the projections of a
multi-algebra $\A$ and those of its weakening, denoted $\A'$, we
use the notation \textquotedblleft $\convDe ij{\A}$\textquotedblright{}
(resp. \textquotedblleft $\convDe ij{\A'}$ \textquotedblright ) to
indicate that it is the projection of $\A$ (resp. $\A'$).
\begin{defn}
\label{defn:affaiblissement} A multi-algebra $\A'$ is called \emph{weakening\index{weakening}}
of a multi-algebra $\A$ when $\A'$ has the same Cartesian product
as $\A$ and when for all distinct $i,j\in\{1,\ldots,m\}$, for all
$b\in\Base_{l}$:
\[
\convDe ij{\A}b\subseteq\convDe ij{\A'}b\text{.}
\]

Let $(\A,\U,\I)$ be a sequential formalism and $\A'$ be a weakening
of $\A$. We say that the formalism $(\A',\U,\I)$ is a \emph{weakening}
of the formalism $(\A,\U,\I)$.
\end{defn}

Note that a weakening of a sequential formalism is indeed a sequential
formalism according to Definition~\ref{defn:combined formalisms},
since the weakened projections are correct. More precisely, we have
the following lemma and the following proposition.
\begin{lem}
\label{lem:proj-affaiblie_approxi-sup}Let $\A$ be a finite multi-algebra,
$\A'$ be a weakening of $\A$ and $R\in\A'$. We have the following
property:
\[
\convDe{}{}{\A}R\subseteq\convDe{}{}{\A'}R\subseteq R
\]
\end{lem}

\begin{prop}
\label{prop:affaiblissement_FLC}Let $(\A,\U,\I)$ be a sequential
formalism, $\A'$ be a weakening of $\A$, and $R\in\A'$. We have
the following property:
\[
\I(\convDe{}{}{\A'}R)=\I(R)
\]

Moreover, the weakening $(\A',\U,\I)$ is a sequential formalism.
\end{prop}

\begin{proof}
Let $R\in\A'$. By Lemma~\ref{lem:proj-affaiblie_approxi-sup}, $\convDe{}{}{\A}R\subseteq\convDe{}{}{\A'}R\subseteq R$.
Given that $\I$ is monotone (Lemma~\ref{lem:ppt_relations}), $\I(\convDe{}{}{\A}R)\subseteq\I(\convDe{}{}{\A'}R)\subseteq\I(R)$.
Since $(\A,\U,\I)$ is a sequential formalism, we have $\I(\convDe{}{}{\A}R)=\I(R)$.
Therefore, $\I(\convDe{}{}{\A'}R)=\I(R)$.

The triplet $(\A',\U,\I)$ is a sequential formalism since $(\A,\U,\I)$
is a sequential formalism and $\I(\convDe{}{}{\A'}R)=\I(R)$ ($\A'$
and $\A$ have the same operators except the projection ; Definition~\ref{defn:combined formalisms}).
\end{proof}
We saw in the previous section that some projections are not superdistributive
over composition or intersection, which prevents some subclasses from
being $\cconv$-distributive. Let us consider the weakening satisfying
that if the projection $\conv ij$ is not superdistributive over composition
or over intersection, then this projection is weakened in the following
way: $\conv ijb=\B$ for all $b\in\Base$ (there is no direct interdependency
from $i$ to $j$). The corresponding multi-algebra is trivially a
weakening and the subclass is indeed $\cconv$-distributive for this
weakened multi-algebra (see Definition~\ref{def:distrib}). Of course,
such an operation does not always preserve the fundamental property
of the satisfiability of algebraically closed scenarios.
\begin{example}
\label{exa:GIA_pblm_cloture}We place ourselves in the framework of
temporal multi-scale reasoning with the Allen interval algebra~\cite{Allen1983}
equipped with its conversion operators, defined by Euzenat~\cite{euzenat2001granularity}
in the context of scales totally ordered by the fineness relation
$\preceq$. As we saw for the point algebra in Example\ref{ex:GPC_multialg},
to be closed under projection is to be closed under the upward conversion
operator and under the downward conversion operator. Consider the
subclass of \emph{preconvex relations} of the interval algebra~\cite{nebel1995reasoning}.
This subclass is tractable in the classical framework, but is it also
in the multi-scale framework? This question is problematic since the
subclass of preconvex relations is not $\cconv$-closed. Since this
subclass is maximal for tractability in the classical framework, its
$\cconv$-closure is necessarily intractable. However, this subclass
is closed with respect to the upward conversion (so it is not closed
under the downward conversion). This property is important because
it is actually possible to omit the downward conversion. Specifically,
consider the weakened multi-algebra where the projections satisfying
$\conv jib=\downgc b$ (i.e. when the scale $g_{i}$ is finer than
the scale $g_{j}$) are replaced in order to satisfy $\conv jib=\ B$.
With this weakened multi-algebra, we can apply the tractability theorems
and thus prove the tractability of preconvex relations. The preconvex
subclass is indeed $\cconv$-closed for this weakening, since its
closure under projection now only applies the upward conversion.
\end{example}

To distinguish the different algebraic closures, we will say that
a network is algebraically closed for a sequential formalism and that
a subclass is algebraically tractable for a sequential formalism.

The following proposition justifies the interest of weakenings: they
make it possible to prove that a subclass is tractable for the initial
formalism when, for example, this one is not $\cconv$-closed.
\begin{prop}
\label{prop:reduc_affaiblissement} Let $\mathcal{F}'=(\A',\U,\I)$
be a weakening of a sequential formalism $\mathcal{F}=(\A,\U,\I)$
and $\S$ be a subset of $\A$. If $\S$ is algebraically tractable
for $\mathcal{F}'$ then $\S$ is algebraically tractable for $\mathcal{F}$.
\end{prop}

\begin{proof}
Let $N$ be a network over $\S$. We suppose that its algebraic closure
for $\mathcal{F}$, denoted by $_{\A}\Cl N$, is not trivially inconsistent.
We will show that $N$ is satisfiable, which proves that $\S$ is
algebraically tractable for $\mathcal{F}$ (if we assume instead that
$_{\A}\Cl N$ is trivially inconsistent, then $N$ is unsatisfiable,
by Proposition~\ref{prop:trivialement-incoherent_incoherent}). Let
$_{\A'}\Cl N$ be the algebraic closure of $N$ for $\mathcal{F}'$.
Since the projections of $\mathcal{F}'$ are less restrictive than
those of $\mathcal{F}$, we have $_{\A}\Cl N\subseteq{}_{\A'}\Cl N$
(because $\convDe{}{}{\A}R\subseteq\convDe{}{}{\A'}R$). Since $_{\A}\Cl N$
is not trivially inconsistent, the algebraic closure of $N$ for $\F'$,
$_{\A'}\Cl N$, is not trivially inconsistent. Thus, $N$ is satisfiable,
since $\S$ is algebraically tractable for $\F'$.
\end{proof}
In general, considering a weakening of a multi-algebra does not allow
us to apply the theorems, since the associated algebraically closed
scenarios may no longer be satisfiable. It is therefore necessary
to consider particular weakenings. We are interested here in weakenings
that preserve the tree structure (see Section~\ref{sub:arconsistency_consistency}),
that we qualify as \emph{tree weakening}. We saw in Section~\ref{sub:arconsistency_consistency}
that the projections of tree multi-algebras satisfy particularly interesting
properties. We have seen, among other things, that the projections
of any of its plenary anti-tree structure (namely the projections
$\conv ij$ such that $i\inArbo j$) summarize the interdependencies.
A weakening of a tree multi-algebra is a tree multi-algebra if the
projections of one of its plenary anti-tree structures remain the
same. This ensures that the two multi-algebras have the same interdependencies.
\begin{defn}
\label{defn:affaiblissement_arbo} Let $\A$ be a finite tree multi-algebra
and $\A'$ be one of its weakening. The multi-algebra $\A'$ is called\emph{
tree weakening\index{tree weakening}} of $\A$ if there exists a
plenary anti-tree structure $\arbo$ of $\A$ such that for all distinct
$i,j\in\{1,\ldots,m\}$, for all $b\in\B_{i}$:
\[
i\inArbo j\sep\implies\sep\convDe ij{\A}b=\convDe ij{\A'}b\text{.}
\]

Let $\A'$ be a tree weakening of a tree multi-algebra $\A$ and let
$(\A,\U,\I)$ be a sequential formalism. We say that the formalism
$(\A',\U,\I)$ is a \emph{tree weakening }of the formalism $(\A,\U,\I)$.
\end{defn}

It is easy to see that the multi-algebra of a tree weakening is a
tree multi-algebra. We can therefore use Proposition~\ref{prop:arborescence}
to prove the satisfiability of the $\cconv$-consistent relations
of this weakened formalism. The tree weakenings indeed preserve the
satisfiability of algebraically closed scenarios. This property follows
from the following lemma in the next proposition:
\begin{lem}
\label{lem:clos_affaiblissement-clos}Let $\A'$ be a tree weakening
of a finite tree multi-algebra $\A$. 
\begin{itemize}
\item The basic relations closed under projection for $\A'$ are also closed
under projection for $\A$. 
\item Algebraically closed scenarios for $\A'$ are also algebraically closed
for $\A$.
\end{itemize}
\end{lem}

\begin{proof}
We recall that $\A'$ is also a tree multi-algebra and that $\A$
and $\A'$ has the same plenary anti-tree structure $\arbo$. In this
proof, the projections are those of $\A$. Let $B$ be a basic relation
closed under projection for $\A'$. Thus $B$ satisfies the following
property: for all distinct $i,j\in\{1,\dots,m\}$, if $i\inArbo j$
then $B_{j}\in\conv ijB_{i}$. Indeed, the projections of $\A$ and
$\A'$ corresponding to the same arc of $\arbo$ are identical (Definition~\ref{defn:affaiblissement_arbo}).
We conclude by Lemma~\ref{lem:clos_partiel_arbo} that $B$ is closed
under projection for $\A$. 

The second result follows directly from the first. More precisely,
a scenario closed under composition for $\A'$ is indeed trivially
closed under composition for $\A$ (it is the same composition). Scenarios
algebraically closed for $\A'$ are therefore also algebraically closed
for $\A$, since to be algebraically closed is to be closed under
composition and under projection.
\end{proof}
\begin{prop}
\label{prop:scenar_affaibli_coherent} Let $\mathcal{F}'$ be a tree
weakening of a sequential formalism $\mathcal{F}$. If the algebraically
closed scenarios for $\mathcal{F}$ are satisfiable, then the algebraically
closed scenarios for $\mathcal{F}'$ are also satisfiable.
\end{prop}

\begin{proof}
Let $S$ be an algebraically closed scenario for $\mathcal{F}'$.
By Lemma~\ref{lem:clos_affaiblissement-clos}, it is also algebraically
closed for $\mathcal{F}$. Since, the algebraically closed scenarios
for $\mathcal{F}$ are satisfiable, $S$ is satisfiable.
\end{proof}
Thus, the tree weakenings preserve the essential property for applying
the slicing theorem, which is in general lost for any weakening. We
can therefore weaken a tree multi-algebra as much as we want as long
as we do not modify the projections of one of its plenary anti-tree
structure.

To summarize, it may happen that we cannot apply the tractability
theorems, in particular because we are interested in a subclass which
is not $\cconv$-closed, dissociable, or stable. A solution to work
around this problem is to look for a tree weakening for which this
subclass satisfies these properties.

\subsection{Summary of the obtained results}

In this section, in order to simplify the application of the results
of this paper, we modify the two tractability theorems, by using the
weakening technique and the stronger properties of the previous sections.

We start by applying the previous results on the slicing theorem (Theorem~\ref{thm:combination}).
\begin{cor}[Weakened slicing theorem]
\label{cor:Th=0000E9or=0000E8me-de-d=0000E9coupage-affaibli} Let
$\S$ be a subset, such that $\S=\S_{1}\times\cdots\times\S_{m}$,
of a sequential formalism $\F=(\A,\U,\I)$ whose multi-algebra $\A$
is a tree multi-algebra and whose algebraically closed scenarios are
satisfiable. Let $\A'$ be a tree weakening of $\A$. 

If the following conditions are satisfied for $\A'$: 
\begin{itemize}
\item each slice $\S_{i}$ is a basic subclass,
\item each slice $\S_{i}$ is minimal,
\item $\S$ is $\cconv$-distributive, and
\item for all distinct $i,j\in\{1,\ldots,m\}$, and for all $r\in\S_{i}$,
$\conv ijr\in\S_{j}$ is satisfied
\end{itemize}
then $\S$ is algebraically tractable (for $\F$ and for $(\A',\U,\I)$).
\end{cor}

\begin{proof}
To prove the weakened slicing theorem, we obviously apply the slicing
theorem, using the properties of the previous sections. By Proposition~\ref{prop:reduc_affaiblissement},
we just have to prove that $\S$ is algebraically tractable for $(\A',\U,\I)$
in order to prove that $\S$ is algebraically tractable for $\F$. 

We therefore prove that $\S$ is algebraically tractable for $(\A',\U,\I)$
using the slicing theorem (Theorem~\ref{thm:combination}). 

Its assumptions are satisfied:
\begin{itemize}
\item $\S$ is a $\icclo$ subset since $\S$ is a subclass (because each
$\S_{i}$ is a subclass (Proposition~\ref{prop:equiv_S_Si})).
\item $\S$ is an atomizable subset by Proposition~\ref{prop:equivalence_coherence_relation}
and since $\S$ is a basic subset (because each $\S_{i}$ is a basic
subset (Proposition~\ref{prop:equiv_S_Si})).
\item $(\A',\U,\I)$ is a sequential formalism (Proposition~\ref{prop:affaiblissement_FLC}).
\item Algebraically closed scenarios for $(\A',\U,\I)$ are satisfiable,
since algebraically closed scenarios for $\F$ are satisfiable, by
Proposition~\ref{prop:scenar_affaibli_coherent} (since $(\A',\U,\I)$
is a tree weakening of $\F$).
\end{itemize}
Its conditions of application are also satisfied:
\begin{itemize}
\item (D1) : Each slice  $\S_{i}$ is indeed minimal.
\item (D2) : $\S$ is dissociable, by Proposition~\ref{prop:distrib-to-simple}
(since $\S$ is $\cconv$-closed, $\cconv$-distributive, and is a
subclass).
\item (D3) : Any $\cconv$-consistent relation $R\in\S$ is satisfiable
(by Proposition~\ref{prop:arborescence}, since $\A'$ is a tree
multi-algebra and the algebraically closed scenarios for $\A'$ are
satisfiable).
\item $\S$ is $\cconv$-closed (by Proposition~\ref{prop:ss-classe_clos-local-clos-global}
since $\S$ is of the form $\S_{1}\times\cdots\times\S_{m}$ and is
a subclass).
\end{itemize}
\end{proof}
We now apply the previous results to the refinement theorem.
\begin{cor}[Weakened Refinement Theorem]
\label{cor:theoreme-d-affinement-affaibli} Let $\F=(\A,\U,\I)$
be a sequential formalism, $\S$ and $\S'$ be two subsets of $\A$,
with $\S$ such that $\S=\S_{1}\times\cdots\times\S_{m}$. Let $H=(h_{1},\ldots,h_{m})$
be a refinement from $\S$ to $\S'$ and $\A'$ be a weakening of
$\A$. If the following conditions hold for $\A'$: 
\begin{itemize}
\item each slice $\S_{i}$ is a subclass,
\item each slice $\S_{i}$ is composition stable through $h_{i}$ ,
\item for all distinct $i,j\in\{1,\ldots,m\}$:
\begin{itemize}
\item $\S_{i}\times\S_{j}$ is projection stable through $(h_{i},h_{j})$
and
\item for all $r\in\S_{i}$, we have $\conv ijr\in\S_{j}$, and
\end{itemize}
\item $\S'$ is algebraically tractable
\end{itemize}
then $\S$ is algebraically tractable (for $\F$ and for $(\A',\U,\I)$). 
\end{cor}

\begin{proof}
To prove the weakened refinement theorem, we obviously apply the refinement
theorem, using the properties of the previous sections. By Proposition~\ref{prop:reduc_affaiblissement},
we just have to prove that $\S$ is algebraically tractable for $(\A',\U,\I)$
in order to prove that $\S$ is algebraically tractable for $\F$. 

We therefore prove that $\S$ is algebraically tractable for $(\A',\U,\I)$
by using the refinement theorem (Theorem~\ref{thm:affinement}).

Its assumptions are satisfied:
\begin{itemize}
\item $(\A',\U,\I)$ is a sequential formalism (Proposition~\ref{prop:affaiblissement_FLC}).
\item $H=(h_{1},\dots,h_{m})$ is a multi-refinement from $\S$ to $\S'$,
two subsets of $\A$.
\end{itemize}
Its conditions of application are also satisfied:
\begin{itemize}
\item (A1) : $\S$ is algebraically stable through $H$ (for $\A'$), by
Proposition~\ref{lem:trois_stabilites}, since:
\begin{itemize}
\item $\S$ is composition stable through $H$, by Proposition~\ref{prop:stable},
because each slice $\S_{i}$ is composition stable through $h_{i}$.
Indeed, this implies that for any slice $\S_{i}$ and all $\scomp$-consistent
$3$-variable network $N$ over $\S_{i}$, $h_{i}(N)$ is still $\scomp$-consistent.
\item $\S$ is projection stable through $H$, by Proposition~\ref{prop:stable},
since for all distinct $i,j\in\{1,\ldots,m\}$, $\S_{i}\times\S_{j}$
is projection stable through $(h_{i},h_{j})$. Therefore $\S_{i,j}$
is projection stable through $(h_{i},h_{j})$.
\end{itemize}
\item (A2) : Algebraically consistent networks over $\S'$ (for $\A'$)
are satisfiable since $\S'$ is algebraically tractable (for $\A'$).
\item $\S$ is a $\icclo$ subset since $\S$ is a subclass because each
$\S_{i}$ is a subclass (Proposition~\ref{prop:equiv_S_Si}).
\item $\S$ is $\cconv$-closed (for $\A'$), by Proposition~\ref{prop:ss-classe_clos-local-clos-global}
(since $\S$ is of the form $\S_{1}\times\cdots\times\S_{m}$, $\S$
is a subclass, and for all distinct $i,j\in\{1,\ldots,m\}$ and for
all $r\in\S_{i}$, $\conv ijr\in\S_{j}$ is satisfied).
\end{itemize}
\end{proof}

\subsection{Illustrative Application of the Theorems : Size-Topology Combination\label{sect:topo_and_size}}

We apply in this section the results of this paper to recover the
tractability results of $\STC$ (see Example~\ref{ex:QST_formalisme}).This
is only an illustration of these results of the literature; the main
interest of the formal framework is that it also applies to subclasses
of large multi-algebras, such as spatio-temporal sequences. We start
by studying the subclass $\RCAHs\times\PAs$ (see Table~\ref{tab:def_classic_subclass}
for the definitions).
\begin{cor}
\label{cor:slice} $\RCAHs\times\PAs$ is algebraically tractable
for $\STC$.
\end{cor}

\begin{proof}
To prove this result, we use a tree weakening of $\STC$. We weaken
the projection from $\PA$ to $\RCAH$ as follows: $\conv{\PA}{\RCAH}b=\B_{\RCAH}$
for all $b\in\PA$.

We directly apply the weakened slicing theorem (Corollary~\ref{cor:Th=0000E9or=0000E8me-de-d=0000E9coupage-affaibli}).

Its assumptions are satisfied:
\begin{itemize}
\item $\RCAHs\times\PAs$ is trivially of the form $\S\times\S'$.
\item $\STC$ is a loose integration and therefore a sequential formalism
(Proposition~\ref{prop:L'int=0000E9gration-l=0000E2che_l=0000E2chement-combin=0000E9}).
\item $\STC$ has a tree multi-algebra since $m=2$ (Definition~\ref{defn:arbo}
; one of its plenary anti-trees is $\RCAH\inA\PA$).
\item Algebraically closed scenarios of $\text{\ensuremath{\STC}}$ are
satisfiable \cite{gerevini2002combining,cohen2017temporal}.
\item The weakening of $\STC$ is a tree weakening of $\STC$ (the anti-tree
is $\RCAH\inA\PA$ since the projection $\conv{\RCAH}{\PA}$ has not
been modified ; see Definition~\ref{defn:affaiblissement_arbo}).
\end{itemize}
Its conditions of application are also satisfied:
\begin{itemize}
\item $\RCAHs$ and $\PAs$ are basic subclasses.
\item $\RCAHs$ and $\PAs$ are minimal (Proposition~\ref{prop:sous-classes-classiques-id-scenarisables}).
\item $\RCAHs\times\PAs$ is $\cconv$-distributive (by Corollary~\ref{cor:sur-distrib-RCCHs}
for $\conv{\RCAH}{\PA}$ and since the weakened projection $\conv{\PA}{\RCAH}$
is trivially superdistributive over composition and intersection).
\item for all $r\in\RCAHs$, we have $\conv{\RCAH}{\PA}r\in\PAs$ (Lemma~\ref{lem:STC-ss-classe-conv-close})
and for all $r'\in\PAs$, we have $\conv{\PA}{\RCAH}r'=\B_{\RCAH}\in\RCAHs$
(since $\conv{\PA}{\RCAH}$ is weakened).
\end{itemize}
\end{proof}
Now consider $\HH$, $\CH$, and $\QH$, the three basic subclasses
of $\RCCH$ maximal for tractability (see Section~\ref{subsec:RW:Qualitative}).
This time we cannot apply the slicing theorem, because simplicity
is not satisfied\footnote{More precisely, simplicity is not satisfied even by considering weakenings
preserving the satisfiability of the algebraically closed scenarios.} (closing under projection then under composition is not enough to
obtain the algebraic closure). We can, however, apply the refinement
theorem.
\begin{cor}
Let $\S$ be one of the following subclasses: $\HH$, $\CH$, or $\QH$.

The subclass $\S\times\PA$ is algebraically tractable for $\STC$.
\end{cor}

\begin{proof}
We apply the weakened refinement theorem, with $\STC$ itself as weakening
(we do not need to weaken the projections). For this, the refinement
used for each subclass $\S$ is its classical refinement towards the
basic relations for $\RCCH$ (that we have denoted by $h_{\S}$ ;
Definition~\ref{def:Affinement-RCC8}) and the refinement $\hmax$
for $\PA$ (Table~\ref{tab:affinement_PA}).

Its assumptions are satisfied:
\begin{itemize}
\item $\STC$ is a loose integration and therefore a sequential formalism
(Proposition~\ref{prop:L'int=0000E9gration-l=0000E2che_l=0000E2chement-combin=0000E9}).
\item $\S\times\PA$ is trivially of the form ``$\S\times\S'$''. 
\item $H=(h_{\S},\hmax)$ is a function from $\S\times\PA$ to $\RCAHs\times\PAs$.
\item $\STC$ is trivially a weakening of $\STC$.
\end{itemize}
Its conditions of application are also satisfied:
\begin{itemize}
\item $\S$ and $\PA$ are subclasses.
\item $\S$ is composition stable through $h_{\S}$ (Proposition~\ref{prop:stab-compo-RCC})
and $\PA$ is composition stable through $\hmax$ (Proposition~\ref{prop:stab-compo-PA}).
\item $\S\times\PA$ is projection stable through $(h_{\S},\hmax)$ (Lemma~\ref{lem:stab-projection-STC}).
\item for all $r\in\S$, we have $\conv{\RCAH}{\PA}r\in\PA$ and for all
$r'\in\PA$, we have $\conv{\PA}{\RCAH}r'\in\S$ (Lemma~\ref{lem:STC-ss-classe-conv-close}).
\item $\RCAHs\times\PAs$ is algebraically tractable (Corollary~\ref{cor:slice}).
\end{itemize}
\end{proof}
Thus the tractability results  of $\STC$ has been recovered.

\section{Discussion and Conclusion\label{sec:Discussion-and-Conclusion}}

In this last section, we discuss the limitations of the multi-algebra
framework. Finally, we conclude and expose futur work.

\subsection{Limitations\label{subsec:Limitations}}

We are interested in this section in the limits of the tractability
results and more generally in the limits of the formal framework that
we propose.

Note on the one hand that, some of the hypotheses of the theorems
are not necessary, but only sufficient, for algebraic tractability.
On the other hand, the theorems only allows to demonstrate algebraic
tractability, however some subclasses are tractable without being
algebraically tractable. It is for example the case of the loose integration
of $\RCCH$ and the (minimum bounding) rectangle calculus of regions~\cite{li2012reasoning}
because some of its algebraically closed scenarios are unsatisfiable.
For these formalisms, one must either consider an analysis of greater
algorithmic complexity than algebraic closure to decide satisfiability,
such as \emph{BC-closure}, or consider a similar expressiveness formalism
whose algebraically closed scenarios are all satisfiable. For example,
one can consider the loose integration of $\RCCH'$ and the rectangle
calculus \cite{cohn2014reasoning} ($\RCCH'$ is $\RCCH$ with a slightly
different semantics: it concerns the definition of the connection
relation), instead of the loose integration of $\RCCH$ and the rectangle
calculus. One can also consider the combination of $\RCCH$ and $\mathrm{DIR}49$
\cite{li2012reasoning} ($\mathrm{DIR}49$ is less expressive than
the rectangle calculus). 

Finally, we recall that the proposed formal framework only applies
to combinations of symmetric formalisms -- although it should be
adaptable to the \textquotedblleft non-symmetric\textquotedblright{}
framework (see Section~\ref{sect:def_formalismes_qualitatifs_sym})
-- similar in structure to loose integration. Thus, tight integrations
\cite{wolfl2009combinations}, combinations of heterogeneous and/or
quantitative information \cite{Meiri1996,inants2016qualitative,jonsson2004complexity},
and spatio-temporal combinations of temporal expressiveness superior
to temporal sequences \cite{gerevini2002qualitative,bennett2002multi,hazarika2001qualitative,muller2002topological}
are outside the framework of multi-algebras.

\subsection{Conclusion and Futur Work}

\subsubsection{Conclusion}

In this paper, we have proposed a formal framework for representing
knowledge, reasoning, and identifying tractable fragments, in a unified
way, in the context of loose integrations, multi-scale formalisms,
and temporal sequences.

First of all, we have generalized the definition of qualitative formalism
of \cite{ligozat2004qualitative} to include formalisms whose basic
relations are not exhaustive or which do not have equality on $\U$,
since they are important in the context of combinations. However,
this definition, like that of \cite{ligozat2004qualitative}, is restricted
to formalisms satisfying the $\sinv{\sinv r}=r$ property, since the
framework we propose requires this assumption.

Secondly, we introduced multi-algebras, the common structure of loose
integrations, multi-scale formalisms, and temporal sequences. The
operators of a multi-algebra make it possible to reason in the context
of these combinations. The projection is added to the classical operators
and makes it possible to propagate the interdependencies between the
relations of combined formalisms. The concepts of the classical framework
apply or become generalized in a transparent way. In particular, a
description is a unique network whose relations are $m$-tuples of
classical relations. Moreover, the knowledge inference mechanism is
always the algebraic closure, which has been generalized to close
under both composition and projection. We then formalized the semantics
of multi-algebra relations, by defining the sequential formalisms.
This semantics preserves most of the fundamental properties of the
classical framework. On the one hand, any solution of a network is
solution of a single scenario, refining this network. On the other
hand, basic relations are pairwise disjointed. Moreover, all the satisfiable
scenarios are algebraically closed. And finally, if the algebraically
closed scenarios are satisfiable, then the decision problem of satisfiability
is in $\NP$, and one can decide satisfiability by backtracking algorithms
using the algebraic closure as a method of pruning and evaluation.

In addition, we have identified conditions for $\cconv$-consistent
relations to be satisfiable, namely the tree structure of interdependencies
and the satisfiability of basic relations closed under projection.

Moreover, we have demonstrated two theorems identifying conditions
for a subclass to be algebraically tractable. The first result provides
conditions for a combination of tractable subclasses to be tractable.
The second result is complementary: it provides conditions for a subclass
to be reduced to a tractable subclass, and thus be tractable. We then
looked at the conditions of the theorems. We have, on the one hand,
determined stronger assumptions, for the tractability theorems, which
have the advantage of being easier to verify, like distributivity
of the operators. We have, on the other hand, proposed the technique
of weakening which allows us, in certain cases, to circumvent the
fact that some of the assumptions of the theorems are not satisfied.
The circumvention concerns in particular the fact that a subclass
is not closed under projection or is not distributive.

We also discussed the applicability limits of this formal framework.
In particular, it does not make it possible to prove tractability
of non-algebraically tractable, tractable subclasses, although it
is possible, \emph{a priori}, to circumvent this problem by considering
a similar expressivity formalism or by considering other decision
procedures. The non-applicability of the framework would come from
the need for a satisfiability analysis having a complexity greater
than the algebraic closure.

Finally, we have illustrated the application of our theorems by recovering
the tractability results of the size-topology combination.

\subsubsection{Futur Work}

The research perspectives of the framework of multi-algebras are placed
on four axes.

On the one hand, there is the study and improvement of algorithms
for satisfiability decision. We have provided a naive procedure for
calculating the algebraic closure. This must, however, be computable
more efficiently for particular families of multi-algebras but also
in the general context. In particular, the following work should be
generalizable to the context of multi-algebras: \cite{sioutis2021dynamic,wehner2021robust,sioutis2019collective,kong2018exploring,sioutis2018leveraging,sioutis2017lazy,sioutis2017efficiently,sioutis2017studying,sioutis2016efficiently,sioutis2016efficient,long2016efficient,sioutis2015simple,sioutis2014incrementally,sioutis2014vertex,sioutis2017algorithmic}. 

Algebraic closure should also be easily parallelizable by parallelizing
the algebraic closure of each subnetwork $N_{i}$, but a generalization
of the parallelization of the classical framework is also possible~\cite{sioutis2019towardsP}.
In addition, the branching method of the classical framework for deciding
satisfiability in intractable cases must be adapted to this context~\cite{Ladkin1992,nebel1997solving}
. There are, however, several possibilities of generalization, and
it is not clear which one has the best computational performance.
Finally, a study of the parameterized complexity of the decision problem
of satisfiability according to the number of entities $n$ and the
number $m$ of combined formalisms should be conducted. This study
would be particularly interesting in the context of temporal sequences
(where $m$ corresponds to the length of the sequence) and in the
context of multi-scale reasoning (where $m$ corresponds to the number
of scales), as in practice, these numbers can be very large.

On the other hand, the framework of multi-algebras can be extended
by studying other problems than satisfiability. The minimality problem
has already been generalized to this context~\cite{cohen2020minimality},
but the study is not finished, in particular two generalizations of
the calculation of the minimal network have been proposed but it is
not clear which one is experimentally more efficient than the other.
Work on minimality in classical case could also be generalizable to
this framework~\cite{amaneddine2013efficient,sioutis2021neighbourhood,amaneddine2013minimal,amaneddine2012path,gerevini2011computing,beek1990exact}.

Other interesting problems than satisfiability and minimality, which
are generalizable in this context, are:
\begin{itemize}
\item redundancy \cite{hu2021large,lee2016redundancy,sioutis2015efficiently}, 
\item robustness \cite{sioutis2021robustness}, 
\item maximizing satisfiability \cite{condotta2015practical,condotta2016local}, 
\item backbones and backdoors \cite{sioutis2019towardsB}, 
\item networks merging \cite{condotta2008framework,condotta2009merging},
... 
\end{itemize}
These problems have not yet been studied in the context of loose integrations,
multi-scale formalisms, and temporal sequences. 

Moreover, the use of other techniques of satisfiability decision,
in particular of higher complexity like the $BC$-closure, is necessary
when the reasoning of the algebraic closure is not complete. Their
introduction within the multi-algebra framework should allow us the
identification of similar general tractability results.

In addition, the results of this paper can be applied to many other
combinations: other loose integrations, multi-scale formalisms, and
temporal sequences, thus making it possible to formalize these reasonings,
to obtain a procedure of satisfiability decision, and to identify
tractable subclasses. It has, for example, already been applied to
obtain new tractability results in the context of several temporal
sequences~\cite{cohen2020tractable,cohen2017temporal}.

Finally, the framework of multi-algebras must be generalized to cover
more combinations. For this purpose, the expressiveness of the multi-algebra
framework must be increased. In particular, it should allow reasoning
in the presence of heterogeneous entities associated with different
partition schemes, like the combination of the region connection calculus
with the interval calculus \cite{gerevini2002qualitative}. Moreover,
it is perhaps generalizable within the framework of quasi-qualitative
formalisms, that is to say within the framework of \textquotedbl qualitative
formalisms\textquotedbl{} whose number of relations is infinite \cite{inants2016qualitative}.

\appendix

\section*{\printindex}

\vskip 0.2in \bibliography{biblio}
 \bibliographystyle{theapa}
\end{document}